\documentclass[11pt]{article}

\usepackage{hyperref}       
\usepackage{booktabs}       
\usepackage{nicefrac}       
\usepackage{microtype}      
\usepackage{times}
\usepackage{color}
\usepackage{subfigure}
\usepackage{graphicx}
\usepackage{amssymb}
\usepackage{amsthm}
\usepackage{amsmath}
\usepackage{algorithm}
\usepackage{algorithmic}
\usepackage{amsfonts}
\usepackage{appendix}
\usepackage{multirow}

\usepackage{caption}
\usepackage{lipsum}

\DeclareCaptionStyle{algori}{singlelinecheck=off,labelsep=space}
\newcommand{\citep}{\cite}

\newtheorem{lemma}{Lemma}
\newtheorem{theorem}[lemma]{Theorem}
\newtheorem{definition}[lemma]{Definition}

\newtheorem{remark}[lemma]{Remark}

\newcommand{\Exp}{{\mathbb{E}}}
\newcommand{\innerprod}[2]{\langle #1,#2 \rangle}
\newcommand{\topic}[1]{ \smallskip\noindent{\bf #1:}}
\DeclareMathOperator{\argmin}{argmin}

\newcommand{\e}{\epsilon}
\newcommand{\R}{\mathbb{R}}

\newcommand{\OPT}{\mathsf{OPT}}

\newcommand{\calD}{\mathcal{D}}
\newcommand{\calP}{\mathcal{P}}
\newcommand{\calQ}{\mathcal{Q}}
\newcommand{\calS}{\mathcal{S}}

\newcommand{\client}{C}
\newcommand{\clientset}{\mathcal{C}}
\newcommand{\server}{S}

\makeatletter
\newtheorem{rep@theorem}{\rep@title}
\renewcommand\therep@theorem{\unskip}
\newcommand{\newreptheorem}[2]{%
\newenvironment{rep#1}[1]{%
 \def\rep@title{#2 \textbf{\ref{##1}}}%
 \begin{rep@theorem}}%
 {\end{rep@theorem}}}
\makeatother

\newreptheorem{theorem}{\textbf{Theorem}}
\newreptheorem{lemma}{\textbf{Lemma}}

\newcommand{\eat}[1]{}

\begin{document}

\title{SVM via Saddle Point Optimization:
	New Bounds and Distributed Algorithms}

\author{Yifei Jin   \\
 Tsinghua University\\
    \and
    Lingxiao Huang \\
    EPFL \\
    \and
	Jian Li \\
    Tsinghua University\\
}

\date{}

\maketitle


 \begin{abstract}

We study two important SVM variants:
hard-margin SVM (for linearly separable cases) and $\nu$-SVM (for linearly non-separable cases).
We propose new algorithms from the
perspective of saddle point optimization.
Our algorithms achieve $(1-\e)$-approximations with running
time  $\tilde{O}(nd+n\sqrt{d / \e})$ for both variants, where $n$ is the number of points and $d$ is the dimensionality.
To the best of our knowledge, the current best algorithm  for $\nu$-SVM  is based on  quadratic programming approach which requires $\Omega(n^2 d)$  time in worst case~\citep{joachims1998making,platt199912}. In the paper, we provide the first nearly linear time algorithm for $\nu$-SVM.
The current best algorithm for hard margin SVM achieved by Gilbert algorithm~\citep{gartner2009coresets} requires $O(nd / \e )$
time. Our algorithm improves the running time by a factor of $\sqrt{d}/\sqrt{\e}$.
Moreover, our algorithms can be implemented in the distributed settings naturally. We prove that our algorithms require $\tilde{O}(k(d +\sqrt{d/\e}))$ communication
cost, where $k$ is the number of clients, which almost matches the theoretical lower bound.
Numerical experiments support our theory and show that our algorithms converge faster on high dimensional,
large and dense data sets, as compared to previous methods.

\end{abstract}


\section{Introduction}
\label{sec:intro}

Support Vector Machine (SVM) is widely used for
classification in numerous applications such as text categorization, image classification, and hand-written characters recognition.

In this paper, we focus on binary classification.
If two classes of points which are linearly separable, one can use the hard-margin SVM (\cite{boser1992training,cortes1995support}), which is to find a  hyperplane that separate two classes of points and the margin is maximized.
If the data is not linearly separable, several popular SVM variants have been proposed, such as $l_2$-SVM, $C$-SVM and $\nu$-SVM (see e.g., the summary in \cite{gartner2009coresets}).
The main difference among these variants is that they use different
penalty loss functions for the misclassified  points.
$l_2$-SVM, as the name implied, uses the $l_2$ penalty loss.
$C$-SVM and $\nu$-SVM are two well-known SVM variants using $l_1$-loss.
$C$-SVM uses the $l_1$-loss
with penalty coefficient $C \in [0, \infty)$ \citep{zhu20041}.
On the other hand,  $\nu$-SVM reformulates $C$-SVM through taking a new regularization parameter $\nu \in (0,1]$~\citep{scholkopf2000new}.
However,  given a $C$-SVM formulation, it is not easy to compute the regularization parameter $\nu$ and obtain an equivalent $\nu$-SVM. Because the equivalence is based on some hard-to-compute constant. Compared to $C$-SVM,
the parameter $\nu$ in $\nu$-SVM has a more clear geometric interpretation: the objective is to minimize the distance between two reduced polytopes defined based on $\nu$~\citep{crisp2000geometry}. However, the best known algorithm for
$\nu$-SVM is much worse than that for $C$-SVM in practice (see below).

In general, SVMs  can be formulated as convex quadratic programs and  solved by quadratic programs in $O(n^2d)$ time~\citep{joachims1998making,platt199912}.
However, better algorithms exists for some SVM variants, which we briefly discuss below.

For hard-margin SVM,
\cite{gartner2009coresets} showed that Gilbert algorithm \citep{gilbert1966iterative}
achieves a $(1-\e)$-approximation with $O(nd / \e \beta^2)$ running time where $\beta$ is the ratio of the
minimum distance to the maximum one among the points.
$l_2$-SVM and $C$-SVM have been studied extensively and
current best algorithms runs in time linear in the number $n$ of data points
  \citep{shalev2007pegasos,franc2008optimized,duchi2009efficient,allen2016katyusha}. 
However, these techniques cannot be extended to $\nu$-SVM directly, mainly because $\nu$-SVM
cannot be transformed to single-objective unconstrained optimization problems. Except the traditional
quadratic programming approach, there is no better algorithm known with provable guarantee for $\nu$-SVM.
Whether $\nu$-SVM can be solved in nearly linear time is still open.

Distributed SVM has also attracted significant attention in recent years.  A number of distributed algorithms for SVM  have been obtained in the past~\citep{graf2004parallel,navia2006distributed,lu2008distributed,forero2010consensus,zhang2012efficient}.
Typically, the communication complexity is one of the key performance measurements for distributed algorithms, and
has been studied extensively (see \cite{yao1979some,orlitsky1990average,kushilevitz1997communication} ). For hard-margin SVM, recently, Liu et al.~\cite{liu2016distributed} proposed a distributed algorithm with $O(kd/\e)$ communication cost, where $k$ is the number of the clients. Hence, it is a natural question to
ask whether the communication cost of their algorithm can be improved.


\eat{
The most popular
distributed model is to  store data in distributed sites,
and those sites collaboratively solve the algorithmic problem of interest
by communicating with each other through network links.
}
\subsection{Our Contributions}
\label{sec:contribution}
We summarize our main contributions as follows.
\begin{enumerate}
\item \topic{Hard-Margin SVM} We provide a new $(1-\e)$-approximation algorithm with running time
 $\tilde{O}(nd+n\sqrt{d}/\sqrt{\e \beta})$,
 where $\beta$ is the ratio of the
 minimum distance to the maximum one among the points (see Theorem \ref{thm:convergence}). \footnote{$\tilde{O}$ notation hides logarithm factors such as $\log(n)$, $\log(\beta)$ and
 $\log(1/\e)$.} Compared to Gilbert algorithm \citep{gartner2009coresets}, our algorithm improves the running time by a factor of $\sqrt{d}/\sqrt{\e}$.
 First, we regard hard-margin SVM as computing the polytope distance between two classes of points. Then we translate the problem to a saddle point
 optimization problem using the properties of the geometric structures (Lemma \ref{lm:svmtosp}), and provide an algorithm to solve the saddle point optimization.


 \item \topic{$\nu$-SVM}  Then, we extend our algorithm to  $\nu$-SVM and
 design an $\tilde{O}(nd+n\sqrt{d}/\sqrt{\e \beta})$ time
 algorithm, which is the most important technical contribution of this paper.  To the best of our knowledge, it is the first
 nearly linear time algorithm for $\nu$-SVM.
 It is known that $\nu$-SVM is equivalent to computing the distance between two reduced
 polytopes~\citep{bennett2000duality,crisp2000geometry}. The obstacle for providing an efficient algorithm based on the reduced polytopes is that the number of vertices in the reduced polytopes may be exponentially large. However, in our framework, we only need to implicitly represent the reduced polytopes. We show that using the similar saddle point optimization framework, together with a new
 nontrivial projection method, $\nu$-SVM can be solved efficiently in the same time complexity as in the hard-margin case. Compared with the QP-based algorithms in previous work \citep{joachims1998making,platt199912}, our algorithm significantly improves the running time, by a factor of $n$.

\item \topic{Distributed SVM} Finally, we extend our algorithms for both hard-margin SVM and $\nu$-SVM to the
distributed setting.
We prove that the communication cost of our algorithm is $\tilde{O}(k(d +  \sqrt{d/\e}))$, which is almost optimal according to the lower bound provided in \cite{liu2016distributed}.
For the hard-margin SVM, compared with the current best algorithm~\citep{liu2016distributed}  with $O(kd/\e)$ communication cost,  our algorithm is more suitable when $\e$ is small and $d$ is large. For $\nu$-SVM, our algorithm is the first practical distributed algorithm.
\end{enumerate}

\eat{ Second, our algorithms  in both centralized and distributed settings and for hard-margin SVM and $\nu$-SVM  can be extended to the
multiclass SVM problems without loss in performance. }

Besides, the numerical experiments support our theoretical bounds.
We compare our algorithms with Gilbert Algorithm~\citep{gartner2009coresets}
and NuSVC, LinearSVC in scikit-learn~\citep{pedregosa2011scikit}.
The experiments show that our algorithms converge faster on high dimensional,
large and dense data sets.

\eat{
The paper is organized as follows. In Section~\ref{sec:series}, we reformulate the hard-margin SVM and $\nu$-SVM
as saddle point optimization problems.  Then we present our algorithms to solve the
saddle point optimization problems and analyze the running time in Section~\ref{sec:svmsp}. In
Section~\ref{sec:dist}, we provide the  distributed version of our algorithms and analyze its
communication complexity. In Section~\ref{sec:exp}, we provide the experimental results
comparing with Gilbert Algorithm~\cite{gartner2009coresets} and NuSVC, LinearSVC in scikit-learn~\cite{pedregosa2011scikit}.
Due to space constraints, we defer many proof details and some experimental results to the appendix.
}

\subsection{Other Related Work}
\label{sec:related}
For the hard-margin SVM, there is an alternative to Gilbert's method, called 
the MDM algorithm, originally proposed by \cite{mitchell1974finding}.
Recently, L{\'o}pez and Dorronsoro
proved that the rate of convergence of MDM algorithm is $O(n^2d \log (1/\e))$ \citep{lopez2015linear} which is a linear convergence w.r.t. $\e$,
but worse than Gilbert Algorithm w.r.t. $n$.

Both $C$-SVM and $l_2$-SVM have been studied extensively in the literature. Basically,
there are three main algorithmic approaches:
the primal gradient-based
methods~\citep{kivinen2004online,shalev2007pegasos,duchi2009efficient,franc2008optimized,allen2016katyusha},
dual quadratic programming
methods~\citep{joachims2006training,smola2007bundle,hsieh2008dual} and
dual geometry methods~\citep{tsang2005core,tsang2007simpler}.
Recently, \cite{allen2016katyusha} provided the current
best algorithms which achieve $O( nd / \sqrt{\e})$ time for $l_2$-SVM and $O(nd/\e)$ time for $C$-SVM.

Some sublinear time algorithms for hard-margin SVM and $l_2$-SVM have been proposed~\citep{clarkson2012sublinear, hazan2011beating}.
These algorithms
are sublinear w.r.t. $nd$, (i.e., the size of the input), but
have worse dependency on $1/\e$.

The algorithmic framework for saddle point optimization
was first developed by Nesterov for structured nonsmooth optimization problem~\citep{nesterov2005excessive}.
He only considered the full gradient in the algorithm.
Recently, some studies have extended it to the stochastic gradient setting~\citep{zhang2015stochastic,allen2016optimization}.
The most related work is~\citep{allen2016optimization}, in which the author obtained an $\tilde{O}(nd+n\sqrt{d}/\sqrt{\e})$ algorithm
for the minimum enclosing ball problem (MinEB) in Euclidean space, using the saddle point optimization.
This result also implies an algorithm for $l_2$-SVM, by the connection between MinEB and $l_2$-SVM (see \cite{tsang2005core,har2007maximum,tsang2007simpler}).
However, the implied algorithm is not as efficient.
Based on \cite{tsang2005core,tsang2007simpler}, the dual of
$l_2$-SVM is equivalent to MinEB by a specific feature mapping. It maps a $d$-dimensional
point to the $(d+n)$-dimensional space. Thus, after the mapping, it takes
quadratic time to solve $l_2$-SVM.
To avoid this mapping, they designed an algorithm called Core Vector Machine (CVM), in which they can solve $l_2$-SVM by solving
$O(1/\e)$ MinEB problems sequentially.


\eat{
\section{Saddle Point Optimization for SVM}
\label{sec:saddle}
}

\section{Formulate SVM as Saddle Point Optimization}
\label{sec:series}
In this section,
we formulate both hard-margin SVM and $\nu$-SVM, and show that they can be reduced to saddle point
optimizations. All vectors in the paper are all column vectors
by default.

\begin{definition}[Hard-margin SVM] Given $n$ points $x_i \in \R^d$ for $1 \leq i \leq n$, each $x_i$ has a label $y_i \in \{\pm 1\}$.
The hard-margin SVM can be formalized as  the following
quadratic programming~\citep{cortes1995support}.
\begin{equation}
  \label{eq:lp}
  \begin{array}{lcl}
	\min\limits_{w, b}  & \frac{1}{2}\| w \|^2   & \\
	\text{s.t.}&  y_i(w^{\mathrm{T}}x_i - b)   \geq 1,  & \forall i
  \end{array}
\end{equation}
\end{definition}

The dual problem of \eqref{eq:lp} is defined as follows, which is equivalent to finding the minimum distance between the two convex hulls of two classes of points~\citep{bennett2000duality} when they are linearly separable. We call the problem the C-Hull problem.
\begin{equation}
  \label{eq:chull}
  \begin{array}{lcl}
	\min\limits_{\eta, \xi}  &\frac{1}{2} \| A \eta - B \xi   \|^2   &   \\
	\text{s.t.}&  \| \eta \|_1 = 1, \| \xi \|_1 = 1.  \quad  \eta \geq  0 , \xi  \geq  0.
  \end{array}
\end{equation}
where  $A $ and $B$ are the  matrices in which each column represents a vector of a point with label $+1$ or $-1$ respectively.

Denote the set of points with label $+1$ by $\calP$ and the set with label $-1$ by $\calQ$. Let $n_1 = |\calP|$ and $n_2 = |\calQ|$.
Since $   \sum_{i} \eta_i  = 1$, we can regard it as a probability distribution among points in $P$ (similarly for $Q$). We denote $\Delta_{n_1}$ to be the
 set of $n_1$-dimensional probability  vectors over $\calP$ and $\Delta_{n_2}$ to be that over
 $\calQ$.  Then, we prove that the C-Hull problem~\eqref{eq:chull} is equivalent to the following saddle point
 optimization in Lemma \ref{lm:svmtosp}.
 We defer the proof to Appendix~\ref{apd:series}.
\begin{lemma}
\label{lm:svmtosp}
Problem C-Hull~\eqref{eq:chull}  is equivalent to the saddle point
optimization~\eqref{eq:saddle}.
\begin{equation}
\label{eq:saddle}
\OPT =  \max \limits_{w}  \min \limits_{ \eta  \in \Delta_{n_1} ,  \xi \in \Delta_{n_2} }
   w^{\rm T} A \eta -   w^{\rm T} B \xi  -  \frac{1}{2}\| w \|^2
\end{equation}
\end{lemma}

Let $ \phi(w, \eta, \xi) =  w^{\rm T} A \eta -   w^{\rm T} B \xi  -  \|
w \|^2 /2 $. Note that $\phi(w, \eta, \xi) $ is  only linear  w.r.t. $\eta $ and
$\xi$. However, in order to obtain an algorithm which converges faster, we hope that the objective function is strongly convex with respect to $\eta$
and $\xi$. For this purpose,
we can add a small  regularization term which
ensures that the objective function is strongly convex. This is a commonly used approach in optimization (see \cite{allen2016optimization} for an example). Here,
we use the entropy function $H( u ) := \sum_{i} u_i \log u_i$ as the regularization term. The new saddle
point optimization problem is as follows.
\begin{equation}
\label{eq:saddle2}
\begin{array}{rl}
 \max \limits_{w}  \min \limits_{ \eta  \in \Delta_{n_1} ,  \xi \in \Delta_{n_2} } &   w^{\rm T} A \eta -   w^{\rm T} B \xi   \\
   & \qquad + \gamma H(\eta) + \gamma H(\xi)   -  \frac{1}{2}\| w \|^2,
   \end{array}
\end{equation}
where $\gamma = \e \beta / 2\log n$. The following lemma describes the efficiency of the above saddle point
optimization~\eqref{eq:saddle2}. We defer the proof to Appendix~\ref{apd:series}.

\begin{lemma}
\label{lm:approx}
Let $(w^*, \eta^*, \xi^*) $ and $(w^{\circ},   \eta^{\circ}, \xi^{\circ})$ be the optimal solution of saddle point
  optimizations~\eqref{eq:saddle} and ~\eqref{eq:saddle2} respectively. Define $\OPT$ as in \eqref{eq:saddle}.  Define
  \begin{equation*}
g(w) :=  \min  \limits_{ \eta  \in \Delta_{n_1} ,  \xi \in \Delta_{n_2} }   w^{\rm T} A  \eta -   w^{\rm T} B \xi
  -  \frac{1}{2}\| w \|^2.
  \end{equation*}
 Then  $g(w^*) - g(w^{\circ}) \leq \e \OPT$ (note that $g(w^*)=\OPT$).
\end{lemma}
We call the saddle point optimization~\eqref{eq:saddle2}
the Hard-Margin Saddle problem, abbreviated as
HM-Saddle. Next, we discuss $\nu$-SVM (see \cite{crisp2000geometry,scholkopf2000new}) and again
provide an equivalent saddle point optimization formulation.

\begin{definition}[$\nu$-SVM]
Given $n$ points $x_i \in \R^{d}$ for $ 1 \leq i \leq n$, each $x_i$ has a label $y_i \in
 \{+1,-1\}$. $\nu$-SVM is the quadratic programming as follows.
\begin{equation}
  \label{eq:nusvm}
  \begin{array}{lcl}
	\min\limits_{w, b,\rho,\delta}  & \frac{1}{2}\| w \|^2 - \rho + \frac{\nu}{2}\sum_{i}\delta_i  & \\
	\text{s.t.}&  y_i(w^{\mathrm{T}}x_i - b)   \geq \rho - \delta_i , \delta_i \geq 0 ,& \forall i
  \end{array}
\end{equation}
\end{definition}

\cite{crisp2000geometry} presented a geometry interpretation for $\nu$-SVM. They proved that $\nu$-SVM is equivalent to the problem of finding the closest distance between two
reduced convex hulls as follows.
\begin{equation}
  \label{eq:reducedhull}
  \begin{array}{lcl}
	\min\limits_{\eta,\xi}  &\frac{1}{2} \| A \eta - B \xi   \|^2     \\
	\text{s.t.}&  \| \eta \|_1 = 1, \| \xi \|_1 = 1. \\
	  & 0 \leq  \eta_i \leq \nu,   0 \leq \xi_j \leq \nu,   \forall i, j  \\
  \end{array}
\end{equation}
We call the above problem the Reduced Convex Hull problem, abbreviated as RC-Hull. The difference
between C-Hull~\eqref{eq:chull} and RC-Hull~\eqref{eq:reducedhull} is that in the latter
one, each entry of $\eta$ and $\xi$ has an upper bound $\nu$. Geometrically, it means to compress the convex hull of $\calP$ and $\calQ$ such that the two reduced convex hulls are
linearly separable. We define $\calD_{n_1} $ to be the domain of $\eta$ in RC-Hull, i.e., $ \{ \eta
\mid \| \eta \|_1 = 1,   0 \leq  \eta_i \leq \nu,  \forall i \}$ and $\calD_{n_2}$ to be the domain of
$\xi$, i.e., $\{ \xi \mid   \| \xi \|_1 = 1,  0 \leq \xi_j \leq \nu, \forall j \} $. Similar to
Lemma~\ref{lm:svmtosp}, we have the following lemma. The proof is deferred to Appendix~\ref{apd:series}.
\begin{lemma}
  \label{lm:eq}
  RC-Hull~\eqref{eq:reducedhull} is equivalent to the following saddle point optimization.
\begin{equation}
\label{eq:nusaddle}
\OPT =  \max \limits_{w}  \min \limits_{ \eta  \in \calD_{n_1} , \; \xi \in \calD_{n_2} }  w^{\rm T} A \eta -   w^{\rm T} B \xi  -  \frac{1}{2}\| w \|^2.
\end{equation}
\end{lemma}

Again, we add two entropy terms to make the objective function strongly convex with
respective to $\eta$ and $\xi$.
\begin{equation}
\label{eq:nusaddle2}
\begin{array}{rl}
  \max \limits_{w}  \min \limits_{ \eta  \in \calD_{n_1} ,    \xi \in \calD_{n_2} }  & w^{\rm T} A \eta -
w^{\rm T} B \xi  \\
+ & \gamma H(\eta) + \gamma H(\xi) - \frac{1}{2}\| w \|^2.
\end{array}
\end{equation}
where $\gamma = \e \beta / 2\log n$. We call this problem a $\nu$-Saddle problem. Similar to
Lemma~\ref{lm:approx}, we can prove that $\nu$-Saddle~\eqref{eq:nusaddle2} is a $(1-\e)$-approximation
of the saddle point optimization~\eqref{eq:nusaddle}. See Lemma~\ref{lm:nuapprox} in
Appendix~\ref{apd:series} for the details.

Overall, we formulate hard-margin SVM
and $\nu$-SVM as saddle point problems and  prove that  through solving HM-Saddle and $\nu$-Saddle, we can solve hard-margin SVM
and $\nu$-SVM.\footnote{
  Some
  readers may wonder why the formulations of HM-Saddle and $\nu$-Saddle only depends on $(w, \eta, \xi)$ but not the
  offset $b$. In fact, according to the fact that the
  hyperplane bisects the
  closest points in the (reduced) convex hulls,
  it is not difficult to show that $b^* = w^{* \rm T}
  (A\eta^* + B \xi^*)/2$. }

\section{Saddle Point Optimization Algorithms for SVM}
\label{sec:svmsp}

In this section, we propose efficient algorithms to solve the two saddle point optimizations: HM-Saddle~\eqref{eq:saddle2}
and $\nu$-Saddle~\eqref{eq:nusaddle2}. The framework is inspired by the prior
work by~\cite{allen2016optimization}.  However, their algorithm does not imply an effective SVM algorithm directly
as discussed in Section \ref{sec:related}. We modify the \emph{update rules} and introduce new \emph{projection methods} to adjust the framework to the HM-Saddle and $\nu$-Saddle problems. We highlight that both the new update rules and projection methods are non-trivial.


First, we introduce a preprocess step to make the data vectors more homogeneous in each coordinate. Then, we explain the update rules and projection methods of our algorithm: Saddle-SVC.

For convenience, we  assume that in the hard margin case $\|x_i\|^2\leq 1$ for $ 1 \leq i \leq n$.~\footnote{ It can be
  achieved by scaling all data by factor $1/ \max \|x_i \|^2 $ in $O(nd)$  time.} Let
$W$ be the $d \times d$ Walsh-Hadamard matrix and $D$ be a $d \times d$ diagonal matrix whose
entries are i.i.d. chosen from $\pm 1$ with equal probability.  Then, we transform the data by
left-producting the matrix $WD$.  Then with high probability, for any point $x_i$ satisfied that~\citep{ailon2010faster}
\begin{equation*}
\forall j \in [d],  \; |(WDx_i)_j| \leq O(\sqrt{\log n /d}).
\end{equation*}
Let $X^{+} = WDA $ and $X^{-} = WDB$.  It means that  after transformation, with high probability, the value of
each entry in $X^{+}$ or $X^{-}$ is at most $O(\sqrt{\log n /d})$.  This
transformation can be completed in $O(n d \log d)$ time by FFT.
Note that $WD$ is an invertible matrix which represents a rotation and mirroring operation.
Hence, it does not affect the optima of the problem.
In fact, the ``Hadamard transform trick" has been used in the numerical analysis
literature explicitly or implicitly
(see e.g., \citep{frieze2004fast,juditsky2013randomized,allen2016optimization}).
Roughly speaking, the main purpose of the transform is
to make all coordinates of $X$ more uniform, such that
the uniform sampling (line 1 in Algorithm 2) is more efficient (otherwise, the large coordinates would have a disproportionate effect on uniform sampling).

After the data transformation, we define some necessary parameters. See Line 4 of Algorithm~\ref{alg:init} for details.
\footnote{
Careful readers may notice that $\gamma  = \e\beta/(2 \log n)$.
But $\beta$ is an unknown parameter, which is the ratio of the
minimum distance to the maximum one among the points.
The same issue also appears in the previous work~\citep{allen2016optimization}.
The role of $\beta$ is similar to the step size in the stochastic gradient descent algorithm.
In practice, we could try several $\beta = 10^{-k}$ for $k \in \mathbb{Z}$ and choose the best one.
}
We use ``$\alpha [t]$'' to represent the value of variable ``$\alpha$'' at iteration
$t$. For example, $w[0]$, $\eta[0], \xi[0]$ are the initial value of $w, \eta, \xi$ and are defined in Line 5 of Algorithm~\ref{alg:init}.



\begin{algorithm}[t]
  \caption{Pre-processing}
  \label{alg:init}
\begin{algorithmic}[1]

  \REQUIRE { $\calP$: $n_1$ points $x_i^+$ with label $+1$ and $\calQ$: $n_2$ points $x_i^-$ with label $-1$}

  \STATE $ W \leftarrow $ $d$-dimensional Walsh-Hadamard Matrix
  \STATE $D \leftarrow$ $d\times d$ diagonal matrix whose  entries are i.i.d. chosen from $\pm 1$
  \STATE $ X^{+} \leftarrow WD \cdot [x^{+}_1 , x^{+}_2, \ldots,  x^{+}_{n_1} ], $  $ X^{-} \leftarrow WD \cdot  [x^{-}_1 , x^{-}_2, \ldots,  x^{-}_{n_2} ] $

  \STATE $ \gamma \leftarrow \frac{\e\beta}{2 \log n}, $ $q \leftarrow
  O(\sqrt{\log n}),$  $ \tau  \leftarrow
  \frac{1}{2q}\sqrt{\frac{d}{\gamma}}, \sigma \leftarrow
  \frac{1}{2q}\sqrt{d\gamma}, \theta \leftarrow 1- \frac{1}{d + q\sqrt{d}/\sqrt{\gamma }} $

  \STATE $w[0] = \mathbf{0}^{\rm T},  \eta[-1]  = \eta[0] = \mathbf{1}^{\rm T}/ n_1, \xi[-1] = \xi[0]  = \mathbf{1}^{\rm T}/ n_2$

\end{algorithmic}
\end{algorithm}

\begin{algorithm}[t]
  \caption{Update Rules of Saddle-SVC}
  \label{alg:update}
  \begin{algorithmic}[1]
    \STATE  Pick an index $i^{*} $ in $[d]$ \text{  uniformly at random } \;
    \STATE $\delta^{+}_{i^*} \leftarrow  \langle X^+_{i^*},  \eta[t] + \theta(\eta[t]  -\eta[t-1]) \rangle$, \;
    \STATE $ \delta^{-}_{i^*} \leftarrow  \langle X^-_{i^*},  \xi[t] + \theta(\xi[t]  -\xi[t-1]) \rangle $ \;
    \STATE
    $\forall i \in [d], $ $  w_i[t+1] $     $ \leftarrow $
  $ \left\{
      \begin{array}{ll}
        (w_i[t] +  \sigma (\delta^{+}_i - \delta^{-}_i) )/ (\sigma  + 1 ), & \text{if } i = i^* \\
        w_i[t] , & \text{if } i \neq i^*
      \end{array} \right.
    $

    \STATE  $ \eta[t+1] \leftarrow $
     $\arg\min\limits_{\eta \in \calS_1} \{ \frac{1}{d} ( w[t] +d(w[t+1] - w[t]))^{\rm T} X^{+}\eta  $
     $ + \frac{\gamma}{d} H(\eta) + \frac{1}{\tau} V_{\eta[t]}(\eta) \}$ \; \\

    \STATE   $ \xi[t+1] \leftarrow $
    $\arg\min\limits_{\xi \in \calS_2} \{ - \frac{1}{d} ( w[t] +d(w[t+1] - w[t]))^{\rm T} X^{-}\xi $
    $  + \frac{\gamma}{d} H(\xi) + \frac{1}{\tau} V_{\xi[t]}(\xi) \}$
  \end{algorithmic}
\end{algorithm}

\topic{Update Rules}  In order to unify HM-Saddle and $\nu$-Saddle in the same
framework,  we use $ (\calS_1, \calS_2) $ to represent the domains $(\Delta_{n_1}, \Delta_{n_2})$ in HM-Saddle
(see formula \eqref{eq:saddle}) or $(\calD_{n_1}, \calD_{n_2})$ in $\nu$-Saddle (see formula \eqref{eq:nusaddle}).

Generally speaking, the update rules alternatively maximize the objective with respect to $w$ and minimize with
respect to $\eta $ and $\xi$. See the details in  Algorithm~\ref{alg:update}.

Firstly, we update $w$ according to Line 4 in Algorithm~\ref{alg:update}.  It is equivalent to a variant of the proximal coordinate gradient method with $l_2$-norm regularization  as follows.
\begin{align}
\label{eq:upw}
\begin{split}
  w_{i^*}[t+1] =& \arg \max_{w_{i^*}}  -\big\{  -(\delta_{i^*}^{+}    -  \delta_{i^*}^{-})w_{i^*} \\
& \qquad +  w^2_{i^*}/2 +   (w_{i^*} - w_{i^*}[t])^2 / 2 \sigma \big\}
\end{split}
\end{align}
We briefly explain the intuition of \eqref{eq:upw}.  Note that the term $ (\delta_{i^*}^{+} -
\delta_{i^*}^{-})$ in \eqref{eq:upw} can be considered as the term $\langle X^{+}_{i^*}, \eta[t]  \rangle  - \langle
X^{-}_{i^*}, \xi[t] \rangle$ adding an extra momentum term $\theta(\eta[t]  -\eta[t-1])$ and $\theta(\xi[t]  - \xi[t-1])$ for dual variable $\eta[t]$ and $\xi[t]$ respectively (see Line 2 and 3 in Algorithm~\ref{alg:update}). Further, $( \langle  X^{+}_{i^*}, \eta[t]  \rangle  - \langle
X^{-}_{i^*}, \xi[t] \rangle) w_{i^*} - w^2_{i^*}/2$ is the term in the objective function \eqref{eq:saddle2} and \eqref{eq:nusaddle2} which are related to $w$.
The $(w_{i^*} - w_{i^*}[t])^2 / 2$) is the $l_2$-norm regularization term.

Moreover, rather than update the whole $w$ vector, randomly selecting one dimension $i^{*} \in [d]$ and updating the
corresponding $w_{i^*}$ in each iteration can  reduce the runtime per round.

The update rules for $\eta$ and $\xi$ are listed in Line 5 and 6 in Algorithm~\ref{alg:update}, which are the proximal gradient method with a Bergman
divergence regularization  $V_{x}(y) =  H(y) - \langle \nabla H(x), y-x \rangle - H(x)$. Similar to $ (\delta_{i^*}^{+} -
\delta_{i^*}^{-})$ in \eqref{eq:upw}, we also add a momentum term $d(w[t+1] - w[t])$ for primal variable $w$
when updating $\eta$ and $\xi$.

\topic{Projection Methods} However, the update rules for $\eta$ and $\xi$ are implicit update rules.  We need to show that we can solve the corresponding optimization problems in line 5 and 6 of Algorithm~\ref{alg:update} efficiently. In fact, for both HM-Saddle and $\nu$-Saddle, we can obtain explicit expressions of these two optimization problems using the
method of Lagrange multipliers.

 First, we can solve the optimization problem for HM-Saddle (in Line 5 and 6) directly, and the explicit expressions for $\eta$ and $\xi$ are as follows.
\begin{align}
  \label{eq:HM}
  \begin{split}
 \eta_i[t+1] &\leftarrow  \Phi(\eta_i[t], X^{+}) / Z^+ ,  \; \forall i \in [n_1],  \\
  \xi_j[t+1]  & \leftarrow  \Phi(\xi_j[t], X^{-}) / Z^{-}, \; \forall j \in [n_2]
  \end{split}
\end{align}
 where $Z^+ $ and $ Z^{-} $ are normalizers that ensures  $\sum_{i} \eta_{i}[t+1] = 1$ and $\sum_{j} \xi_j[t+1] =1$, and
 \begin{align}
   \label{eq:phi}
   \begin{split}
 \Phi(\lambda_i, X)  =  &  \exp \big\{ ( \gamma + d\tau^{-1})^{-1} (d\tau^{-1} \log \lambda_i  - \\
   & \qquad \qquad y_i   \cdot \langle  w[t] +d(w[t+1] - w[t], X_{\cdot i} )  \rangle ) \big\}
   \end{split}
 \end{align}
Note that the factors $Z^{+}$ and $Z^{-}$ are used to project the value $ \Phi(\eta_i[t],
X^{+})$ and $ \Phi(\xi_j[t], X^{-})$ to the domains $\Delta_{n_1}$ and $\Delta_{n_2}$. The above update rules of $\eta$ and $\xi$ can be also considered as the multiplicative weight
update method (see \cite{arora2012multiplicative}).

Next, we consider $\nu$-Saddle. Compared to HM-Saddle, $\nu$-Saddle has extra constraints that $\eta_i, \xi_j \leq \nu$.  Thus, we need another  projection process~\eqref{eq:nu} to
ensure that $\eta[t+1]$ and $\xi[t+1]$ locate in domain $\calD_{n_1}$ and $\calD_{n_2}$ respectively.  For convenience, we only present the projection for $\eta$ here.  The projection for $\xi$ is similar. Let $\eta_i$  be $\Phi(\eta_i[t], X^{+})/ Z^{+}$.
\begin{equation}
  \label{eq:nu}
  \begin{array}{l}
    \mathbf{while} \quad \varsigma := \sum_{\eta_i > \nu} (\eta_i - \nu) \neq 0:\\
    \qquad \Omega=  \sum_{\eta_i < \nu} \eta_i  \\
    \qquad \forall i, \quad \mathbf{if} \; \eta_i \geq \nu, \quad \mathbf{then} \; \eta_i = \nu    \\
    \qquad \forall i, \quad \mathbf{if} \; \eta_i < \nu, \quad \mathbf{then} \; \eta_i = \eta_i (1 +\varsigma/ \Omega)
  \end{array}
\end{equation}
Note that there are at most $1/\nu$ (a constant) entries $\eta_i$ of value $\nu$ during the whole projection process. In each iteration, there must be at least 1 more entry $\eta_i=\nu$ since we make all entries $\eta_j>\nu$ equal to $\nu$ after the iteration. Thus, the number of iterations
in~\eqref{eq:nu} is at most $1/\nu$. By \eqref{eq:nu}, we project $\eta$ and $\xi$ to the
domains $\calD_{n_1}$ and $\calD_{n_2}$ respectively.

We claim that the result of projection
\eqref{eq:nu} is exactly the optimal solution in Line 5. The proof
is deferred to Appendix~\ref{apd:svmsp}. Thus, we need $O(n/\nu)$ time to compute $\eta[t+1]$. Since
we assume that $\nu$ is a constant, it only costs linear time. In practice, if $\nu$ is extremely small, we have
another update rule to get $\eta[t+1]$ and $\xi[t+1]$  in $O( n\log n)$ time. See Appendix \ref{apd:svmsp} for details.
Finally, we give our main theorem for our algorithm as follows. See the proof in Appendix~\ref{apd:cvg}.

\begin{theorem}
  \label{thm:convergence}
  Algorithm~\ref{alg:update} computes  $(1-\e)$-approximate solutions for HM-Saddle and
  $\nu$-Saddle by $\tilde{O}(d+\sqrt{d/\e \beta})$ iterations. Moreover, it takes $O(n)$ time for each iteration.
\end{theorem}

Combining with Lemmas~\ref{lm:svmtosp}, \ref{lm:approx} and~\ref{lm:eq},
we obtain
$(1-\e)$-approximate solutions for C-Hull and RC-Hull problems. Hence by strong duality,  we obtain
$(1-\e)$-approximations for hard-margin SVM and $\nu$-SVM in $\tilde{O}(n(d+\sqrt{d/\e \beta}))$ time.

\begin{theorem}
  \label{thm:hmnu}
A $(1-\e)$-approximation for either hard-margin SVM or $\nu$-SVM can be computed in
$\tilde{O}(n(d+\sqrt{d/\e \beta}))$ time.
\end{theorem}


\section{Distributed SVM}
\label{sec:dist}

\topic{Server and Clients Model} We extend Saddle-SVC to the distributed setting and call it
Saddle-DSVC. We consider the popular distributed setting: the \emph{server} and
\emph{clients} model. Denote  the server by $\server$. Let $\clientset$  be the set of clients and
$|\clientset| = k$. We use the notation $\client.\alpha$ to represent any variable $\alpha$
saved in client $\client$  and use $\server.\alpha$ to represent a variable $\alpha$ saved in the
server.

First, we initialize some parameters in each client as the
pre-processing step in Section~\ref{sec:svmsp}. Each client maintains the same random diagonal
matrix $D_{d\times d}$ and the total number of points in each type (i.e, $|\calP|=n_1$ and
$|\calQ|=n_2$).\footnote{It can be realized using $O(k)$ communication bits.}
Moreover, each client $\client$ applies a Hadamard
transformation to its own data and initialize the partial probability
vectors $\client.\eta$  and $\client.\xi$ for its own points.
\eat{
Formally speaking, assume there are $ m_1$ points $x^+_1, x^+_2, \ldots, x^+_{m_1}$
and $m_2$ points $x^-_1, x^-_2, \ldots, x^-_{m_2}$ maintained in $C$. We use $\mathbf{1}^{m}$ to
denote a vector with all components being $1$. The initialization is as follows.
\begin{align*}
 \client.X^{+} = WD \cdot [x^{+}_1 , x^{+}_2, \ldots,  x^{+}_{m_1} ], \; &\client.\eta[-1] = \client.\eta[0] = n^{-1}_1\mathbf{1}^{m_1} \\
\client.X^{-} = WD \cdot [x^{-}_1 , x^{-}_2, \ldots,  x^{-}_{m_2} ], \;  & \client.\xi[-1] = \client.\xi[0] =n_2^{-1}\mathbf{1}^{m_2}
\end{align*}
}

We first consider HM-Saddle. The interaction between clients and the server can be divided into three rounds in each iteration.
\begin{enumerate}
\item In the first round, the server randomly chooses a
    number $i^* \in [d]$ and broadcasts $i^*$ to all
    clients. Each client computes $\client.\delta_{i^*}^{+}$ and $\client.\delta_{i^*}^{-}$  and
sends them back to the server.
\item In the second round, the server sums up all $\client.\delta_{i^*}^{+}$ and $\client.\delta_{i^*}^{-}$
  and computes $ \server.\delta_{i^*}^{+} $ and $ \server.\delta_{i^*}^{-} $. We can see that
  $\server.\delta_{i^*}^{+}$ (resp. $\server.\delta_{i^*}^{-}$) is exactly $\delta_{i^*}^{+}$
  (resp. $\delta_{i^*}^{-}$) in  Algorithm~\ref{alg:update}. The server broadcasts $ \server.\delta_{i^*}^{+} $ and $
  \server.\delta_{i^*}^{-}$ to all clients. By $ \server.\delta_{i^*}^{+} $ and $ \server.\delta_{i^*}^{-}
  $, each client updates $w$ individually. Moreover, each client $\client\in \clientset$ updates its
  own $\client.\eta$ and $\client.\xi$ according to the new directional vector
  $w$. In order to normalize the probability vectors $\eta$ and $\xi$, each client sends the
  summation $\client.Z^+$ and $\client.Z^-$ to the server.
\item In the third round, the server computes $(\server.Z^{+}, \server.Z^{-}) \leftarrow \sum_{\client
    \in \clientset}(\client.Z^{+}, \client.Z^{-})$ and
  broadcasts to all clients the normalization   factors $\server.Z^+$  and  $\server.Z^-$. Finally, each client
    updates its partial probability vector $\client.\eta$ and $\client.\xi$ based on
    the normalization factors.
\end{enumerate}
As we discuss in Section~\ref{sec:svmsp}, for $\nu$-Saddle, we need another $O(1/\nu)$
rounds to project $\eta$ and $\xi$ to the domains $\calD_{n_1}$ and $\calD_{n_2}$.

\begin{enumerate}
 \setcounter{enumi}{3}
\item Each client computes $\client.\varsigma^{+}, \client.\varsigma^{-}
  $ and $ \client.\Omega^{+}  $,  $\client.\Omega^{-} $ according to~\eqref{eq:nu} and sends them to the
  server. The server
  sums up all  $\client.\varsigma^{+}, \client.\varsigma^{-}, \client.\Omega^{+}, \client.\Omega^{-}$
  respectively and gets $\server.\varsigma^{+}, \server.\varsigma^{-}, $ $\server.\Omega^{+},
  \server.\Omega^{-}$. If both $\server.\varsigma^{+} $ and $ \server.\varsigma^{-}$ are zeros, the server stops this
  iteration. Otherwise, the server broadcasts to all clients the  factors  $\server.\varsigma^{+},
  \server.\varsigma^{-}, \server.\Omega^{+},  \server.\Omega^{-}$. All clients update their
  $\client.\eta$ and $\client.\xi$ according to~\eqref{eq:nu} and repeat Step 4  again.

\end{enumerate}

We give the pseudocode in Algorithm~\ref{alg:dist} in Appendix~\ref{apd:dist}.
By Theorem~\ref{thm:convergence}, after $ T = \tilde{O}( d +
\sqrt{d/\e})$ iterations, all clients compute the same $(1-\e)$-approximate solution $w = w[T]$
for SVM. W.l.o.g, let the first client send $w$ to the server. By at most $O(n)$ more communication cost, the server can compute the offset $b$, the margin for hard-margin SVM and the objective value for the $\nu$-SVM. The correctness of
Algorithm Saddle-DSVC is oblivious since we obtain the same $w[t]$ as in
Saddle-SVC after each iteration.

\eat{
Then each client $\client \in \clientset$
sends the two values $\min_{X_i^+\in \client} \{w^{\rm T} X_i^+\}$ and $\max_{X_j^-\in \client}
\{w^{\rm T} X_j^-\}$ to the server. Finally, the server computes $\min_{i\in [n_1]} \{w^{\rm
  T} X_i^+\}$ and $\max_{j\in [n_2]} \{w^{\rm T} X_j^-\}$, and computes the margin
$\rho$.
}

\topic{Communication Complexity of Saddle-DSVC}
Note that in each iteration of Algorithm~\ref{alg:dist}, the server and clients interact three
times for hard-margin SVM and $O(1/\nu)$ times for $\nu$-SVM. Thus, the communication cost of each
iteration is $O(k)$. By Theorem~\ref{thm:convergence}, it takes $\tilde{O}(d+\sqrt{d/\e})$ iterations.  Thus, we have the following theorem.
\begin{theorem}
  \label{thm:commcost}
  The communication cost of Saddle-DSVC is $\tilde{O}(k(d +
  \sqrt{d/\e}))$.
\end{theorem}

Liu et el. \cite{liu2016distributed} prove that the lower bound of the communication cost for  distributed SVM is $\Omega(k \min \{d, 1/\e\})$.
\begin{theorem} [Theorem 6 in~\cite{liu2016distributed}]
  \label{thm:bound}
  Consider a set of $d$-dimension points distributed at $k$ clients.
  The communication cost to achieve a $(1-\e)$-approximation of the distributed SVM problem is at least $\Omega(k \min\{ d,  1/\e \} )$
  for any $\e>0$.
\end{theorem}

If $ d = \Theta( 1/\e)$, the communication lower bound is $\Omega ( k (d + \sqrt{d/\e}))$ which matches the
communication cost of  Saddle-DSVC.

\eat{
claim that  the communication cost of
\textsf{Saddle-DSVC} is $\tilde{O}(k(d +  \sqrt{d/\e}))$, and

We defer the details to in Appendix~\ref{apd:dist}.

Liu et al.~\cite{liu2016distributed} proved a theoretical lower bound of the communication cost for distributed SVM as follows.\footnote{ Note that the statement of
  Theorem~\ref{thm:bound} is not exactly the same as the Theorem 6 in Liu et
  al.~\cite{liu2016distributed}. We prove that they are equivalent in Appendix~\ref{apd:dist}.} Note that
if $ d = \Theta( 1/\e)$, the
communication lower bound is $\Omega ( k (d + \sqrt{d/\e}))$ which matches the
communication cost of our algorithm \textsf{DisSVMSPSolver}.

\begin{theorem} [Theorem 6 in~\cite{liu2016distributed}]
  \label{thm:bound}
  Consider a set of $d$-dimension points distributed at $k$ clients.
  The communication cost to achieve a $(1-\e)$-approximation of the distributed SVM problem is at least $\Omega(k \min\{ d,  1/\e \} )$
  for any $\e>0$.
\end{theorem}
}

\begin{figure}[t]
  \centering
  \subfigure[a9a, n=32561, d=123 ]{
    \includegraphics[width = 0.24\textwidth]{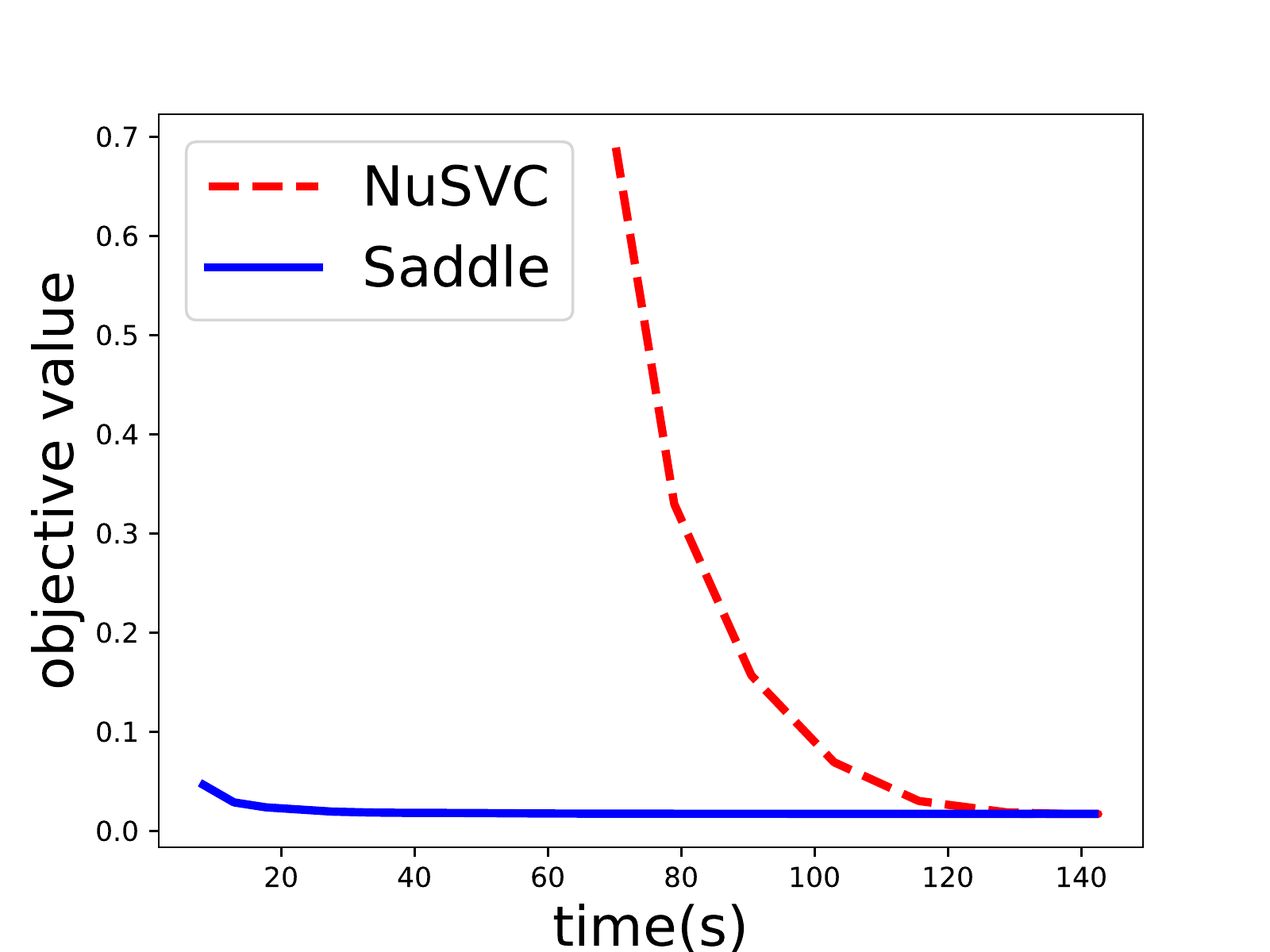}
    \includegraphics[width = 0.24\textwidth]{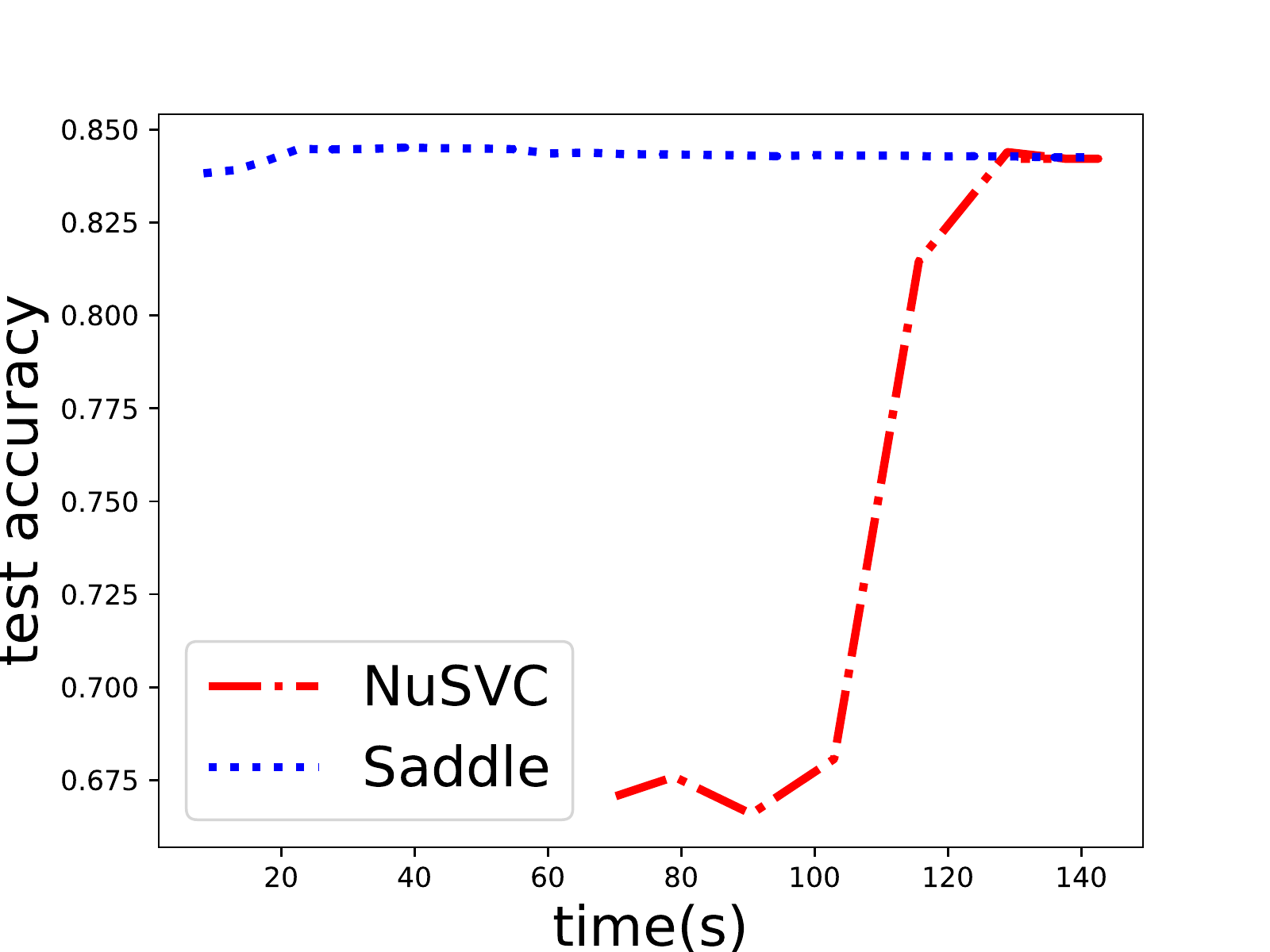}}
  \subfigure[ijcnn1: n=49990,d=22 ]{
    \includegraphics[width = 0.24\textwidth]{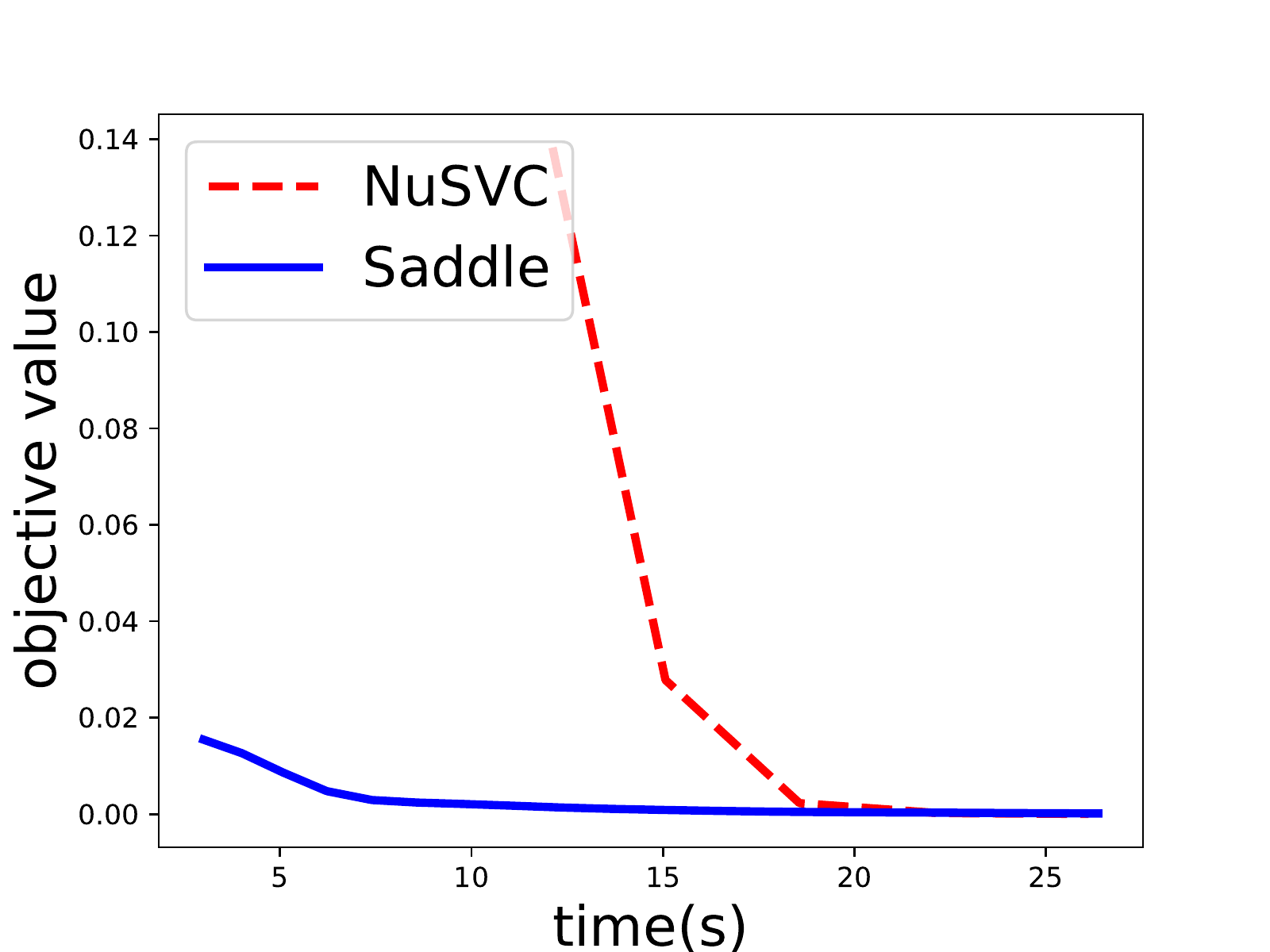}
    \includegraphics[width = 0.24\textwidth]{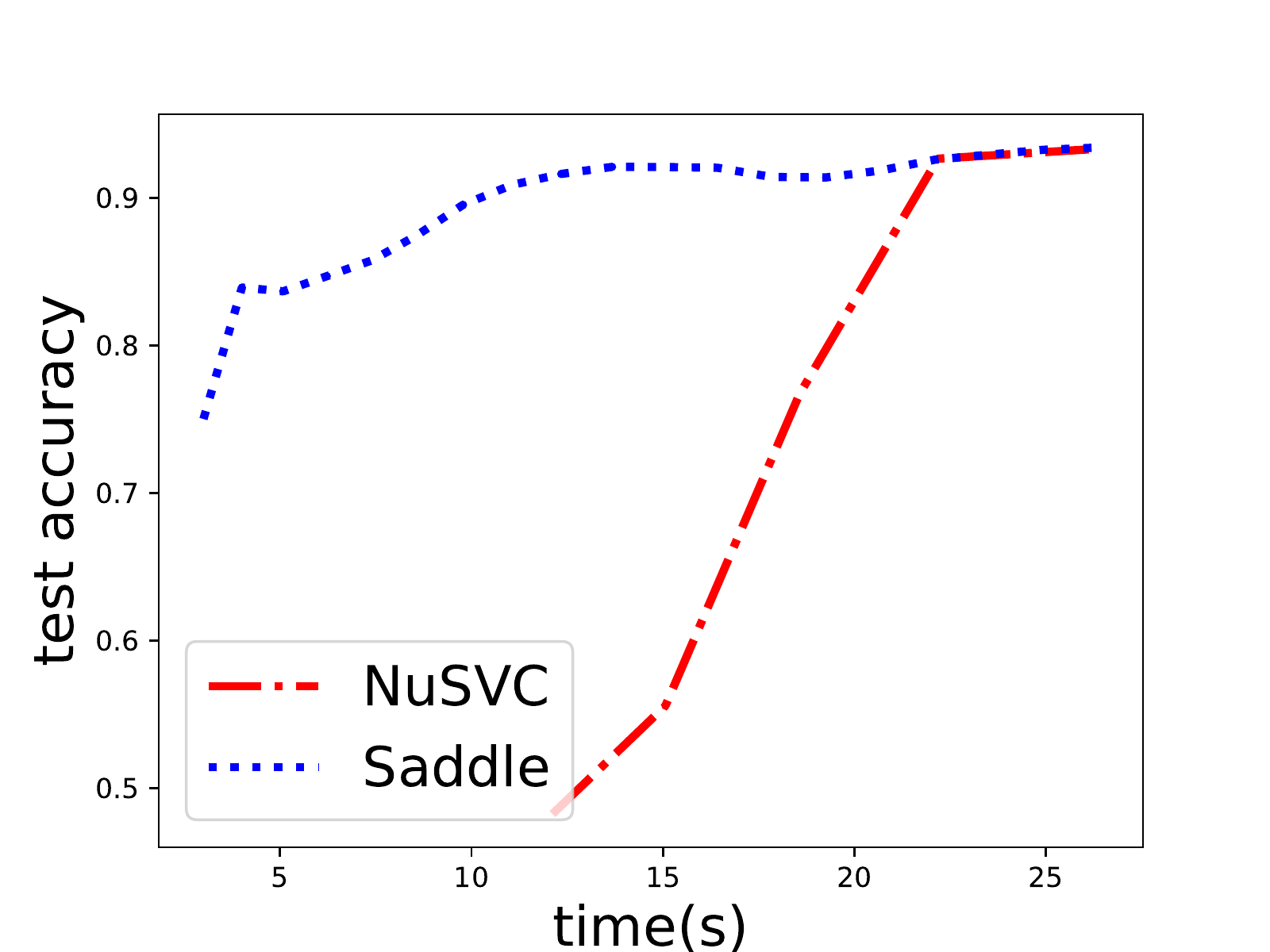}}
  \subfigure[phishing: n=11055, d=68 ]{
    \includegraphics[width = 0.24\textwidth]{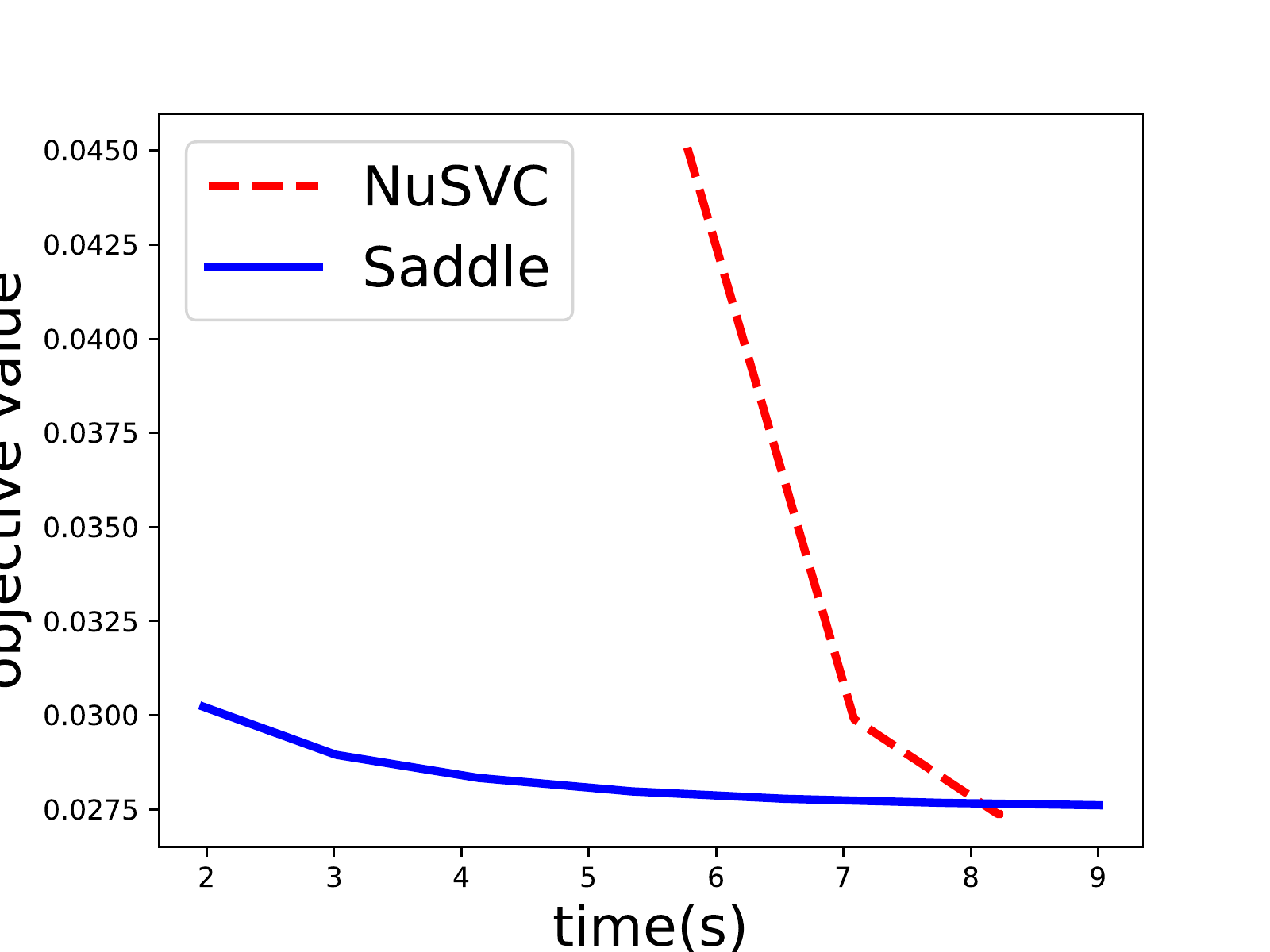}
    \includegraphics[width = 0.24\textwidth]{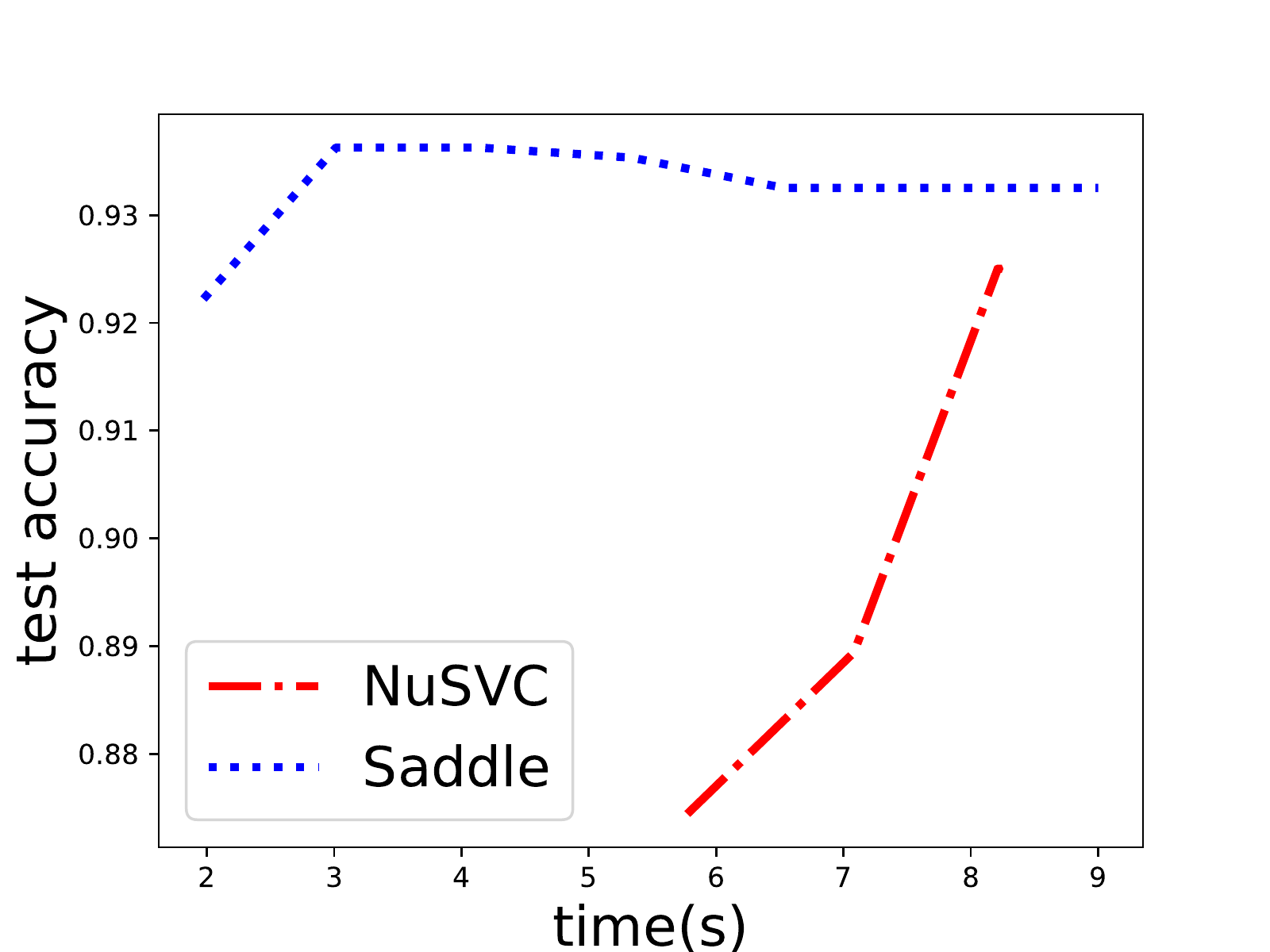}}
      \subfigure[skin\_nonskin: n=245057, d=3 ]{
    \includegraphics[width = 0.24\textwidth]{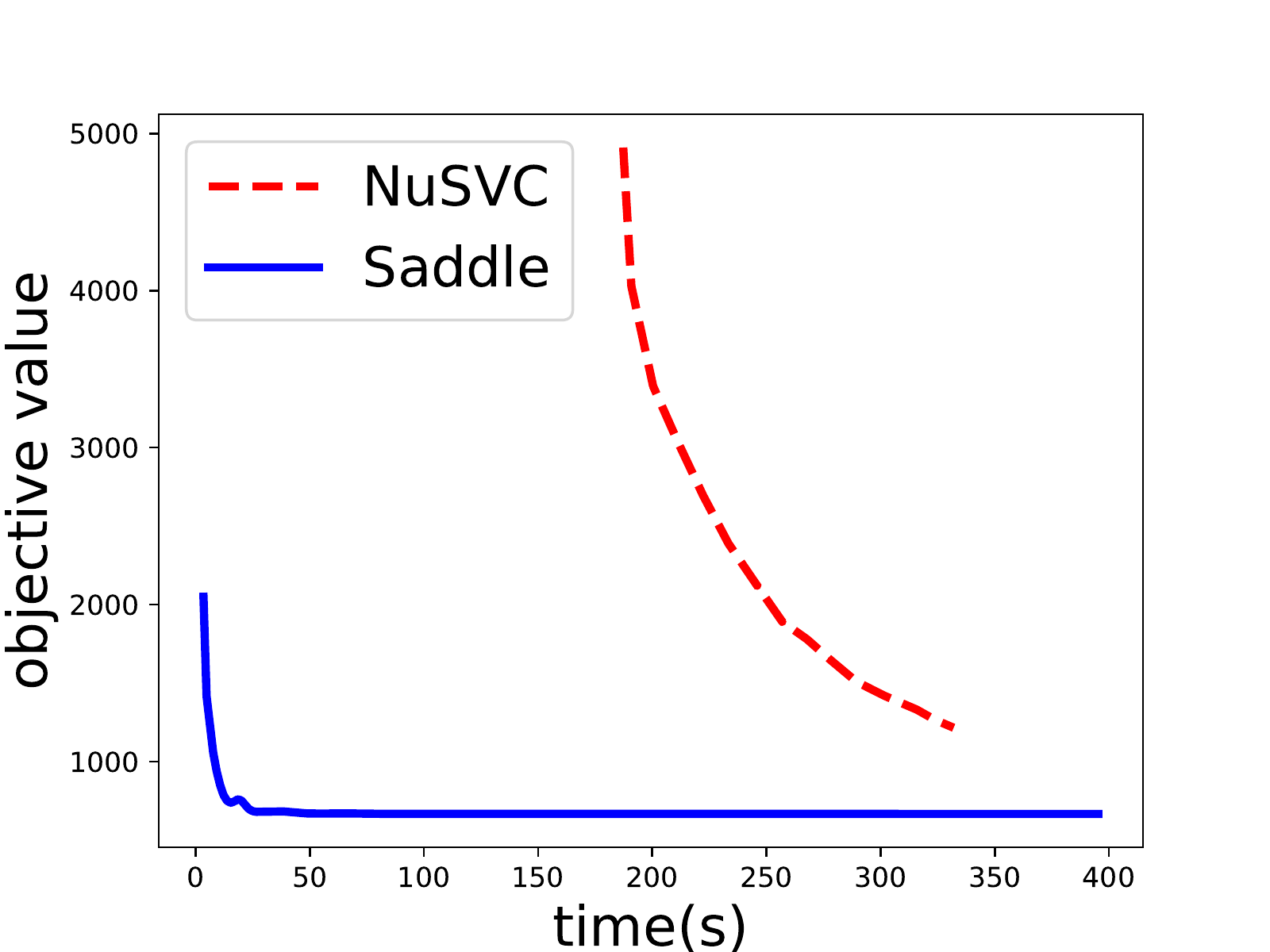}
     \includegraphics[width = 0.24\textwidth]{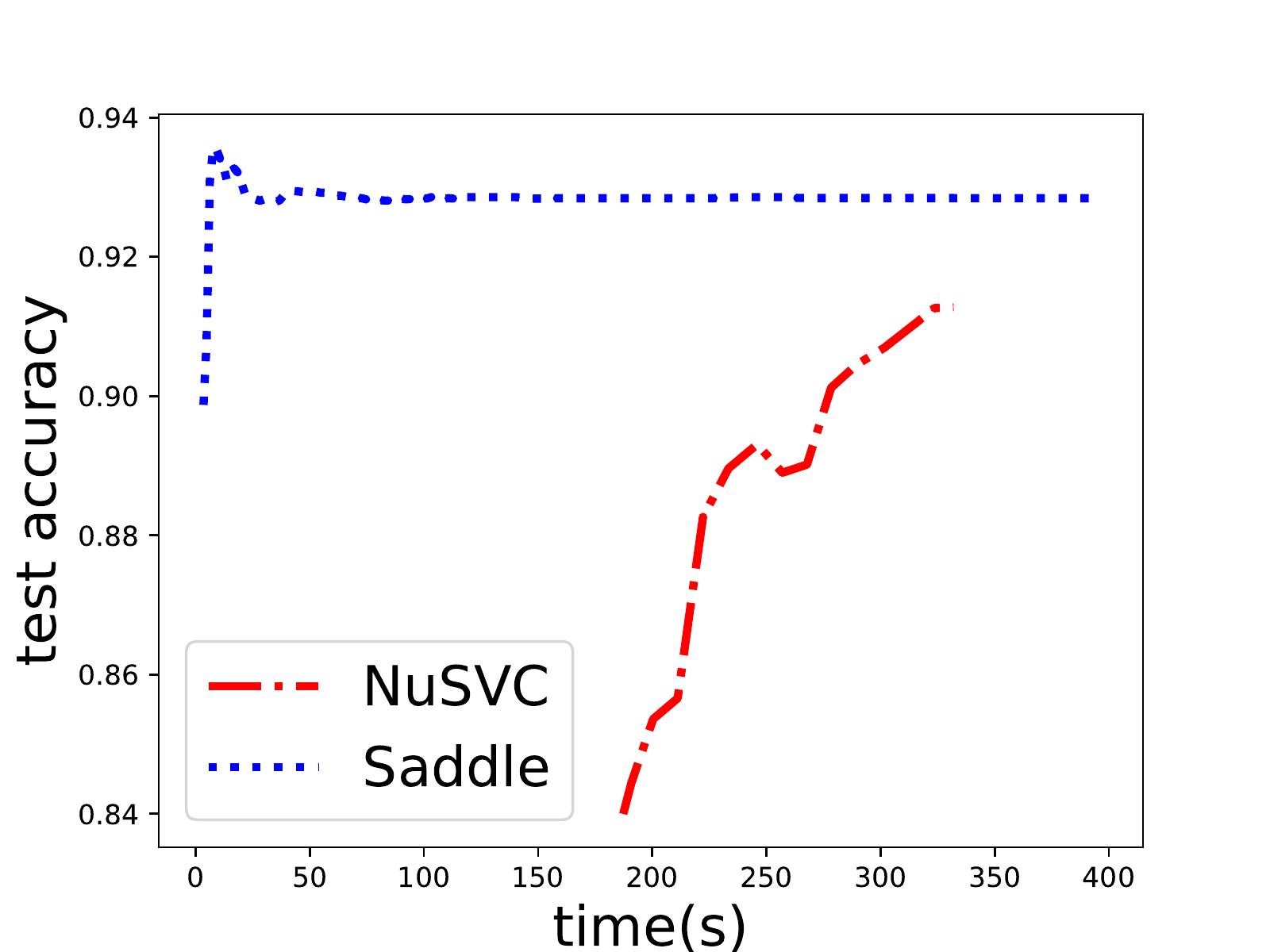}}
  \caption{ Saddle-SVC vs. NuSVC: the left side figures illustrate the convergence of objection value and the right side figures illustrate the test accuracy along with time. }
  \label{fig:nusvc}
\end{figure}

\begin{figure}[t]
  \centering
    \includegraphics[width = 0.32\textwidth]{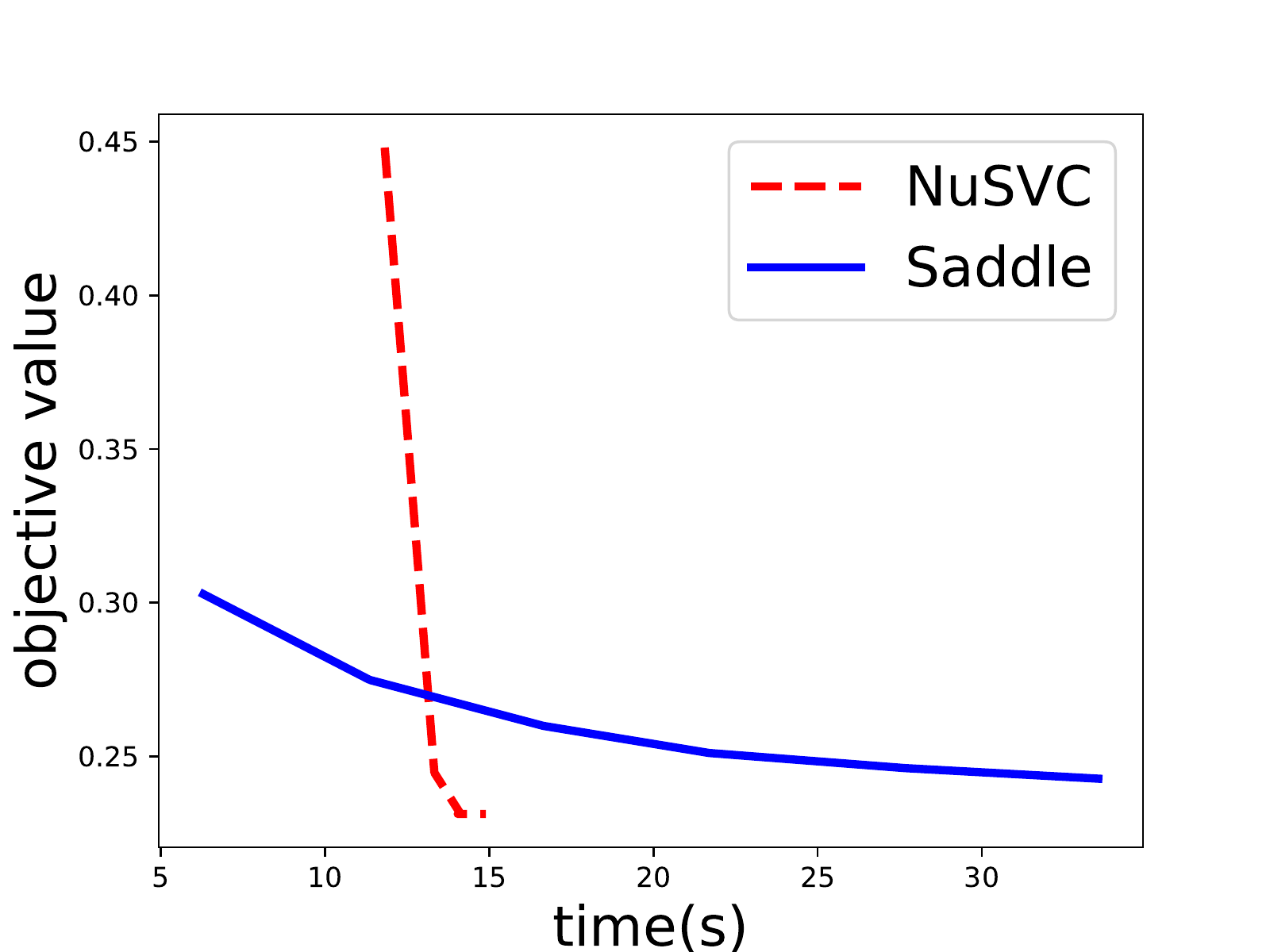}
    \includegraphics[width = 0.32\textwidth]{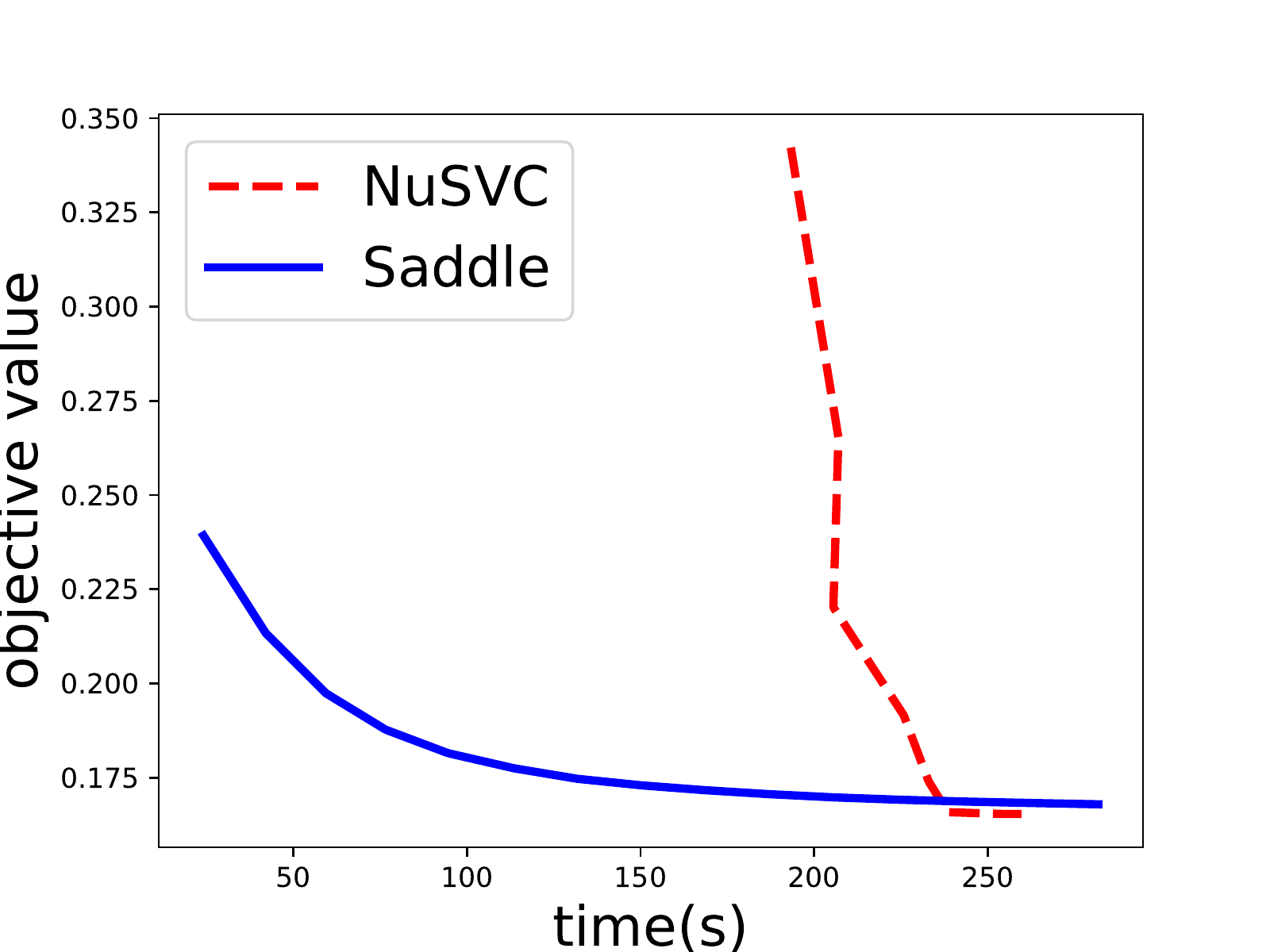}
    \includegraphics[width = 0.32\textwidth]{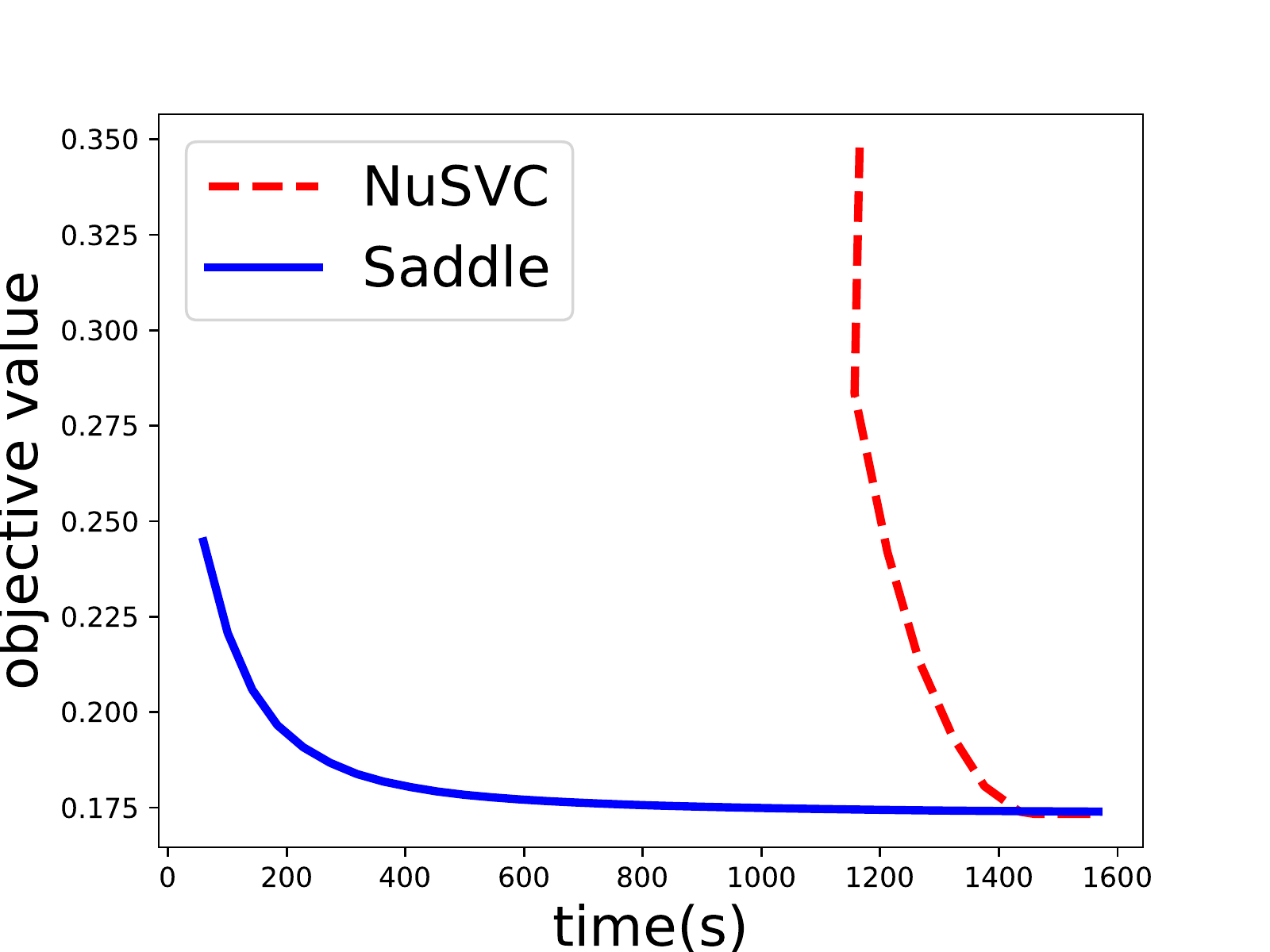}
  \caption{Saddle-SVC converges relatively faster than NuSVC along with data size increasing. The dimension of synthetic dimension $d = 512$. The data sizes are $ n = 5000, 20000, 50000$ respectively from left to right.}
  \label{fig:nusvcth}
\end{figure}

\section{Experiments}
\label{sec:exp}
In this section, we analyze the performance of Saddle-SVC and Saddle-DSVC for both $\nu$-SVM and hard-margin SVM.

First, we compare Saddle-SVC for $\nu$-SVM with NuSVC in scikit-learn \citep{pedregosa2011scikit}.
Current best $\nu$-SVM solver is based on quadratic programming.  NuSVC is one
of the fastest QP-based realization, which based on the famous SVM library LIBSVM~\citep{chang2011libsvm}.
We compare Saddle-SVC with NuSVC and show that when the two reduced polytopes are linearly separable under the parameter $\nu$, Saddle-SVC converges faster than NuSVC, especially when the data size is large and dense.
As a supplement, in Appendix~\ref{apd:exp}, we also compare Saddle-SVC for $\nu$-SVM with LinearSVC in scikit-learn based on LIBLINEAR~\citep{fan2008liblinear} which is the current best algorithm for linear kernel C-SVM and $l_2$-SVM.\footnote{However, we should note that LinearSVC is used to process  $C$-SVM and $l_2$-loss SVM, but not $\nu$-SVM or hard-margin SVM. Thus, their objective function are incomparable. We compare the test accuracy instead of the objective value.} We show that for the large and dense data set, Saddle-SVC is comparable to LinearSVC and even better.

 %

Next, we compare Saddle-SVC for hard-margin SVM with Gilbert algorithm~\citep{gilbert1966iterative}.  Gilbert algorithm is the current best algorithm for hard-margin SVM.  We show that Saddle-SVC converges faster when the data dimension is large.

On the other hand, we also implement our algorithm in the distributed setting and compare it with distributed Gilbert algorithm~\citep{liu2016distributed} and HOGWILD!~\citep{recht2011hogwild}.
We note that the current best distributed algorithm for hard-margin SVM is distributed Gilbert algorithm~\citep{liu2016distributed}.
Our experiments indicate that  Saddle-DSVC has lower communication cost in practice. On the other hand, there is no  practical distributed algorithm for $\nu$-SVM so far. Our algorithm is the first distributed algorithm for $\nu$-SVM.  To evaluate the performance of our distributed algorithm, we first show the convergence curve of Saddle-DSVC on some common datasets. As a supplement, in Appendix~\ref{apd:exp},  we also compare the convergent rate with HOGWILD!~\citep{recht2011hogwild}.~\footnote{Note that HOGWILD!  is used to solve $C$-SVM or $l_2$-SVM.} We show that Saddle-DSVC converges faster than HOGWILD! w.r.t. communication cost.

The CPU of our platform is Intel(R) Xeon(R) CPU E5-2690 v3 @ 2.60GHz, and the system is CentOS Linux. We use both synthetic and
 real-world data sets. The real data sets are from~\cite{chang2011libsvm}. See  Appendix~\ref{apd:exp} for the way to generate synthetic data. In each experiment, we mainly care about the performance of algorithms w.r.t. $n$ and $d$ since they are data dependent parameters. 

\topic{Saddle-SVC vs. NuSVC}\footnote{Note that NuSVC uses another equivalent form of $\nu$-SVM. The paramater $\mu$ in NuSVC equals $2/n\nu$ for $\nu$ in \eqref{eq:nusvm}. See details in Appendix \ref{apd:exp}.}
Here we use the data sets  ``a9a'', ``ijcnn1'', ``phishing'' , and ``skin\_nonskin" from \cite{chang2011libsvm}.
Note that ``a9a'', ``ijcnn1''  has the corresponding test set ``a9a.t'', ``ijcnn1.t''. For ``phishing'' and ``skin\_nonskin", we random choose $10\% $ data as the test set and let the remaining part be the training set.   Let
$$
\nu = 1 / (\alpha \min (n_1, n_2)),
$$
and set $\alpha = 0.85$ for $\nu$-SVM.
We show the experiment results in Figure~\ref{fig:nusvc} and  we can see that Saddle-SVC converges faster with the similar test accuracy. Our algorithm performs much better when the data size is large. We show that the results in Figure~\ref{fig:nusvcth} based on synthetic data sets sampling from the same distribution with different sizes.

We discuss a bit more for the parameter selection.  Chang and Lin \cite{chang2001training} show that $\nu$-SVM, $\nu \leq 1$ is feasible  $ \nu$ should larger than  $ 1 / \min (n_1, n_2) $ where $n_1$ and $n_2$ are the number of the two classes of points respectively. Moreover, if $\nu$ is too close to $1$, the $\nu$-SVM has poor prediction ability because of the two reduced polytopes may not separable. We discuss the detail reasons in Appendix~\ref{apd:exp}. We find that, in the experiment, $\alpha > 0.7$  usually ensures that the two reduced polytopes are linearly separable, i.e., the objective function converges to a positive number.   In Appendix~\ref{apd:exp}, we also do experiments for other $\alpha$s and show that if  $\alpha$ is small, $\nu$-SVM model has poor prediction ability.

\topic{Saddle-SVC vs. Gilbert Algorithm}
For the hard-margin SVM, we compare Saddle-SVC  with Gilbert Algorithm.  We use linearly separable data ``iris" and ``mushrooms". Since it is hard to find a large real data set which is linearly separable, we generate some synthetic data sets  and show that Saddle-SVC converges faster when data dimension is large. We repeat the iterations of Saddle-SVC and compute the objective function every $T$ rounds. If  the difference between two consecutive  objective value  is less than $\epsilon$, then output the results. See the results in Table~\ref{tab:gilbert}, in which we can see that Saddle-SVC gets smaller objective value (the closest  distance between the two polytopes) with less running time when data dimension is large.
\begin{table}[t]
  \centering
  \caption{Saddle-SVC vs. Gilbert Algorithm.  $\epsilon = 0.001$. iris: $n = 150, $ $d = 4$. mushrooms: $n = 8124, $ $d = 112$, Synthetic data: $n = 10000$.    }
  \label{tab:gilbert}
  \begin{tabular}{| c | c | c | c | c |}
  \hline
      \multirow{2}*{data set} &   \multicolumn{2}{|c|}{Saddle-SVC}  &   \multicolumn{2}{|c|}{Gilbert }  \\
    \cline{2-5}
  & obj & time   &  obj & time   \\
  \hline
  iris & 0.835 & 0.152s & 0.835 & 0.0005s \\
  \hline
  mushrooms & 0.516 & 11.2s & 0.517 & 12.5s \\
  \hline
  \end{tabular}

  \begin{tabular}{ | c | c | c | c | c |}
    \hline
    synthetic data &   \multicolumn{2}{|c|}{Saddle-SVC}  &
    \multicolumn{2}{|c|}{Gilbert }  \\
    \cline{2-5}
	 dimension & obj & time   &  obj & time   \\
      \hline
     8 & 0.395 & 9.33s &  0.397 & 1.52s    \\
    \hline
     32 & 0.468 & 28.8s &  0.470 & 10.1s   \\
     \hline
      128 & 0.436 & 64.0s &  0.438 & 152s   \\
      \hline
      512 & 0.496 & 189s & 0.498 & 2327s \\
    \hline
  \end{tabular}
\end{table}

\topic{Saddle-DSVC}  For hard-margin SVM, we compare Saddle-DSVC  with  distributed Gilbert algorithm.
We compare the margins  w.r.t.  the communication cost. We count all information communication between the clients and server as the communication cost.
The data sets are   ``mushrooms'' and synthetic data sets with different dimensions. Figure~\ref{fig:hm-dist} illustrates that Saddle-DSVC converges
faster w.r.t. communication cost.  The data is distributed to $k=20$ nodes.  Note that it takes $kd$ communication cost if each client sends a point to the server. We set one
unit of $x$-coordinate to represent $kd$ communication cost.

\begin{figure}[t]
  \centering
  \subfigure[synthetic data, \newline $d = 128, $ $n =2\times 10^4$]{
    \includegraphics[width = 0.45\textwidth]{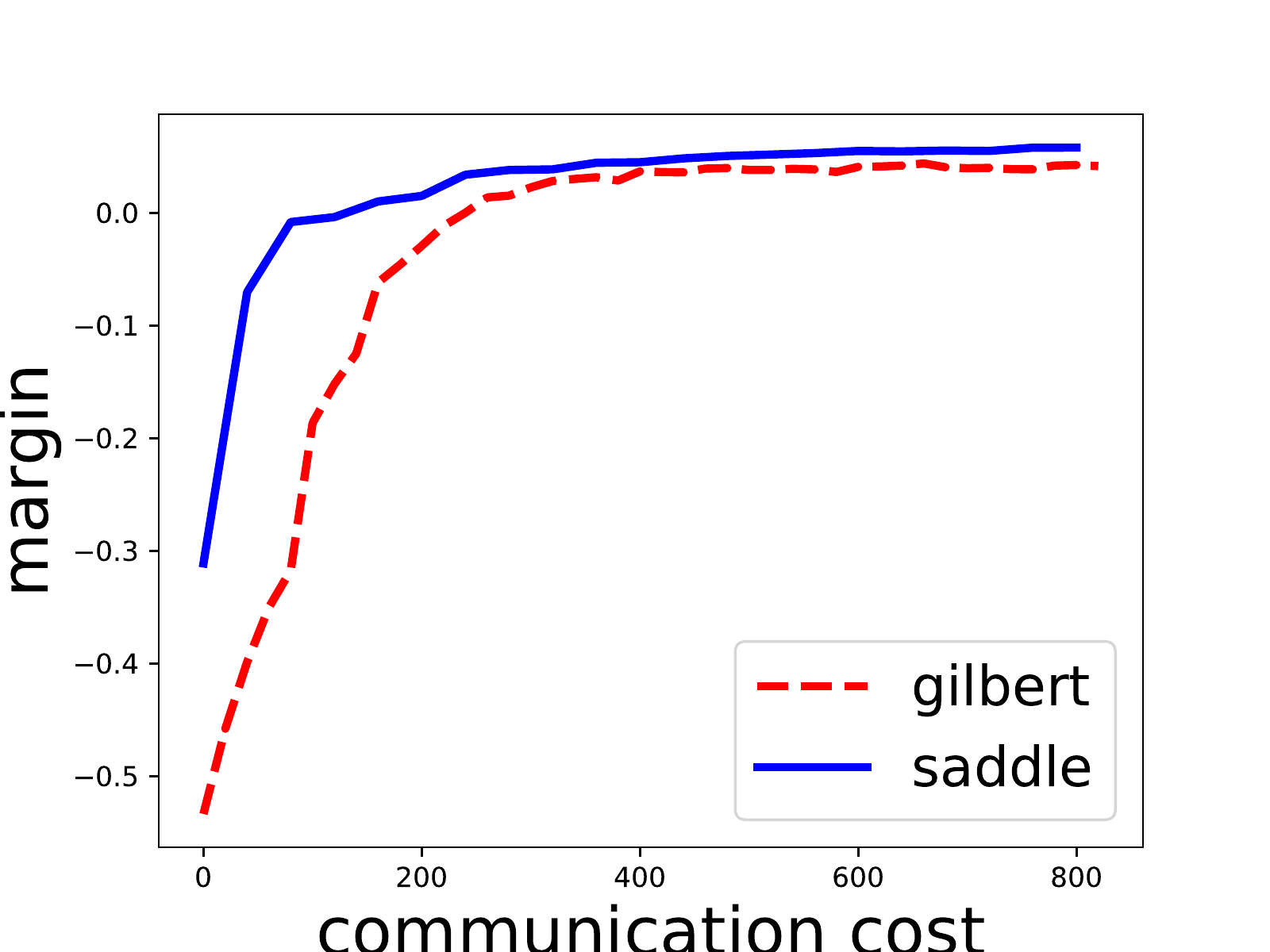}}
  \subfigure[synthetic data, \newline $d = 256, $ $n =2\times 10^4$]{
    \includegraphics[width = 0.45\textwidth]{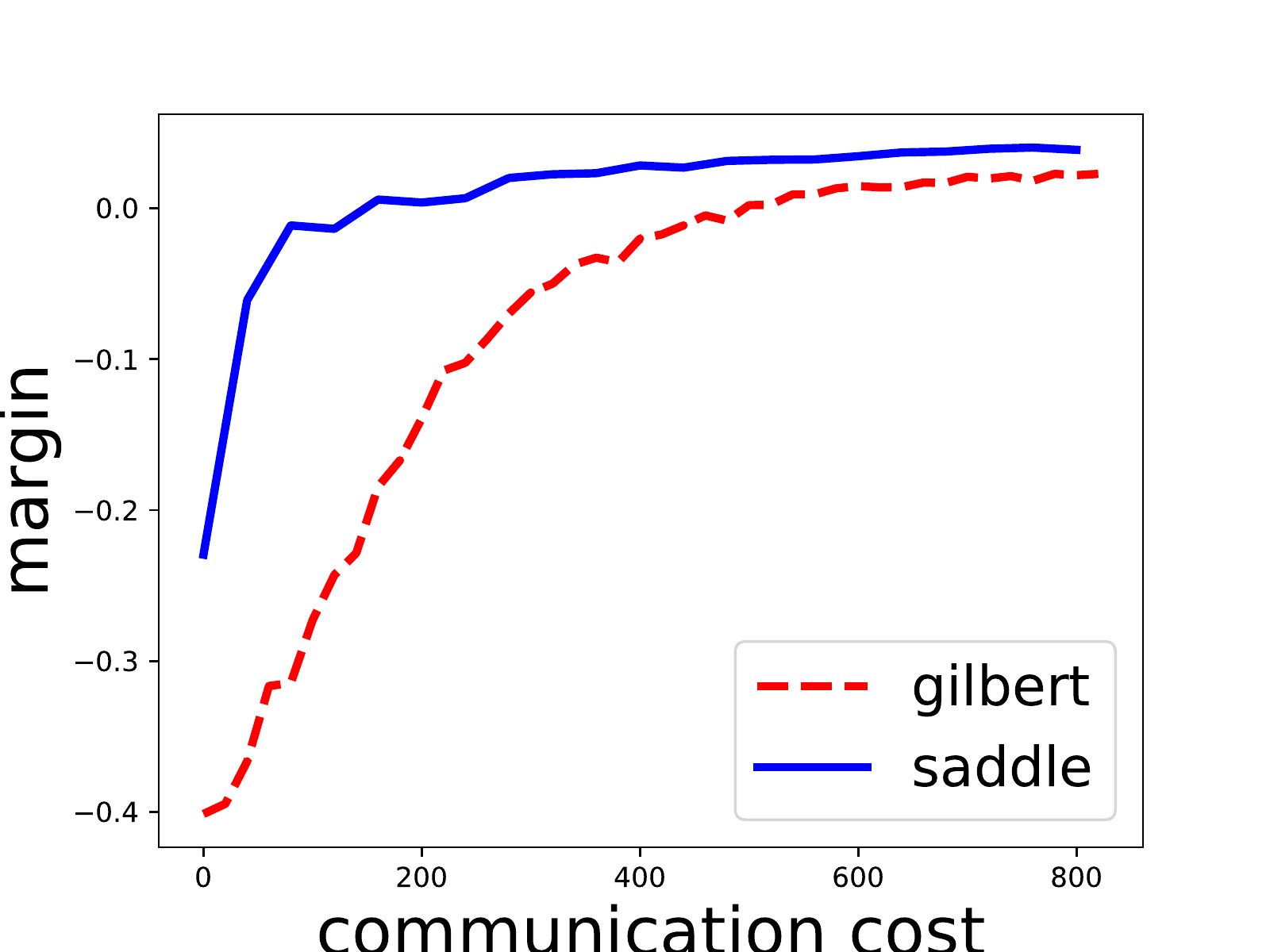}}
  \subfigure[synthetic data, \newline $d = 512, $ $n =2\times 10^4$]{
    \includegraphics[width = 0.45\textwidth]{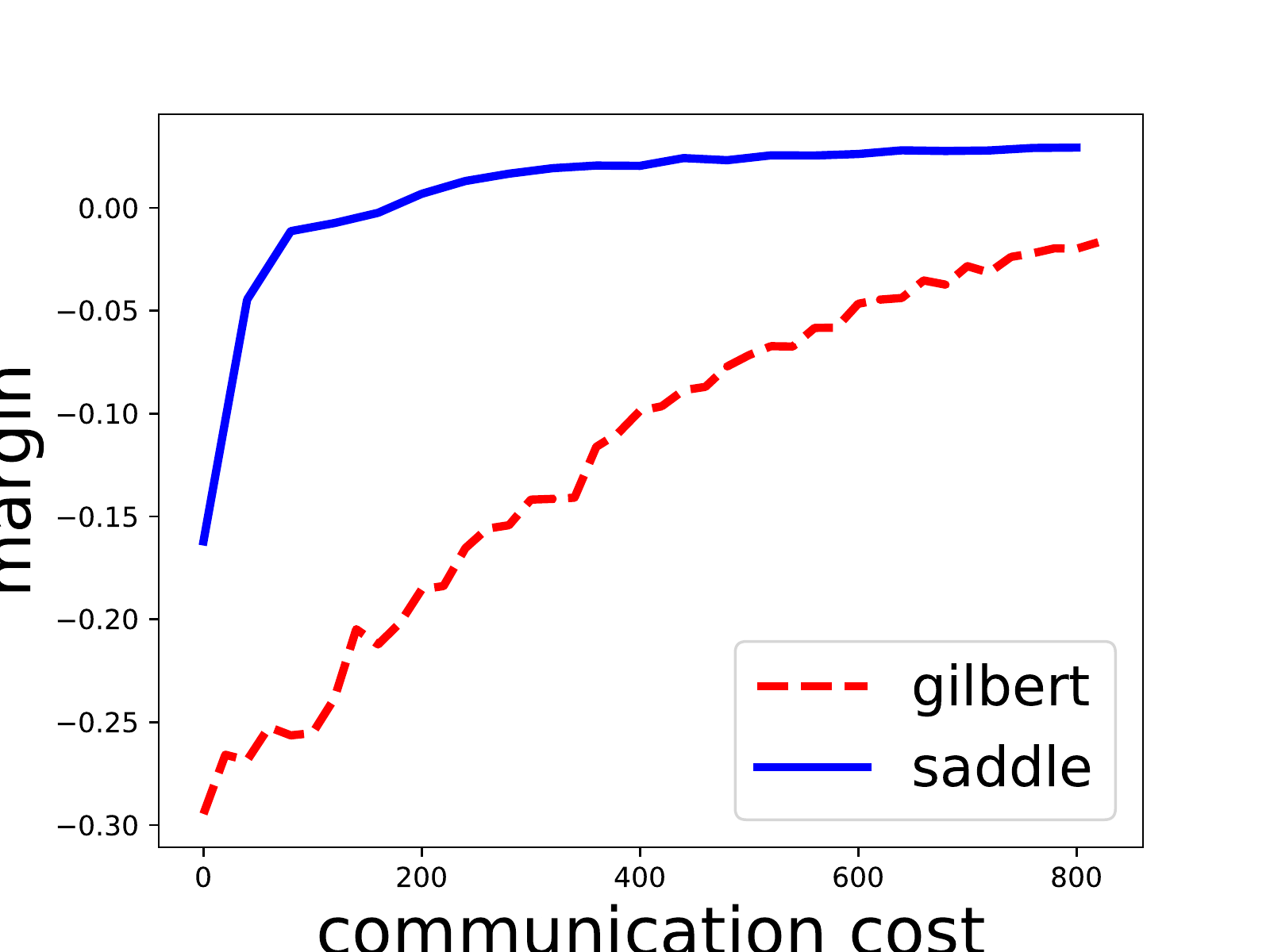}}
  \subfigure[mushrooms, \newline $d = 112, $ $ n = 8124$]{
    \includegraphics[width = 0.45\textwidth]{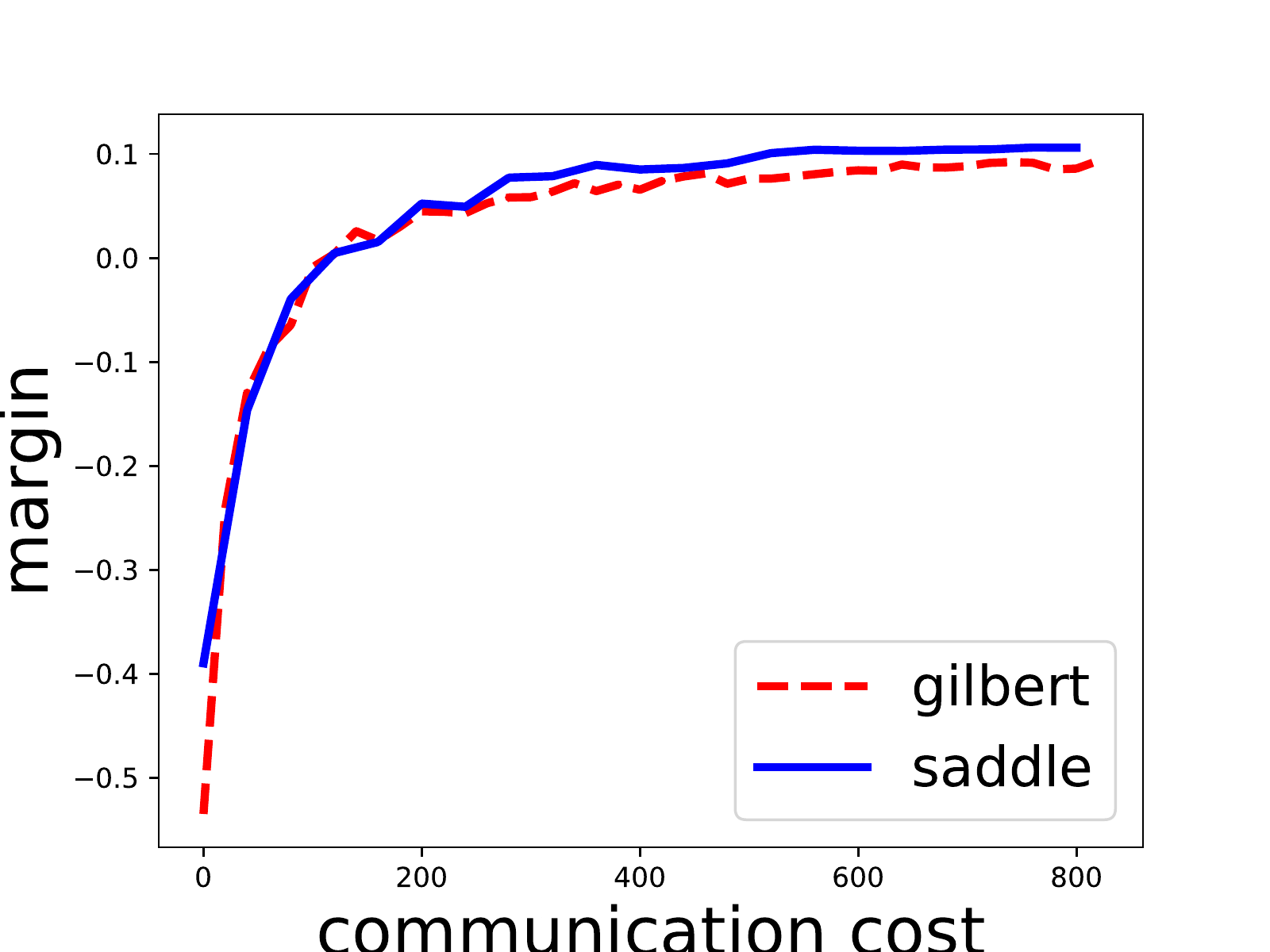}}
  \caption{The margins solved by Gilbert and  Saddle-DSVC w.r.t. communication cost. Here  $k=20$ and one unit is $kd$ communication cost. }
  \label{fig:hm-dist}
\end{figure}

For the $\nu$-SVM,  we analyze the convergence property on some common data sets
including ``phishing'', ``a9a'', ``gisette", ``madelon" from~\cite{chang2011libsvm}. We show the details in Figure~\ref{fig:obj}.
Besides, we also compare Saddle-DSVC with HOGWILD!.  We compare the accuracy instead of objective value since they solve different SVM variants.
We provide the experiment details in Appendix~\ref{apd:exp} and show that our algorithm is convergent faster  w.r.t. communication cost.

\begin{figure}[t]
\centering
  \subfigure[gisette]{
    \includegraphics[width = 0.45\textwidth]{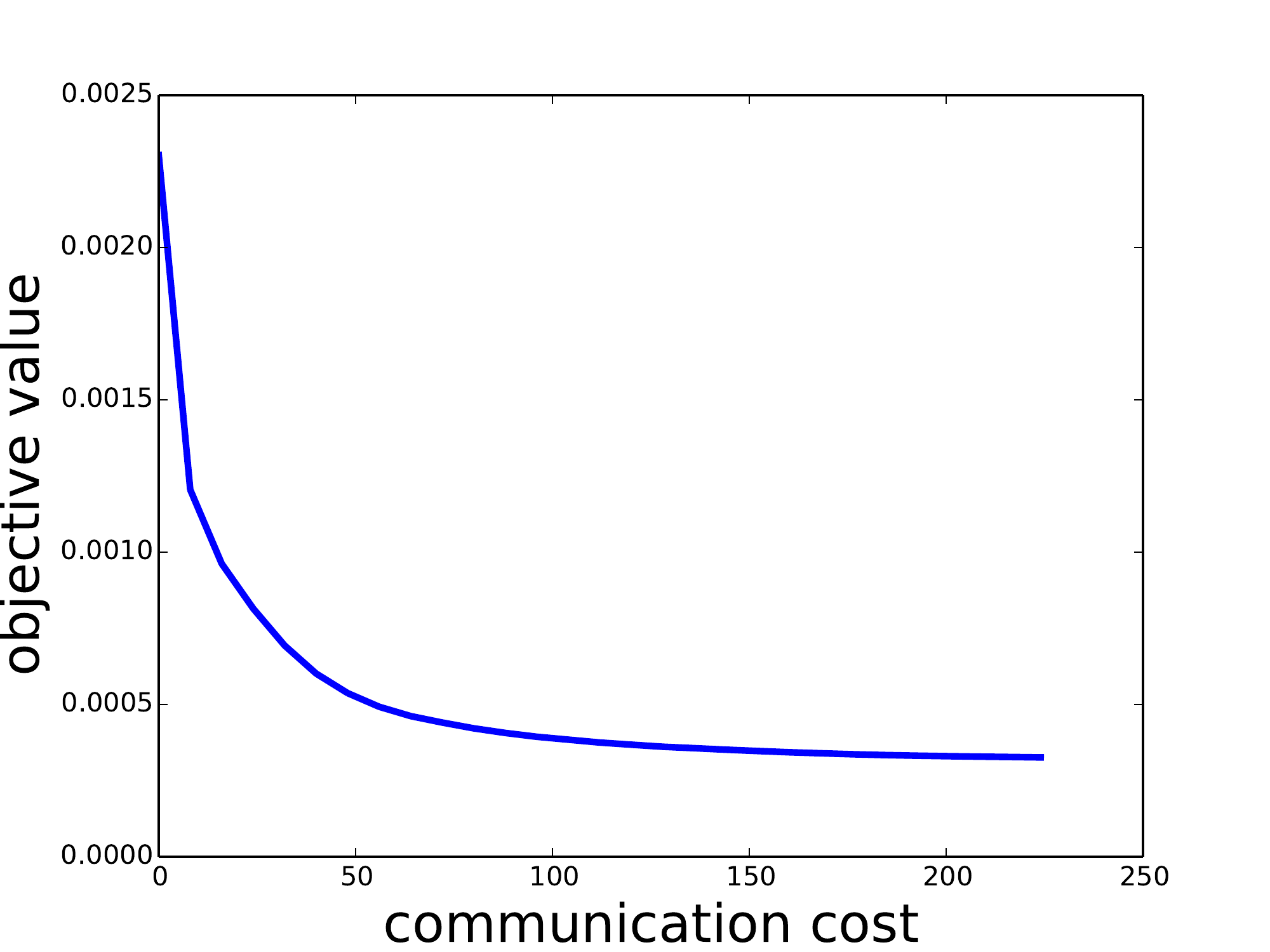}}
  \subfigure[madelon]{
    \includegraphics[width = 0.45\textwidth]{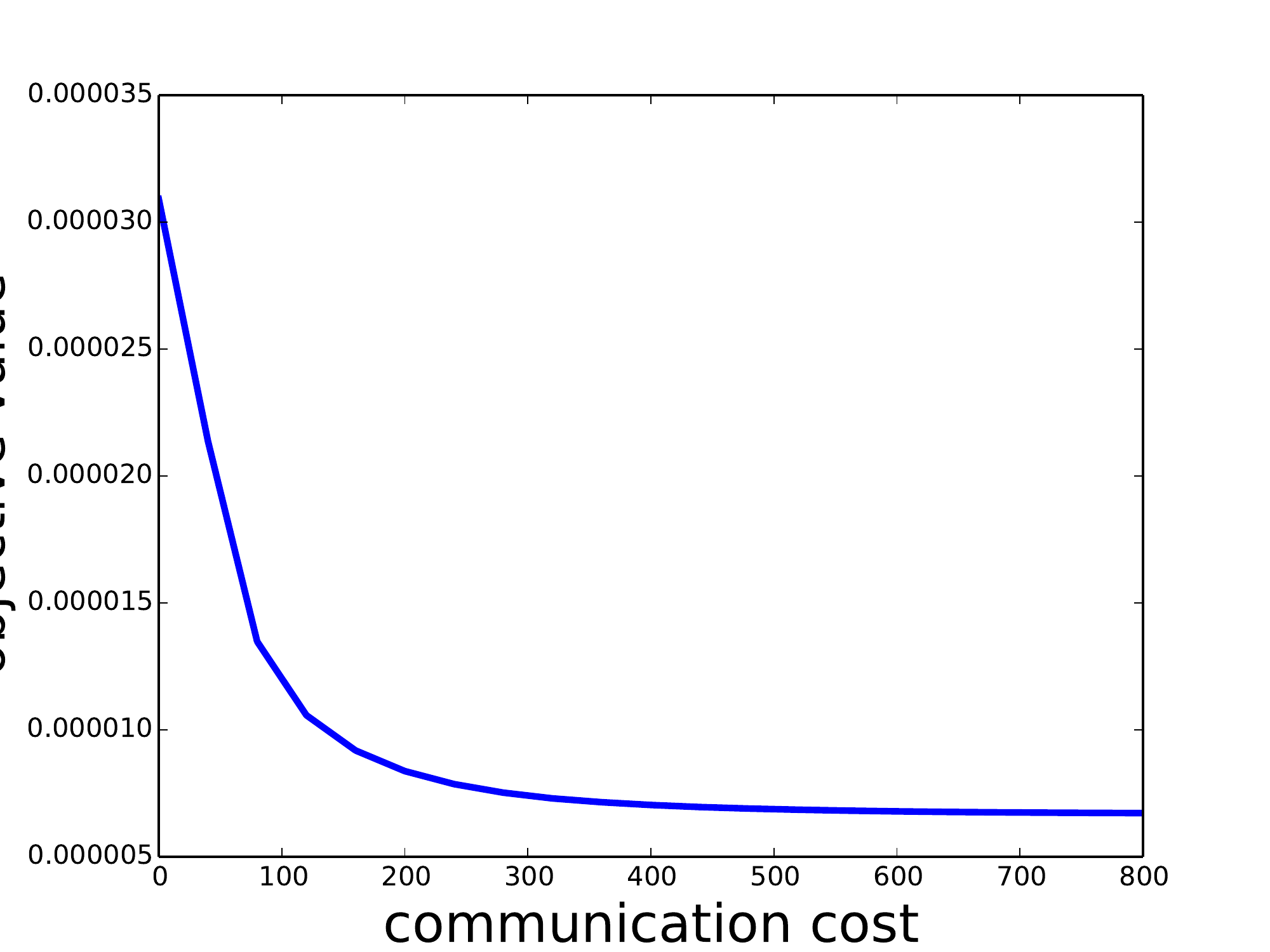}}
  \subfigure[phshing ]{
    \includegraphics[width = 0.45\textwidth]{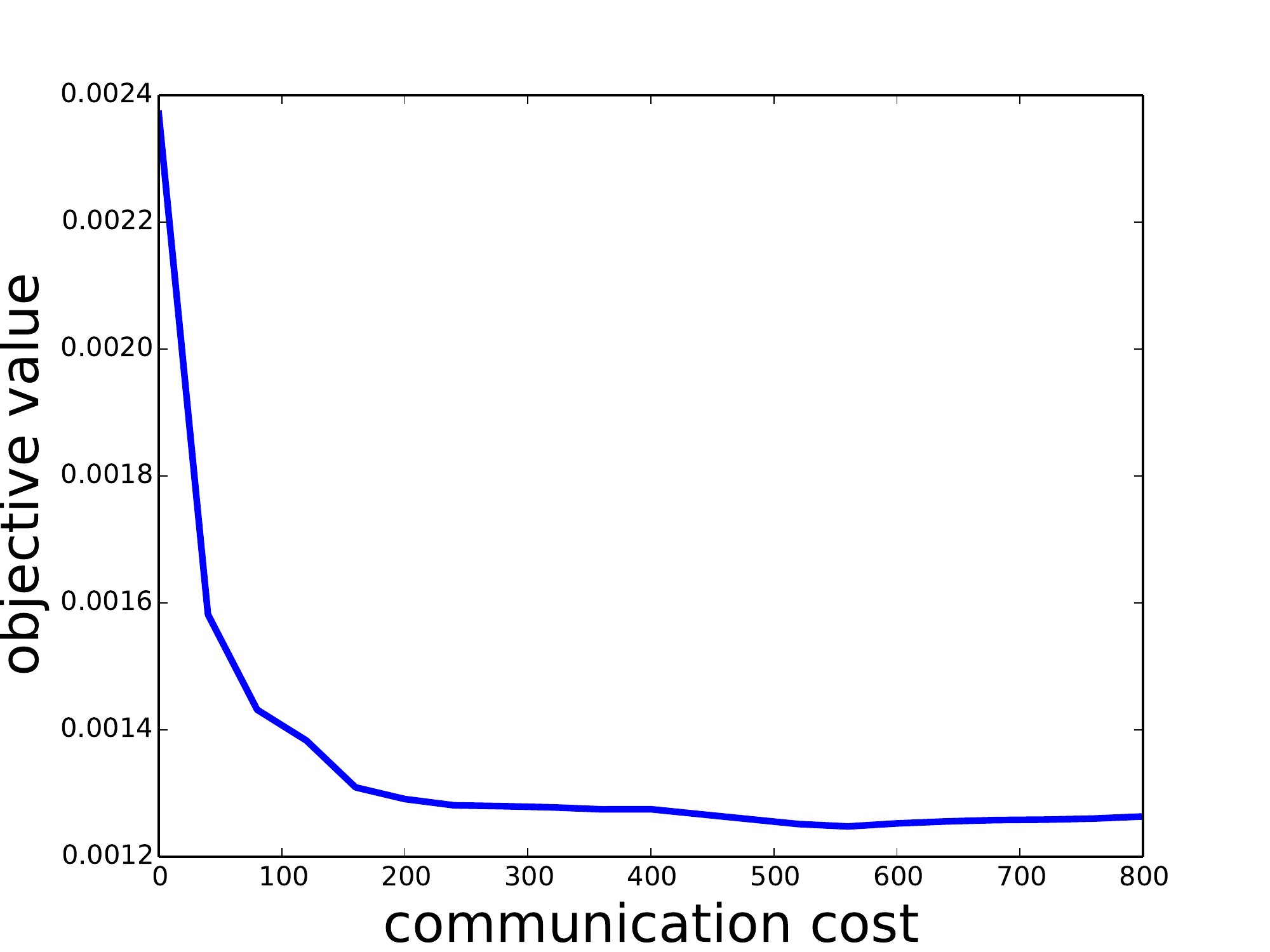}}
  \subfigure[a9a]{
    \includegraphics[width = 0.45\textwidth]{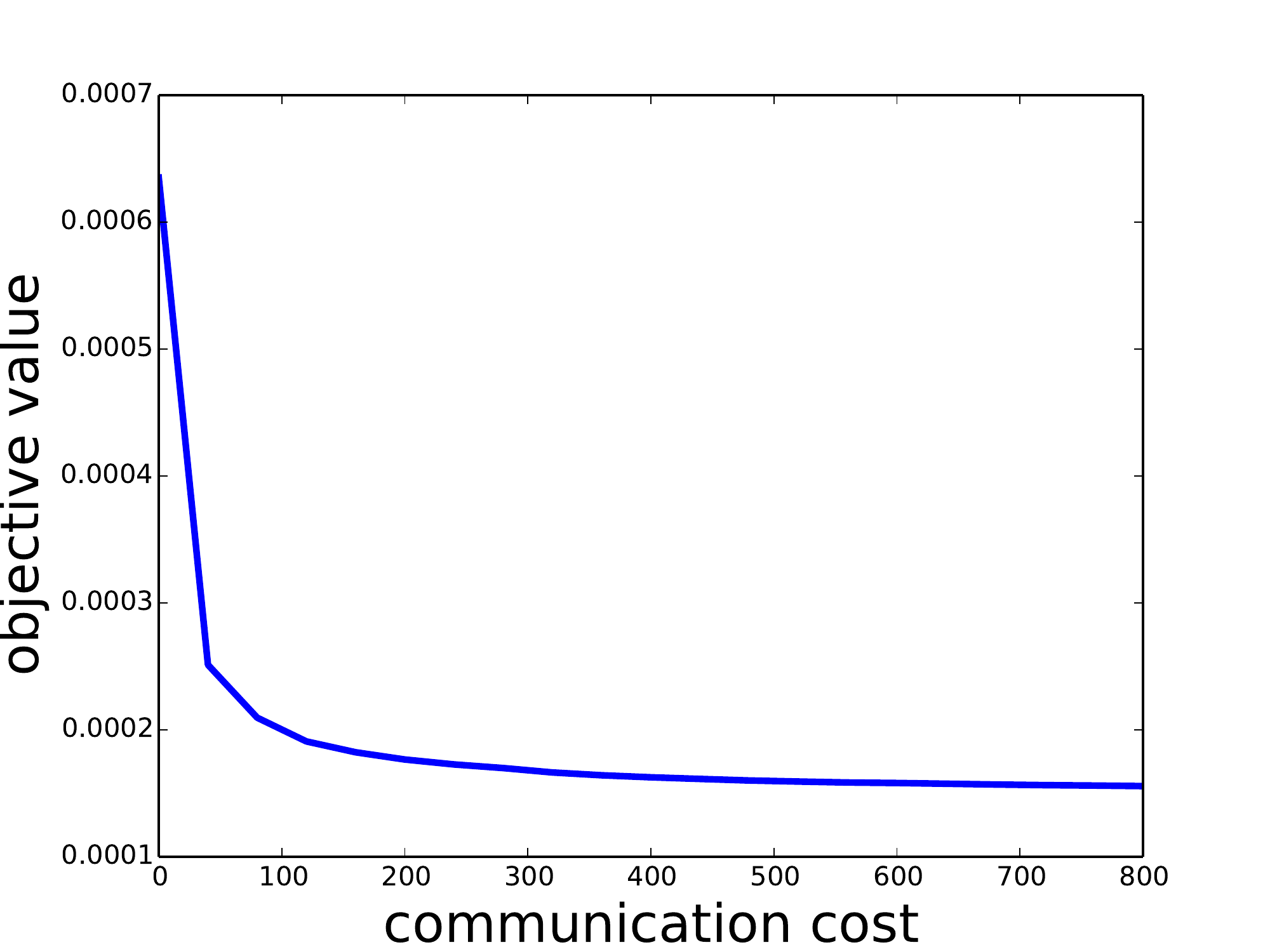}}
  \caption{The  objective value of Saddle-DSVC for $\nu$-SVM w.r.t. communication cost. Here $k=20$ and one unit is $kd$ communication cost.  gisette: $n = 6000, d = 5000$. madelon: $n = 2000, d = 500$. a9a: $n = 32561, d = 123$. phishing:  $n = 11055, d = 68$.}
  \label{fig:obj}
\end{figure}

\newpage
\clearpage
\bibliography{citation}
\bibliographystyle{plain}

\clearpage
\appendix

\section{The Equivalence of the Explicit and Implicit Update Rules of $\eta$ and $\xi$}
\label{apd:svmsp}

\begin{lemma}[Update Rules of HM-Saddle]
  \label{lm:upHM}
  The following two update rules are equivalent.
  \begin{itemize}
  \item $\eta[t+1] :=$
  \begin{align*}
  \arg\min\limits_{\eta \in \Delta_{n_1}} & \Big\{ \frac{1}{d} (w[t] +d(w[t+1] - w[t]))^{\rm T} X \eta \\
   & \qquad \qquad + \frac{\gamma}{d} H(\eta)  + \frac{1}{\tau} V_{\eta[t]}(\eta) \Big\}
    \end{align*}
  \item  $ \eta_i[t+1] :=$
  \begin{align*}
    Z^{-1} \exp \big\{  & ( \gamma + d\tau^{-1})^{-1} (d\tau^{-1} \log \eta_i[t] \\
     & \qquad - \langle    w[t] +d(w[t+1] - w[t]) ,  X_{\cdot i } \rangle ) \big\}
    \end{align*}
     for each $i \in   [n_1]$, where $Z = \sum_{i}  \eta_i$
    \footnote{Recall that $X_{\cdot i}$ is the $i$th column of $X$.}
  \end{itemize}
\end{lemma}

\begin{proof}
    The Lagrangian function of the first optimization formulation is
\begin{align*}
L(\eta, \lambda) = &\frac{1}{d}  ( w[t] +d(w[t+1] - w[t]))^{\rm T} X \eta  \\
& +   \frac{\gamma}{d} H(\eta) + \frac{1}{\tau}V_{\eta[t]}(\eta) + \lambda(\sum_{i}\eta_i -1)
 \end{align*}
Thus, we have
\begin{align*}
  \frac{\partial L}{\partial \eta_i} = 0 &= (\gamma d^{-1}+ \tau^{-1})\log \eta_i  \\
	&  + d^{-1} \langle w[t] +d(w[t+1] - w[t]),   X_{\cdot i}  \rangle \\
  &  -  \tau^{-1} \log \eta_i[t] + (\lambda + \tau^{-1}), \forall i \\
   \frac{\partial L}{\partial \lambda} = 0 &= -1 + \sum_{i}\eta_i
\end{align*}
Solve the above equalities, we obtain
\begin{align*}
  \eta_i[t+1] = & Z^{-1} \exp \big\{ ( \gamma + d\tau^{-1})^{-1} (d\tau^{-1} \log \eta_i[t]  \\
  -  & \langle w[t] +d(w[t+1] - w[t]),   X_{\cdot i} \rangle \big\}
\end{align*}\end{proof}
\begin{lemma}[Update Rules of $\nu$-Saddle]
  \label{lm:upnu}
  The following three update rules are equivalent.

\topic{Rule 1} $\eta[t+1] :=$
\begin{align*}
  \arg\min\limits_{\eta \in \calS_1}  & \Big\{ \frac{1}{d} ( w[t] +d(w[t+1] - w[t]))^{\rm T} X \eta  \\
  + & \frac{\gamma}{d} H(\eta) + \frac{1}{\tau}V_{\eta[t]}(\eta)  \Big\}
  \end{align*}

\topic{Rule 2}
\begin{itemize}
\item Step 1: $ \eta_i = $
\begin{align*}
Z^{-1} \exp \big\{  & ( \gamma + d\tau^{-1})^{-1} (d\tau^{-1} \log \eta_i[t]  \\
-  & \langle  w[t] +d(w[t+1] - w[t]), X_{\cdot i}  \rangle ) \big\}
\end{align*}

    for each $i \in [n_1]$,  where $Z = \sum_{i}  \eta_i$.
    \item Step 2: Sort $ \eta_i$ by the increasing order. W.l.o.g., assume that $\eta_1,\ldots,\eta_{n_1}$ is in increasing order. Define $\varsigma_i = \sum_{j\geq i} (\eta_j-\nu)$ and $\Omega_i=\sum_{j<i} \eta_j$. Find the largest index $i^*\in [n]$ such that $\varsigma_{i^*}\geq 0$ and $\eta_{i^*-1}(1+\varsigma_{i^*}/\Omega_{i^*})<\nu$ by binary search.
\item  Step 3:
    $$\forall i,   \eta_i[t+1] =        \left\{
      \begin{array}{ll}
        \eta_i(1+\varsigma_{i^*}/\Omega_{i^*}), & \text{if } i < i^* \\
        \nu , & \text{if } i \geq i^*
      \end{array} \right.
    $$
    \end{itemize}
  \topic{Rule 3}
  \begin{itemize}
  \item Step 1: $ \eta_i :=$
  \begin{align*}
   Z^{-1} \exp &\big\{ ( \gamma + d\tau^{-1})^{-1} (d\tau^{-1} \log \eta_i[t]  \\
   -& \langle  w[t] +d(w[t+1] - w[t]), X_{\cdot i}  \rangle ) \big\}
   \end{align*}
  for each $i \in [n_1]$, where $Z = \sum_{i}  \eta_i$.
    \item Step 2:
\begin{equation}
  \begin{array}{l}
    \mathbf{while} \quad \varsigma := \sum_{\eta_i > \nu} (\eta_i - \nu) \neq 0:\\
    \qquad \Omega=  \sum_{\eta_i < \nu} \eta_i  \\
    \qquad \forall i, \quad \mathbf{if} \; \eta_i \geq \nu, \quad \mathbf{then} \; \eta_i = \nu    \\
    \qquad \forall i, \quad \mathbf{if} \; \eta_i < \nu, \quad \mathbf{then} \; \eta_i = \eta_i (1 +\varsigma/ \Omega)
  \end{array}
\end{equation}
\end{itemize}
\end{lemma}
\begin{proof}
  Similar to the proof of Lemma~\ref{lm:upHM}, we first give the Lagrangian function of the
  first optimization formulation as follows.
  \begin{align*}
L(\eta, \lambda, \sigma) = &\frac{1}{d} ( w[t] +d(w[t+1] - w[t]))^{\rm T} X \eta \\
+   & \frac{\gamma}{d} H(\eta) + \frac{1}{\tau}V_{\eta[t]}(\eta) \\
 + & \lambda(\sum_{i}\eta_i -1)  +  \sum_{i} \alpha_i (\eta_i - \nu)
  \end{align*}
By KKT conditions, we have the following.
 \begin{align*}
   0 &= (\gamma d^{-1}+ \tau^{-1})\log \eta_i \\
    & \qquad + d^{-1} \langle    w[t] +d(w[t+1] - w[t]) , X_{\cdot i} \rangle \\
       &   \qquad    -  \tau^{-1} \log \eta_i[t] + (\lambda +\alpha_i + \tau^{-1}), \forall i \\
   1 &=  \sum\nolimits_{i}\eta_i  \\
   0 & = \alpha_i(\eta_i - \nu), \forall i \\
   0 &\geq \eta_i-\nu, \forall i \\
   0 &\leq \alpha_i, \forall i
\end{align*}
We first show the equivalence between Rule 1 and Rule 2. Note that $\eta[t+1]$ in Rule 2 satisfies the second and the fourth KKT conditions.
We only need to give all $\alpha_i$ and $\lambda$ satisfying other KKT conditions for Rule 2. Let
\begin{align*}
\tilde{\eta}_i = \exp \big\{ &( \gamma + d\tau^{-1})^{-1} (d\tau^{-1} \log\eta_i[t]  \\
 &- \langle w[t] +d(w[t+1] - w[t]), X_{\cdot i}  \rangle ) \big\}. \\
  \end{align*}
Let
\begin{align*}
 & \eta_i = Z^{-1} \tilde{\eta_i} \\
 & = Z^{-1} \exp  \big\{ ( \gamma + d\tau^{-1})^{-1} (d\tau^{-1} \log \eta_i[t]  \\
 & \qquad - \langle  w[t] +d(w[t+1] - w[t]), X_{\cdot i}  \rangle ) \big\}
\end{align*}
       as defined in Step 1 of Rule 2.
For $1\leq i\leq i^*-1$, let $\alpha_i=0$. For $i\geq i^*$, let
$$\alpha_i=(\gamma d^{-1}+ \tau^{-1})^{-1}\ln \frac{\eta_i(1+\varsigma_{i^*}/\Omega_{i^*})}{\nu}\geq 0.$$
The inequality follows from the definition of $i^*$. Note that we only need to prove that $\eta_{i^*}(1+\varsigma_{i^*}/\Omega_{i^*})\geq \nu$. If $\eat_{i^*}\geq \nu$, then the above inequality holds directly. Otherwise if $\eta_{i^*}< \nu$ and $\eta_{i^*}(1+\varsigma_{i^*}/\Omega_{i^*})< \nu$, we have that $\varsigma_{i^*+1} = \varsigma_{i^*}+ \nu-\eta_{i^*}>0$ and $\Omega_{i^*+1}=\Omega_{i^*}+ \eta_{i^*}$. We also have the following inequality
\begin{align*}
\eta_{i^*}(1+\frac{\varsigma_{i^*+1}}{\Omega_{i^*+1}})-\nu  &=  \frac{\eta_{i^*} (\varsigma_{i^*}+\Omega_{i^*})-\Omega_{i^*}\nu}{\Omega_{i^*}+ \eta_{i^*}} \\
= & \frac{\eta_{i^*}(1+\varsigma_{i^*}/\Omega_{i^*})-\nu}{\Omega_{i^*}(\Omega_{i^*}+ \eta_{i^*})}<0, \\
\end{align*}
which contradicts with the definition of $i^*$.
Finally, randomly choose an index $i$, let
$$\lambda=(\gamma d^{-1}+ \tau^{-1})^{-1}\ln\frac{Z\eta_i}{\eta_i[t+1]}-\alpha_i-\tau^{-1}.$$
By the chosen of $\alpha_i$, it is not hard to check that the value of $\lambda$ is the same for any index $i$.
Thus, $\eta_i[t+1],\alpha_i$ and $\lambda$ are the unique solution of KKT conditions. So Rule 1 and Rule 2 are equivalent. By a similar argument (define suitable $\alpha_i$ and $\lambda$), we can prove that Rule 1 and Rule 3 are equivalent, which finishes the proof.
\end{proof}
\eat{
Then according to the third equation, for any two $\eta^*_i,\eta^*_j <  \nu$,  we have
$\alpha_i=\alpha_j=0$. Thus, by the first condition, we have Property 1:
\begin{equation*}
\frac{\eta^*_i}{\eta^*_j} = \frac{\tilde{\eta_i}}{\tilde{\eta_j}}
\end{equation*}
Assume  $\eta^*_i<\nu$. It  implies $\alpha_i=0$. Then for any
$\tilde{\eta}_j\leq \tilde{\eta}_i$, we have Property 2:
\begin{equation*}
 \eta^*_j=\tilde{\eta}_j\cdot \exp\{ -(\gamma d^{-1}+ \tau^{-1}) (\lambda +\alpha_j + \tau^{-1}) \}\leq \tilde{\eta}_i\cdot \exp\{ -(\gamma d^{-1}+ \tau^{-1}) (\lambda + \tau^{-1}) \}=\eta^*_i<\nu.
\end{equation*}
Next, assume $\tilde{\eta}_i\leq \tilde{\eta}_j$. We have Property 3:
$$
\frac{\eta^*_i}{\eta^*_j} \geq \frac{\tilde{\eta_i}}{\tilde{\eta_j}}.
$$
We prove this property by discussing four cases. 1) $\eta^*_i=\eta^*_j=\nu$, then Property 3 obviously holds. 2) $\eta^*_i<\eta^*_j=\nu$, we have $\alpha_j\geq 0 =\alpha_i$ which implies that
$$
 \frac{\eta^*_i}{\eta^*_j}=\frac{\tilde{\eta_i}\cdot \exp\{ -(\gamma d^{-1}+ \tau^{-1}) (\lambda +\alpha_i + \tau^{-1}) \}}{\tilde{\eta_j}\cdot \exp\{ -(\gamma d^{-1}+ \tau^{-1}) (\lambda +\alpha_j + \tau^{-1}) \}}\geq \frac{\tilde{\eta_i}}{\tilde{\eta_j}}.
$$
3) $\eta^*_j<\eta^*_i=\nu$, then by the same argument as 2), we have the following
$$
\frac{\eta^*_j}{\eta^*_i} \geq \frac{\tilde{\eta_j}}{\tilde{\eta_i}}\geq 1,
$$
which implies $\eta^*_j\geq \eta^*_i= \nu$. It is a contradiction. 4) $\eta^*_i,\eta^*_j<\nu$, the correctness of Property 3 follows from Property 2.

Finally, for any $\tilde{\eta}_i>\nu$, the following Property 3 holds: $\eta^*_i=\nu$. This is because if $\eta^*_i<\nu$, then for any $\tilde{\eta}_j\leq \tilde{\eta}_i$, we have $\eta^*_j<\mu$ by Property 2. Then by Property 1, we conclude
$$
\eta^*_j=\frac{\eta^*_i\cdot\tilde{\eta}_j}{ \tilde{\eta}_i}<\tilde{\eta}_j.
$$
On the other hand, if $\tilde{\eta}_j\geq \tilde{\eta}_i>\mu$, we have $\eta^*_j\leq \mu<\tilde{\eta}_j$. Overall, we have $\sum_j \eta^*_j< \sum_j \tilde{\eta}_j=1$, which is a contradiction.

It is not hard to see that $\eta^*$ is the unique vector satisfying $1=\sum_i \eta^*_i$, $0\leq
\eta^*_i\leq \nu$ ($\forall i$) and Property 1,2,3,4.
Moreover, we can verify that the output $\eta$ of the second update rules also satisfies $1=\sum_i \eta_i$, $0\leq \eta_i\leq \nu$ ($\forall i$) and Property 1,2,3,4. Thus, $\eta$ must equal to $\eta^*$ which is the optimal solution of the first optimization problem.
}

\begin{remark}
We analyze Rule 2 in Lemma \ref{lm:upnu}. Roughly speaking, we find a suitable value $\eta_{i^*}$, set all value $\eta_{j}>\eta_{i^*}$ to be
$\nu$, and scales up other values by some factor $1+\varsigma_{i^*}/\Omega_{i^*}$. We can verify that the running
time of Rule 2 is $O(n\log n)$ since both the
sorting time and the binary search time are $O(n\log n)$. On the other hand, recall that the running
time of Rule 3 is $O(n/\nu)$ (explained in Section \ref{sec:svmsp}). Thus, if the parameter $\nu$ is extremely small, we can use Rule 2 in practice.
\end{remark}

\begin{algorithm}[t]
  \caption{Pre-processing in Clients}
  \label{alg:pre}
\begin{algorithmic}[1]

\REQUIRE { $\calP$: $n_1$ points $x_i^+$ with label $+1$ and $\calQ$: $n_2$ points $x_i^-$ with label $-1$, distributed at $k$ clients}

\FORALL {clients in $\clientset$}{
    \STATE  $ W \leftarrow $ $d$-dimensional Hadamard Matrix
    \STATE  $D \leftarrow$ $d\times d$ diagonal matrix whose  entries are i.i.d. chosen from $\pm 1$
    \STATE $ \gamma \leftarrow \frac{\e\beta}{  2 \log n}, $ $q \leftarrow
    O(\sqrt{\log n})$
    \STATE   $ \tau  \leftarrow \frac{1}{2q}\sqrt{\frac{d }{\gamma}}, \sigma \leftarrow
    \frac{1}{2q}\sqrt{d \gamma}, \theta \leftarrow 1- \frac{1}{d + q\sqrt{d}/\sqrt{\gamma }}. $
\STATE $d$-dimension vector $ w^{(0)} \leftarrow [0,\ldots, 0] $
  }
\ENDFOR

  \FOR { client $\client \in \clientset$}{
   \STATE Assume that there are $ m_1$ points $x^+_1, \ldots, x^+_{m_1}$ and $m_2$ points $x^-_1,  \ldots, x^-_{m_2}$ maintained in $C$
   \STATE $ \client.X^{+} \leftarrow WD \cdot [x^{+}_1 , x^{+}_2, \ldots,  x^{+}_{m_1} ] $,
   \STATE $ \client.X^{-} \leftarrow WD \cdot [x^{-}_1 , x^{-}_2, \ldots,  x^{-}_{m_2} ] $,
   \STATE  $\client.\eta[-1] = \client.\eta[0] \leftarrow [ \frac{1}{n_1},  \ldots,
   \frac{1}{n_1}]^{\rm T}\in \R^{m_1}$,
   \STATE $\client.\xi[-1] = \client.\xi[0]\leftarrow  [\frac{1}{n_2},   \ldots,
   \frac{1}{n_2}]^{\rm T} \in \R^{m_2}$.
  }
\ENDFOR
\end{algorithmic}
\end{algorithm}

\section{Details for Distributed Algorithms: Saddle-DSVC}
\label{apd:dist}
This section is supplementary for Section~\ref{sec:dist}.
First, we give the pseudocode of DisSaddle-SVC. See Algorithm~\ref{alg:pre} for
the pre-processing step for each clients. Recall that we assume there are $ m_1$ points $x^+_1, x^+_2, \ldots, x^+_{m_1}$
and $m_2$ points $x^-_1, x^-_2, \ldots, x^-_{m_2}$ maintained in $C$. We use $\mathbf{1}^{m}$ to
denote a vector with all components being $1$. The initialization is as follows.
\begin{align*}
 & \client.X^{+} = WD \cdot [x^{+}_1 , x^{+}_2, \ldots,  x^{+}_{m_1} ], \\
 &\client.\eta[-1] = \client.\eta[0] = n^{-1}_1\mathbf{1}^{m_1} \\
& \client.X^{-} = WD \cdot [x^{-}_1 , x^{-}_2, \ldots,  x^{-}_{m_2} ], \\
& \client.\xi[-1] = \client.\xi[0] =n_2^{-1}\mathbf{1}^{m_2}
\end{align*}
Next, see Algorithm~\ref{alg:dist} for the interactions between
the server and clients in every iteration. Note that only $\nu$-Saddle needs the fourth round in Algorithm~\ref{alg:dist}. We use $\text{flag}_{\nu} \in \{ \mathbf{True, False}\}$  to distinguish the two cases. If we consider $\nu$-Saddle, let
$\text{flag}_{\nu}$ be \textbf{True}. Otherwise, let
$\text{flag}_{\nu}$ be \textbf{False}.

Then, we analyze the communication cost.
\eat{
\begin{reptheorem}{thm:commcost}[restated]
  The communication cost of Saddle-DSVC is $\tilde{O}(k(d +
  \sqrt{d/\e}))$.
\end{reptheorem}
}

\begin{theorem}
  \label{thm:commcost}
  The communication cost of Saddle-DSVC is $\tilde{O}(k(d +
  \sqrt{d/\e}))$.
\end{theorem}
\begin{proof}
Note that in each iteration of Algorithm~\ref{alg:dist}, the server and clients interact three
times for hard-margin SVM and $O(1/\nu)$ times for $\nu$-SVM. The communication cost of each
iteration is $O(k)$. By Theorem~\ref{thm:convergence}, it takes
$\tilde{O}(d+\sqrt{d/\e})$ iterations. Thus, the total communication cost is $\tilde{O}(k(d +
  \sqrt{d/\e}))$.
\end{proof}

\begin{center}
\hrule
\vspace{0.5ex}
\captionof{algorithm}{Saddle-DSVC}\label{alg:dist}
\vspace{-1ex}
\hrule
\vspace{0.5ex}
  \begin{algorithmic}[1]
  \FOR{$t \leftarrow 0 $ \textbf{to} $T - 1$ }{
    \STATE \emph{\# first round}
    \STATE  \textbf{Server:}  Pick an index $i^{*} \in \{1,2, \ldots, d\}$  uniformly at random and
    send $i^{*}$ to every client.

  \FOR{client $\client \in \clientset$}{
      \STATE $ \client.\delta_{i^*}^{+} \leftarrow $ $\langle  \client.X^{+}_{i^*}, \client.\eta[t] + \theta(\client.\eta[t]  - \client.\eta[t-1]) \rangle $
      \STATE $ \client.\delta^{-}_{i^*} \leftarrow $  $\langle \client.X^{-}_{i^*}, \client.\xi[t] + \theta (\client.\xi[t]  - \client.\xi[t-1]) \rangle  $

      \STATE Send $\client.\delta_{i^*}^{+} $ and $\client.\delta_{i^*}^{-} $ to server.
  }\ENDFOR

    \STATE \emph{\# second round}
  \STATE \textbf{Server:}
   Let $\server.\delta_{i^*}^{+} = \sum_{\client \in \clientset } \client.\delta_{i^*}^+ $ and
$\server.\delta_{i^*}^{-} = \sum_{\client \in \clientset }\client.\delta_{i^*}^- $. Broadcast $ \server.\delta_{i^*}^{+} $ and $ \server.\delta_{i^*}^{-} $.

 \FOR{client $\client \in \clientset$}
  {
    \STATE  $ \forall i \in [d],  w_i[t+1] \leftarrow  $
    $\left\{
          \begin{array}{ll}
            (w_i[t] +  \sigma  (\server.\delta_{i}^{+} -  \server.\delta_{i}^{-} ) )/ (\sigma
            + 1 ), & \text{if } i = i^* \\
                x
        \end{array} \right.
        $

        \STATE   $\forall j, \client.\eta_j[t+1] \leftarrow $
        $\exp \big\{ (\gamma + d\tau^{-1})^{-1}(d\tau^{-1}
            \log \client.\eta_j[t] $ \\
           $\hspace{4cm}  -  \langle  w[t] +d(w[t+1] - w[t])
            ,  \client.X^{+}_{\cdot j} \rangle ) \big \} $
        \STATE  $\forall j, \client.\xi_j[t+1] \leftarrow $
        $ \exp \big\{
          (\gamma + d\tau^{-1})^{-1}(d\tau^{-1} \log \client.\xi_j[t] $ \\
         $ \hspace{4cm}+  \langle  w[t] +d(w[t+1] -
          w[t]), \client.X^{-}_{\cdot j} \rangle ) \big\}  $
        \STATE
        $\client.Z^+  \leftarrow \sum_{j}  \client.\eta_j[t+1] ,$
         $\client.Z^-  \leftarrow \sum_{j}  \client.\xi_j[t+1] $
        \STATE Send $\client.Z^+ $ and $\client.Z^-$ to server
}\ENDFOR

    \STATE \emph{\# third round}
\STATE \textbf{Server:}  Let $(\server.Z^+, \server.Z^{-}) \leftarrow \sum_{\client \in \clientset}(\client.Z^+,\client.Z^{-} ) $,
and broadcast $\server.Z^+ $ and $\server.Z^-$.

   \FOR{client $\client \in \clientset$}
  {
    \STATE $ \client.\eta_j[t+1] \leftarrow \client.\eta_j[t+1] / \server.Z^+  $, $\forall
     \client.\xi_j[t+1] \leftarrow  \client.\xi_j[t+1] / \server.Z^- $
  }  \ENDFOR

    \STATE \emph{\# fourth round, only for $\nu$-Saddle. $\text{flag}_{\nu}$ is true if use the code
    for $\nu$-Saddle}
  \IF {$\text{flag}_{\nu}$\textbf{ is True}  }
{
  \REPEAT
  {
    \FOR{client $\client \in \clientset$}
    {

      \STATE $\client.\varsigma^{+} = \sum_{\eta_i > \nu} (\eta_i - \nu) $,
      $\client.\Omega^{+} =  \sum_{\eta_i < \nu} \eta_i  $.
      \STATE       $\client.\varsigma^{-} = \sum_{\xi_j > \nu} (\xi_j - \nu) $,
      $\client.\Omega^{-} =  \sum_{\xi_j < \nu} \xi_j  $.
      \STATE Send $\client.\varsigma^{+}, \client.\varsigma^{-}, \client.\Omega^{+},
      \client.\Omega^{-}$ to server.
    } \ENDFOR

    \STATE  \textbf{Server:} $(\server.\varsigma^{+}, \server.\varsigma^{-}, \server.\Omega^{+}, \server.\Omega^{-}) \leftarrow $ \\
    $\hspace{4cm}\sum_{\client \in \clientset} (\client.\varsigma^{+},  \client.\varsigma^{-}, \client.\Omega^{+}, \client.\Omega^{-})$.
      \FOR{client $\client \in \clientset$}{
              \STATE $\forall i$, $\mathbf{ if }  \; \eta_i > \nu,  \mathbf{ then } \; \eta_i = \nu$; \\
              $\forall i$, $\mathbf{ if } \; \eta_i < \nu,  \mathbf{ then } \;  \eta_i = \eta_i (1 + \server.\varsigma^{+} / \server.\Omega^{+} )$
              \STATE $\forall j$, $\mathbf{ if }  \; \xi_j > \nu,  \mathbf{ then } \; \xi_j = \nu $; \\
               $\forall j$,    $\mathbf{ if } \; \xi_j < \nu,  \mathbf{ then } \;  \xi_j =\xi_j (1 +\server.\varsigma^{-} / \server.\Omega^{-}) $
      }\ENDFOR
  }  \UNTIL{ $\server.\varsigma^{+}$ \AND $\server.\varsigma^{-}$ \textbf{are zeroes}}
}
\ENDIF
}
\ENDFOR
\end{algorithmic}
\vspace{1ex}
\hrule
\end{center}

Liu et al. \cite{liu2016distributed} proved a theoretical lower bound of the communication cost for
distributed SVM as follows.  Note that the statement of
  Theorem~\ref{thm:bound} is not exactly the same as the Theorem 6 in ~\citep{liu2016distributed}. This is because they omit the case that $d < 1/\e$. We prove that
  they are equivalent briefly.  Note that
if $ d = \Theta( 1/\e)$, the communication lower bound is $\Omega ( k (d + \sqrt{d/\e}))$ which matches the
communication cost of our algorithm Saddle-DSVC.

\begin{theorem} [Theorem 6 in~\cite{liu2016distributed}]
  \label{thm:bound}
  Consider a set of $d$-dimension points distributed at $k$ clients.
  The communication cost to achieve a $(1-\e)$-approximation of the distributed SVM problem is at least $\Omega(k \min\{ d,  1/\e \} )$
  for any $\e>0$.
\end{theorem}

\eat{
For the lower bound, readers may notice that the Theorem~\ref{thm:bound} is not exactly the same as
Theorem 6 in Liu et al.~\cite{liu2016distributed}. This is because they omit the case that $d < 1/\e$.

\begin{reptheorem}{thm:bound}[restated]
  Consider a set of $d$-dimension points distributed at $k$ clients.
  The communication cost to achieve a $(1-\e)$-approximation of the distributed SVM problem is at least $\Omega(k \min\{ d,  1/\e \} )$
  for any $\e>0$.
\end{reptheorem}
}

\begin{proof}[Proof Sketch]
In Theorem 6 of~\cite{liu2016distributed}, the authors obtain a lower bound $\Omega(kd)$ if $\e \leq
(\sqrt{17} - 4)/16\d ) $. Their proof can be extended to the case $\e \geq  (\sqrt{17} -
4)/16\d ) $. In this case, we can make a reduction from the $k$-OR problem in which each client
maintains a $( (\sqrt{17} -
4)/16\e ) $-bit vector instead of a $d$-bit vector. As the proof of Theorem 6
in~\cite{liu2016distributed}, we can obtain a lower bound $\Omega(k/\e)$, which proves the theorem.
\end{proof}

\section{Missing Proofs}
\label{apd:series}

\begin{figure}[h]
  \centering
  \includegraphics[width = 0.6\textwidth]{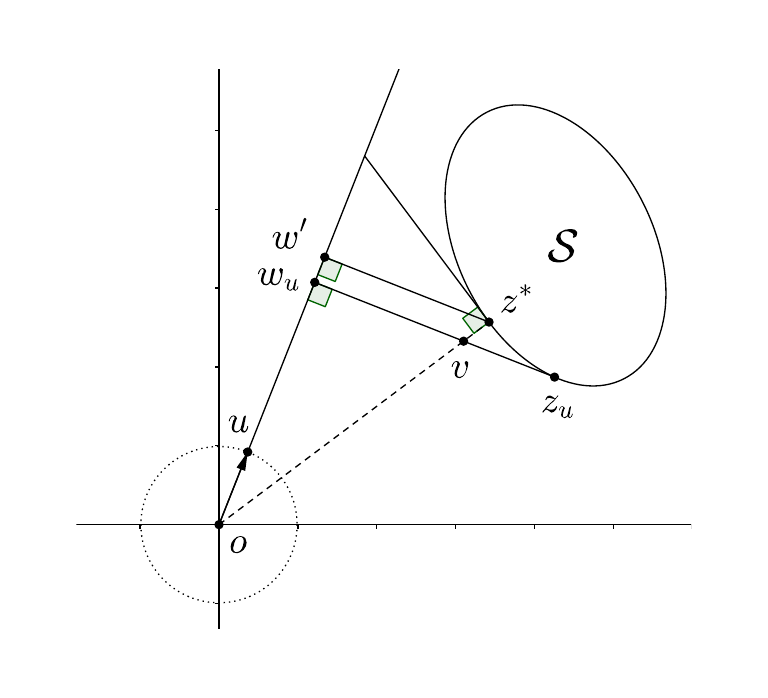}
  \caption{The equivalence between C-Hull and saddle point optimization~\eqref{eq:saddle}.}
  \label{fig:eq}
\end{figure}
\begin{replemma}{lm:svmtosp}[restated]
  Problem C-Hull~\eqref{eq:chull}  is equivalent to the saddle point
optimization~\eqref{eq:saddle}.
\end{replemma}
\begin{proof}
Consider the saddle point optimization~\eqref{eq:saddle}. First, note that
  \begin{equation*}
    w^{\rm T} A \eta -   w^{\rm T} B \xi  -  \frac{1}{2}\| w \|^2  =
   w^{ \rm T} (A \eta - B \xi ) - \frac{1}{2}\| w \|^2
  \end{equation*}
The range of the term $ (A \eta - B \xi ) $   for $ \eta  \in \Delta_{n_1} ,  \xi \in \Delta_{n_2}$  is a convex set, denoted by $\calS$. Since the
convex hulls of $\calP$  and  $\calQ$ are linearly separable, we have $ 0 \notin \calS$. Denote $ \phi
(w, z) = w^{  \rm T} z - \frac{1}{2}\| w \|^2  $ for any $w \in \R^{d}, z \in \calS$. Then
\eqref{eq:saddle} is equivalent to $\max_w \min_{z\in \calS}\phi(w,z)$. Note that
$$
\max_w \min_{z\in  \calS}\phi(w,z)\geq \min_{z\in \calS}\phi(\mathbf{0}^d,z)=0.
$$
 Thus, we only need to consider those
directions $w\in \R^d$ such that there exists a point $z\in \calS$ with $w^T z\geq 0$. We use
$\mathcal{W}$ to denote the collection of such directions.

Let $u$ be a unit vector in $\mathcal{W}$. Denote

$$z_{u}:=\arg\min_{z\in \calS} \phi(u,z)=\arg\min_{z\in \calS} u^T z.
$$

 By this definition, $z_u$ is the point with smallest projection distance to $u$ among $\calS$ (see Figure~\ref{fig:eq}). Observe that if a direction $w=c\cdot u$ ($c>0$), then we have $\arg\min_z \phi(w,z)=\arg\min_z \phi(u,z)$. Also note that

 \begin{equation*}
  \max_{w=c\cdot u:c>0} w^{ \rm T} z_{u} - \frac{1}{2}\| w  \|^2  = \max_{w=c\cdot u:c>0} \frac{1}{2} (-\|w - z_{u} \|^2  +
  \|  z_{u} \|^2).
 \end{equation*}
 Let

$$w_{u}:=\arg\max_{w=c\cdot u:c>0}\phi(w,z_u)= \arg \min_{w=c\cdot u:c>0} \|w - z_{u}
\|^2.
$$

$w_{u}$ is the projection point of $z_u$ to the line $ou$, where $o$ is the origin. See Figure~\ref{fig:eq} for an example. Overall, we have
\begin{align*}
  &\max_{w} \min_{ \eta  \in \Delta_{n_1} ,  \xi \in \Delta_{n_2} } w^{\rm T} (A \eta - B \xi)-
  \frac{1}{2}\| w\|^2  \\
  =  & \max_{u\in \mathcal{W}:\|u\|=1}  \frac{1}{2} ( - \| w_{u} - z_{u} \|^2 +  \|
  z_{u} \|^2)  \\
  =  & \max_{u\in \mathcal{W}:\|u\|=1} \frac{1}{2} \| w_{u} \|^2.
\end{align*}
The last equality is by the Pythagorean theorem. Let $z^*$ be the closest point in $\calS$ to the
origin point. Next, we show that $\max_{u\in \mathcal{W}:\|u\|=1} \| w_{u} \|^2=\|z^*\|^2 $. Given a
unit vector $u\in \mathcal{W}$, define $w'$ to be the projection point of $z^*$ to the line $ou$. By the definition of $z_u$ and $w_u$, we have that $\max_{u}  \| w_{u} \|^2 \leq \|w'\|^2\leq \| z^* \|^2$. Moreover, let $u=z^*/\|z^*\|$. In this case, we have $\| w_{u} \|^2  = \| z^* \|^2$. Thus, we conclude that $\max_{u}  \| w_{u} \|^2 = \| z^* \|^2$.

Overall, we prove that
\begin{align*}
  & \max_{u\in \mathcal{W}:\|u\|=1} \frac{1}{2}\| w_u\|^2 = \frac{1}{2} \| z^*\|^2   \min_{z\in \calS}\frac{1}{2}\|z\|^2 \\
   =  & \min_{ \eta  \in \Delta_{n_1} ,   \xi \in \Delta_{n_2} }  \frac{1}{2}\| A \eta - B \xi \|^2
\end{align*}
Thus, C-Hull~\eqref{eq:chull}  is equivalent to the saddle point optimization~\eqref{eq:saddle}.
\end{proof}

\begin{replemma}{lm:approx}[restated]
Let $(w^*, \eta^*, \xi^*) $ and $(w^{\circ},   \eta^{\circ}, \xi^{\circ})$ be the optimal solution
of saddle point  optimizations~\eqref{eq:saddle} and ~\eqref{eq:saddle2} respectively. Define $\OPT$ as in \eqref{eq:saddle}.  Define
  \begin{equation*}
g(w) :=  \min  \limits_{ \eta  \in \Delta_{n_1} ,  \xi \in \Delta_{n_2} }   w^{\rm T} A  \eta -   w^{\rm T} B \xi
  -  \frac{1}{2}\| w \|^2.
  \end{equation*}
 Then  $g(w^*) - g(w^{\circ}) \leq \e \OPT$ (note that $g(w^*)=\OPT$).
\end{replemma}

\begin{proof}
  Let
  \begin{align*}
    \phi(w, \eta, \xi) &= w^{\rm T} A \eta - w^{\rm T} B \xi - \frac{1}{2}\| w \|^2 , \\
  \phi_{\gamma}(w, \eta, \xi) &= \phi(w, \eta, \xi) + \gamma H(\eta) + \gamma H(\xi), \\
 \tilde{\eta}, \tilde{\xi}  &= \argmin_{\eta \in \Delta_{n_1} , \xi \in \Delta_{n_2}} \phi(w^{\circ}, \eta, \xi ).
  \end{align*}

 By the definition of saddle points, we have
  \begin{eqnarray*}
	g(w^{\circ})& = & \phi(w^{\circ}, \tilde{\eta}, \tilde{\xi}) = \phi_{\gamma}(w^{\circ},
    \tilde{\eta}, \tilde{\xi}) -  \gamma H(\tilde{\eta}) - \gamma H(\tilde{\xi})   \\
	&\geq& \phi_{\gamma}(w^{\circ}, \eta^{\circ}, \xi^{\circ}) -  \gamma H(\tilde{\eta}) - \gamma H(\tilde{\xi}) \\
	&\geq &\phi_{\gamma}(w^{*}, \eta^{\circ}, \xi^{\circ}) -  \gamma	H(\tilde{\eta}) - \gamma H(\tilde{\xi})   \\
	&=& \phi (w^{*}, \eta^{\circ}, \xi^{\circ}) -  \gamma H(\tilde{\eta}) -
	\gamma H(\tilde{\xi})   +\gamma H(\eta^{\circ}) + \gamma H(\xi^{\circ})  \\
	&\geq&  \phi (w^{*}, \eta^{*}, \xi^{*}) -  \gamma H(\tilde{\eta}) -
	\gamma H(\tilde{\xi})   +  \gamma H(\eta^{\circ}) + \gamma H(\xi^{\circ})  \\
	&=& g(w^{*})  -  \gamma H(\tilde{\eta}) -
	\gamma H(\tilde{\xi})   +  \gamma H(\eta^{\circ}) + \gamma H(\xi^{\circ}) 	\\
	&\geq& g(w^{*})  -  \gamma H(\tilde{\eta}) -
	\gamma H(\tilde{\xi}).
  \end{eqnarray*}
  Note that  entropy function satisfies $0 \leq H(u) \leq \log n$ for any $u\in \Delta_n$. Thus,   $ \gamma
  H(\tilde{\eta}) + \gamma H(\tilde{\xi})  \leq   \frac{\e \beta}{2	\log n}  \cdot (\log n_1 + \log
  n_2) \leq \e \OPT$.
  Overall, we prove that $g(w^*) - g(w^{\circ}) \leq \e \OPT$.
\end{proof}

\begin{replemma}{lm:eq}[restated]
  RC-Hull~\eqref{eq:reducedhull} is equivalent to the following saddle point optimization.
\begin{equation*}
\OPT =  \max \limits_{w}  \min \limits_{ \eta  \in \calD_{n_1} , \; \xi \in \calD_{n_2} }  w^{\rm T} A \eta -   w^{\rm T} B \xi  -  \frac{1}{2}\| w \|^2.
\end{equation*}
\end{replemma}

\begin{proof}
The proof is almost the same to the proof of Lemma \ref{lm:svmtosp}. The only difference is that the range of the term $ (A \eta - B \xi ) $ is another convex set defined by $ \eta  \in \calD_{n_1} ,  \xi \in \calD_{n_2}$.
\end{proof}

\begin{lemma}
\label{lm:nuapprox}
Let $(w^*, \eta^*, \xi^*) $ and $(w^{\circ},   \eta^{\circ}, \xi^{\circ})$ be the optimal solution
of saddle point  optimizations~\eqref{eq:nusaddle} and ~\eqref{eq:nusaddle2} respectively. Define $\OPT$ as in \eqref{eq:nusaddle}. Define
  \begin{equation*}
g(w) :=  \min  \limits_{ \eta  \in \calD_{n_1} ,  \xi \in \calD_{n_2} }   w^{\rm T} A  \eta -   w^{\rm T} B \xi
  -  \frac{1}{2}\| w \|^2.
  \end{equation*}
 Then  $g(w^*) - g(w^{\circ}) \leq \e \OPT$.
\end{lemma}
\begin{proof}
  Note that $ \calD_{n_1} $ is a convex polytope contained in $\Delta_{n_1}$ and $ \calD_{n_2} $ is a convex polytope contained in
  $\Delta_{n_2}$. It is not hard to verify that the proof of Lemma~\ref{lm:approx} still
  holds for  $ \calD_{n_1} $ and $\calD_{n_2}$.
\end{proof}

\subsection{Proof of Theorem~\ref{thm:convergence}}
\label{apd:cvg}
For preparation, we give two useful Lemmas~\ref{lm:A1}
and~\ref{lm:A2}.  Recall that  $V_x(y)$ is the Bregman divergence
function which is defined as $  H(y) - \langle \nabla H(x), y-x \rangle - H(x)$.

The two lemmas generalize  Lemma A.1 and Lemma A.2 in~\cite{allen2016optimization} by changing the domain $\Delta_{m}$ to a convex polytope $\calS_{m}$ contained in $\Delta_{m}$. However, refer to the proofs of Lemma A.1 and Lemma A.2, it still work for the general version.

\begin{lemma}
  \label{lm:A1}
  Let $x_2 = \argmin_{z \in \calS_{m} } \left\{ \frac{V_{x_1}(z)}{\tau} + \gamma H(z) \right\}$.
  Let $\calS_{m}$ be a convex polytope contained in $\Delta_{m}$. Then for every   $u \in \calS_m $, we have
  \begin{align*}
    & \frac{1}{\tau} V_{x_1}(u) - \left(  \frac{1}{\tau} + \eta \right) V_{x_2}(u) - \frac{1}{2\tau}
    \| x_2 - x_1 \|_1^2   \\
    & \geq  \gamma H(x_2) - \gamma H(u).
  \end{align*}
\end{lemma}

\begin{lemma}
  \label{lm:A2}
  Let $ x = \argmin_{z \in \calS_m } \left\{ H(z) \right\} $. Let $\calS_{m}$ be a convex polytope contained in
  $\Delta_{m}$. Then for all $u \in \calS_m $,
  \begin{equation*}
    H(u) - H(x) \geq  V_{x}(u).
  \end{equation*}
\end{lemma}

Combing the above lemmas and  almost the same analysis as in Theorem 2.2
in~\cite{allen2016optimization}, we obtain the following Theorem~\ref{thm:allen}.

\begin{theorem}
  \label{thm:allen}
 After $T $ iterations of Algorithm~\ref{alg:update} (both
HM-Saddle and $\nu$-Saddle versions),  we obtain a directional vector $w[T]\in \R^d$ satisfying that
  \begin{align*}
&(\tau^{-1} + 2\gamma d^{-1}) \mathbb{E} \left[ V_{\eta[T]} ( \eta^{\circ} )  +
  V_{\xi[T]} ( \xi^{\circ} ) \right]   \\
  & \qquad + ((4\sigma)^{-1} + 1) \mathbb{E} [\|w^{\circ} - w[T] \|^2 ]  \\
\leq & \theta^{T} \cdot \left(  2\left( \tau^{-1} + 2\gamma d^{-1}\right) \log n
+ ((2\sigma)^{-1} + 1) \|  w^{\circ}\|^2 \right),
  \end{align*}
where  $ \tau  \leftarrow \frac{1}{2q}\sqrt{\frac{d}{\gamma}}, \sigma \leftarrow
  \frac{1}{2q}\sqrt{d \gamma}, \theta \leftarrow 1-\frac{1}{d + q\sqrt{d}/\sqrt{\gamma }}$, for some $q =
  O(\sqrt{\log n}).
$
\end{theorem}
\begin{proof}[Proof Sketch]
 The  difference between our statement and Theorem 2.2 in~\cite{allen2016optimization} is that
 we update two probability vectors $\eta$ and $\xi$ instead of one in an iteration. Thus, we have
 two terms $V_{\eta[T]} ( \eta^{\circ} )$ and $V_{\xi[T]} ( \xi^{\circ} )$ on the left hand side.   Moreover, we  care about convex polytopes $\calS_{1} \subset \Delta_{n_1}$ and
 $\calS_{2} \subset \Delta_{n_2}$ instead of $\Delta_{n_1}$ and $\Delta_{n_2}$.

However, these differences do not influence the correctness of the proof of Theorem 2.2
in~\cite{allen2016optimization}. Note that we replace Lemma A.1 and  Lemma A.2 in~\cite{allen2016optimization} by  Lemma~\ref{lm:A1} and
Lemma~\ref{lm:A2}. It is not hard to verify the proof of Theorem 2.2
in~\cite{allen2016optimization} works for our theorem.
\end{proof}

We also need the following lemma.

\begin{lemma}
  \label{lm:g2x}
 Define   \begin{equation*}
g(w) :=  \min  \limits_{ \eta  \in \calS_{1} ,  \xi \in \calS_{2} }   w^{\rm T} A  \eta -   w^{\rm T} B \xi
  -  \frac{1}{2}\| w \|^2.
  \end{equation*}
  where $\calS_{1}$ and $\calS_{2}$ are two convex polytopes such that  $\calS_{1} \subset
  \Delta_{n_1}$ and $\calS_{2} \subset \Delta_{n_2}$.  For any $u,v\in
  \R^d$, we have
$$g(u) - g(v) \leq  2(1 + \| v \|) \| u-v \|.$$
\end{lemma}
\begin{proof}
 Denote by $\nabla g(w)$ \emph{any} subgradient of $g(w)$ at point $w$. We write $\nabla
 g(w)= A \tilde{\eta}_w - B\tilde{\xi}_w - w$ for any arbitrary $\tilde{\eta}_w \in \calS_{1} ,
\tilde{\xi}_w\in \calS_{2}$ satisfying that $g(w)=w^{\rm T} A \tilde{\eta}_w -
w^{\rm T} B \tilde{\xi}_w - \|w\|^2$.
Note that $ A \tilde{\eta}_w$ (resp. $B \tilde{\xi}_w$) can be
considered as a weighted
combination of all points $x_i$ (resp. $x_i$), we claim that $\| A \tilde{\eta}_w\|\leq 1$
($\| B \tilde{\xi}_w \|\leq 1$) owing to the assumption that every $x_i$ satisfies $\|x_i\|\leq
1$. Next, we compute as follows
  \begin{align*}
   & g(u)-g(v)=\int_{\tau=0}^{1} \innerprod{\nabla g(v+\tau(u-v))}{u-v} d \tau \\
   =& \int_{\tau=0}^{1} \innerprod{ A\tilde{\eta}_{v+\tau(u-v)}  -
      B \tilde{\xi}_{v+\tau(u-v)}  -(v+\tau(u-v))}{u-v} d \tau \\
   \leq & \|  A \tilde{\eta}_{v+\tau(u-v)} \| \|u-v \| + \|
    B \tilde{\xi}_{v+\tau(u-v)} \| \|u-v \|  \\
    + & \int_{\tau=0}^{1}\innerprod{-v}{u-v} d \tau - \frac{1}{2}\|u-v\|^2 \\
   \leq & \|u-v\|+ \|u-v\| + \|v\| \|u-v\|  \leq  2(1+\|v\|) \|u-v\| .
  \end{align*}
\end{proof}

Now we are ready to prove Theorem~\ref{thm:convergence} as follows.

\begin{reptheorem}{thm:convergence}[restated]
  Algorithm~\ref{alg:update} computes  $(1-\e)$-approximate solutions for HM-Saddle and
  $\nu$-Saddle by $\tilde{O}(d+\sqrt{d/\e \beta})$ iterations. Moreover, it takes $O(n)$ time for each iteration.
\end{reptheorem}
\begin{proof}
  Let
  \begin{align*}
   \psi(n, d) &=   \big(  2\left( \tau^{-1} + 2\gamma d^{-1} \right) \log n  \\
   & + ((2\sigma)^{-1} + 1) \|  w^{\circ}\|^2 \big) \cdot  ((4\sigma)^{-1} + 1)^{-1}
  \end{align*}

  According to Theorem~\ref{thm:allen}, we have
  \begin{align*}
  & \Exp [\|w^{\circ} - w[T] \|^2 ]  \leq  \theta^{T}  \psi(n, d)   \\
  \Rightarrow & \Exp [\|w^{\circ} - w[T] \| ]    \leq  \theta^{T/2}  \psi^{1/2}(n, d)
  \end{align*}
  In order to get a $(1-\e)$-approximate solution,  according to Lemma~\ref{lm:g2x}, it suffices to
  choose $T$ such that
  \begin{eqnarray*}
 & & \Exp \big[g(w^{\circ})-g(w[T]) \big] \\
 & \leq  &\Exp \big[2 (1+\|w[T]\|) \cdot \|w^{\circ}-w[T]\| \big] \\
&  \leq \; & 2\Exp \big[ \big(1+\|w^{\circ}-w[T]\|+\|w^{\circ}\| \big) \cdot \|w^{\circ}-w[T]\| \big] \\
&  = \; & 2\Exp \big[\|w^{\circ}-w[T]\|^2\big]+2\Exp \big[(1+ \|w^{\circ}\| )\|w^{\circ}-w[T]\| \big]\\
 & \leq \; & 2\theta^{ T}  \psi(n, d)+ 2(1+ \|w^{\circ}\| )\theta^{T/2}  \psi^{1/2}(n, d)  \\
&\leq \; & \e \OPT.
  \end{eqnarray*}
  Note that $\theta = 1 - \frac{1}{d + q\sqrt{d}/\sqrt{\gamma }}  = 1 - \frac{1}{d + q\sqrt{d}/\sqrt{\e\beta/2
      \log n}} $. Thus, we only need to have
  \begin{align*}
    T  &\geq  \log_{\theta} \left( \frac{\e \OPT}{2 \psi(n,d)}
    \right)+ 2\log_{\theta} \left( \frac{\e \OPT}{1 + \| w^{\circ} \| }  \cdot \psi^{-1/2}(n,d)
    \right) \\
    & \geq 2 (d+\sqrt{2d/\e\beta}\cdot O(\log n))\log \left( \frac{ (1 + \| w^{\circ} \|) \psi(n,d)
      }{\e \OPT} \right) \\
 & =   \tilde{\Omega}( d + \sqrt{d/\e \beta})
  \end{align*}
\end{proof}

\section{Supplementary Materials of Experiments}
\label{apd:exp}

\topic{Data set}  We use both synthetic and real-world data sets. The real data is
from~\cite{chang2011libsvm} including the separable data set  ``iris" and ``mushrooms'' and
non-separable data set  ``w8a'', ``gisette'', ``madelon'', ``phishing'', ``a1a'', ``a5a'',``a9a'', ``ijcnn1'',
``skin\_nonskin". We summary the information of the data in Table~\ref{tab:realdata}.

\begin{table}[t]
  \centering
  \caption{The sketch of the data sets. Here $n$ is the number of points. $n_1$ is the number of points with $+1$ label and $n_2$ is the number of points with $-1$ label. $d$ is the dimension of the features. $nnz$ is the non-zeros data ratio. }

  \label{tab:realdata}
  \begin{tabular}{ | c | c | c | c | c | c | }
    \hline
    \multirow{2}*{data set} & \multicolumn{5}{|c|}{parameters}  \\
    \cline{2-6}
         & $n $ & $n_1$ & $n_2$  & $d$  & $nnz$  \\
      \hline
      a1a & 1605 & 395 & 1210 & 119 & 0.12 \\
      \hline
      a5a & 6414 & 1569 & 4845 & 122 & 0.114 \\
      \hline
      a9a & 32561 & 7841 & 24,720 & 123 & 0.113  \\
    \hline
    phishing & 11055 & 6157 & 4898 & 68 & 0.441 \\
    \hline
    mushrooms & 8124 & 3916 & 4208 & 112 & 0.188 \\
    \hline
    iris & 150 & 100 & 50 & 4 & 0.978 \\
    \hline
    gisette & 6000 & 3000 & 3000 & 5000 & 0.99  \\
    \hline
      w8a &  49749 & 1479 & 48270 & 300 &0.038  \\
      \hline
      ijcnn1 & 49990 & 4,853 & 45,137 & 22 & 0.590   \\
      \hline
      skin\_nonskin & 245057 & 50859 & 194198 & 3 & 0.982 \\
    \hline
  \end{tabular}
\end{table}

Besides the real world data, we generate some synthetic data sets. There are  three types synthetic data: 1) separable synthetic data, 2) non-separable synthetic data, 3) sparse non-separable synthetic data. We describe the ways to generate them as follows.

\begin{itemize}
\item Separable synthetic data: we randomly choose a hyperplane $H$ which overlaps with the unit norm ball in $\R^{d}$
space. Then we randomly sample $n$
points in a subset of the unit ball such that the ratio of the maximum distance among the
points to $H$ over the minimum distance to $H$ is $\beta_1 = 0.1$. Let the labels of points above $H$  be
$+1$ and let others be $-1$.

\item Non-separable synthetic data: The difference  from the separable synthetic data is that for those points with distance to $H$
smaller than $\beta_2 = 0.1$, we randomly choose their labels to be $+1$ or $-1$ with equal
probability. Moreover, we also use real-world

\item Sparse non-separable synthetic data: First, we set a parameter ``nnz" which represent the number of non-zeros elements in each point. The only difference between the dense non-separable synthetic data is that we randomly sample $n$ points such that each point only has ``nnz" non-zeros non-zeros points.
\end{itemize}

\topic{$\mu$-SVM form used in NuSVC}
The form of the $\mu$-SVM  used in scikit-learn is a variant of the form in the paper. We give the formulation as follows.
\begin{equation}
  \label{eq:scikit}
  \begin{array}{lcl}
	\min\limits_{w, b,\rho,\delta}  & \frac{1}{2}\| w \|^2 - \mu\rho' + \frac{1}{n}\sum_{i}\delta_i  & \\
	\text{s.t.}&  y_i(w^{\mathrm{T}}x_i - b)   \geq \rho' - \delta_i , \delta_i \geq 0 ,& \forall i
  \end{array}
\end{equation}
\cite{crisp2000geometry} prove that through reparameterizing, the above formulation is equivalent to
$\nu$-SVM~\eqref{eq:nusvm}. Concretely speaking, let
$$
\nu = \frac{2}{\mu n}, \; \rho = \frac{\rho'}{\mu}.
$$
Then, \eqref{eq:scikit} can be transformed to $\nu$-SVM~\eqref{eq:nusvm}.
\eat{
The parameter $\nu$ in table~\ref{tab:nu} is equivalent in this sense. See Section 2.1 in~\cite{crisp2000geometry}
for more details.
}

\topic{Parameter $\nu$ in $\nu$-SVM}
As we have discussed in Section~\ref{sec:exp}, although when $\nu$ belongs to $ [1 / \min (n_1, n_2) ,1) $, $\nu$-SVM has feasible solution, where $n_1$ is the number of points with positive label and $n_2$ is the number of points with negative label. Not all feasible  $\nu$ can induce a reasonable prediction model.
If $\nu$ is too close to 1, the two reduced polytopes are not separable. The closest distance between the two reduced polytopes is zero. Note that in general the overlapping points are not unique. Hence the solution is not unique. Moreover, because the solution corresponds two overlapped points, the vector $w$ (which represents the vector determined by the two points) is not unique, hence, is unstable. Overall, here we select a relatively small $\nu$.

Recall that we let
$$
\nu = 1 / (\alpha \min (n_1, n_2)).
$$
We set $\alpha = \{0.1, 0.3, 0.5 \}$ and train the $\nu$-SVM model on the data set ``a9a", ``ijcnn1", ``phishing". We list the results in Table~\ref{tab:alpha}.
\begin{table}[t]
  \centering
  \caption{Experiments on different parameter $\nu$ in $\nu$-SVM. Here $n$ is the number of points. $n_1$ is the number of points with $+1$ label and $n_2$ is the number of points with $-1$ label. }
  \label{tab:alpha}
  \begin{tabular}{ | c | c | c | c | c | c | }
    \hline
    \multirow{2}*{data set} & \multirow{2}*{$\alpha$} & \multicolumn{2}{|c|}{LIBSVM}  & \multicolumn{2}{|c|}{Saddle-SVC}  \\
    \cline{3-6}
         &   &   Obj & Test Acy &  Obj & Test Acy   \\
      \hline
  \multirow{3}*{a9a}& 0.1 & 6e-12 & 0.35 &   6e-4   & 0.69   \\
  \cline{2-6}
     & 0.3 & 6e-13 & 0.36 &    7e-4   & 0.69   \\
     \cline{2-6}
          & 0.5 & 6e-13 & 0.71  &    3e-4   & 0.70   \\
    \hline
     \multirow{3}*{phishing}   & 0.1 & 6e-11 & 0.89 & 3e-4 & 0.82  \\
    \cline{2-6}
    & 0.3 & 0.002 & 0.93 & 0.002 & 0.93  \\
    \cline{2-6}
    & 0.5 & 0.01 & 0.92 & 0.01 & 0.93  \\
    \hline
      \multirow{3}*{ijcnn1} & 0.1 &  2e-12 & 0.17 & 0.0039 & 0.73   \\
      \cline{2-6}
       & 0.3 &  6e-13 & 0.17 & 0.002 & 0.47   \\
       \cline{2-6}
          & 0.5 &  3e-13 & 0.80 & 0.0004 & 0.31   \\
      \hline
  \end{tabular}
\end{table}

\topic{Saddle-SVC vs. LinearSVC} As discussed before, they solve different SVM variants. Thus, we use the test accuracy instead of the objective values to evaluate the convergent rate.
First, we explain the stop criteria of Saddle-SVC. In Theorem~\ref{thm:hmnu}, we prove that Saddle-SVC converge in $\tilde{O}(d+\sqrt{d/\e \beta})$ rounds. Let  $T = d+\sqrt{d/\e \beta}$. We repeat the iterations of Saddle-SVC and compute the objective function every $T$ rounds. If  the difference between two consecutive  objective value is less than $\epsilon$, then output the results.  We note that LinearSVC is very efficient for sparse data set. But for the dense data set, Saddle-SVC performs better.  In the experiment, we use ``nnz" to represent the ratio of  non-zero elements to all elements.  We show that the parameter nnz significant affects the efficient of LinearSVC, but Saddle-SVC is barely affected.
We use ``skin\_nonskin" and ``w8a" and synthetic data sets with different parameter nnz  to evaluate the performance.  We list the details in Table~\ref{tab:linear}.

\begin{table}[t]
  \centering
  \caption{Saddle-SVC vs. LinearSVC: The parameter $\alpha$ for Saddle-SVC is 0.85. The parameter $C$ for LinearSVC is 8. skin\_nonskin: $n = 245057, d = 3$. w8a: $n = 49745, d = 300$.  Synthetic data: $n = 100000, d =128$.}

  \label{tab:linear}
  \begin{tabular}{ | c | c | c | c | c | c |}
    \hline
    \multirow{2}*{data set} &   \multirow{2}*{nnz}  &  \multicolumn{2}{|c|}{Saddle-SVC}  &
    \multicolumn{2}{|c|}{LinearSVC}  \\
    \cline{3-6}
	 & & test acy  & time   & test acy & time   \\
      \hline
     skin & 0.98 &   0.931  & 40.0s &  0.913 & 654s   \\
    \hline
      w8a &  0.03  & 0.984 & 3075s &  0.986 & 12.5s   \\
     \hline
      synthetic & 0.1 &0.804 & 393s &  0.830 & 28.2s   \\
      \hline
      synthetic& 0.5 & 0.844 & 369s & 0.843 & 214s \\
       \hline
      synthetic &  0.9 & 0.825 & 363s &   0.828 & 537s \\
    \hline
  \end{tabular}
\end{table}

\topic{Saddle-DSVC vs. HOGWILD!}
As Saddle-DSVC is the first practical distributed algorithm for $\nu$-SVM. We use another popular distributed algorithm called HOGWILD! for comparison.
Note that HOGWILD! is use to solve $C$-SVM and $l_2$-SVM but not, $\nu$-SVM. Thus, instead of the objective function, we use the accuracy to evaluate the performance of the algorithms.  Here we choose use HOGWILD! for $C$-SVM and Saddle-DSVC for $\nu$-SVM. See the details in Figure~\ref{fig:morerate}. For comparison, we also provide the results of Gilbert Algorithm. Here we choose $\alpha = 0.85$ for Saddle-DSVC and $C = 32$ for HOGWILD!.  We can see that Saddle-DSVC converges faster than HOGWILD! w.r.t. communication cost.
Moreover, Saddle-DSVC is more stable than HOGWILD! algorithm.

\eat{
Besides the results in Figure~\ref{fig:obj}, we test more data sets for the distributed algorithms. We give the results in Figure~\ref{fig:morerate}.
By the experimental results, we obtain the same conclusion as in Section \ref{sec:exp}.

\begin{figure}[t]
  \centering
  \subfigure[synthetic data, $d=128, n = 10000$]{
    \includegraphics[width = 0.21\textwidth]{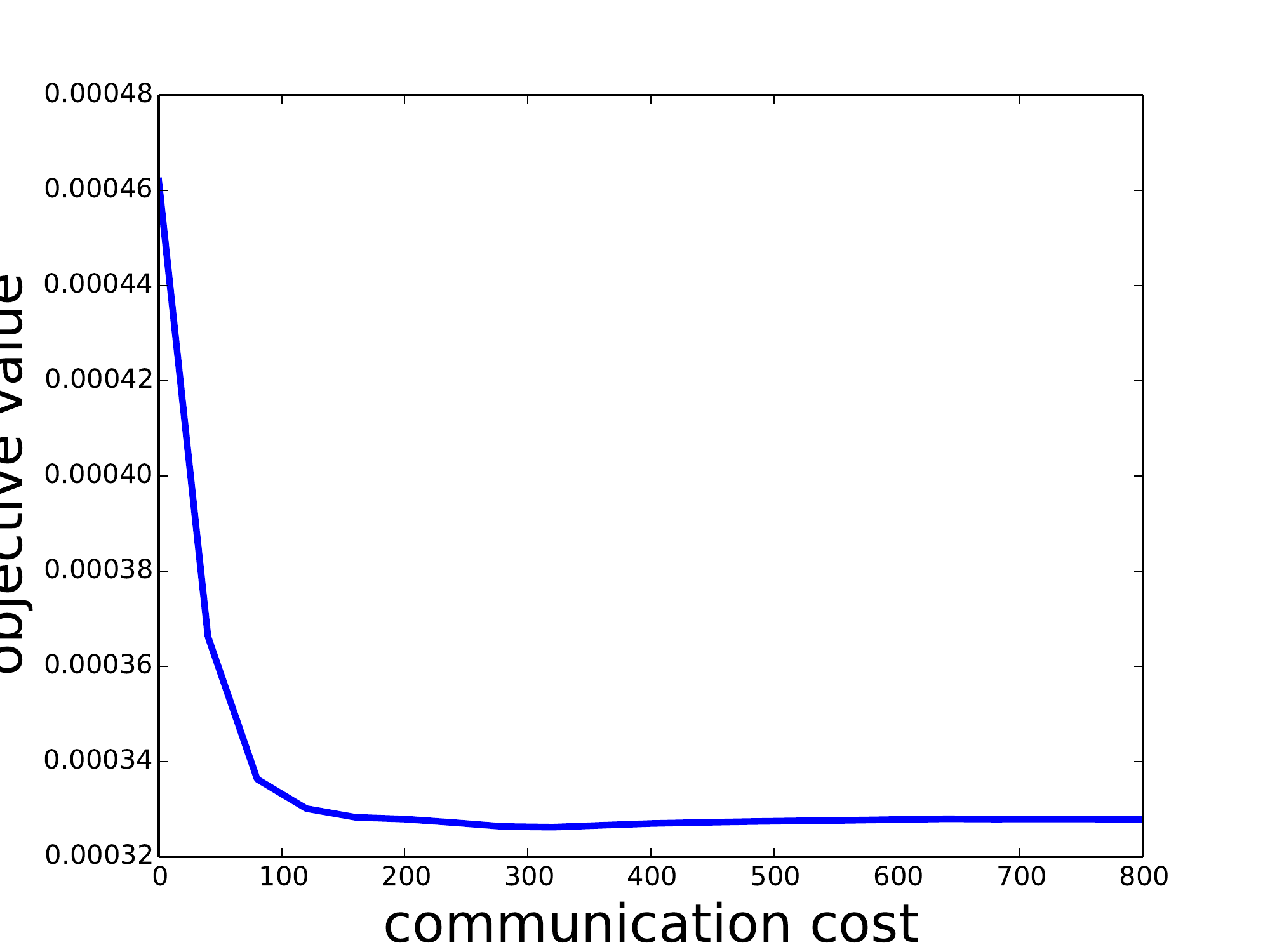}}
  \subfigure[synthetic data, $d=256, n= 10000$]{
    \includegraphics[width = 0.21\textwidth]{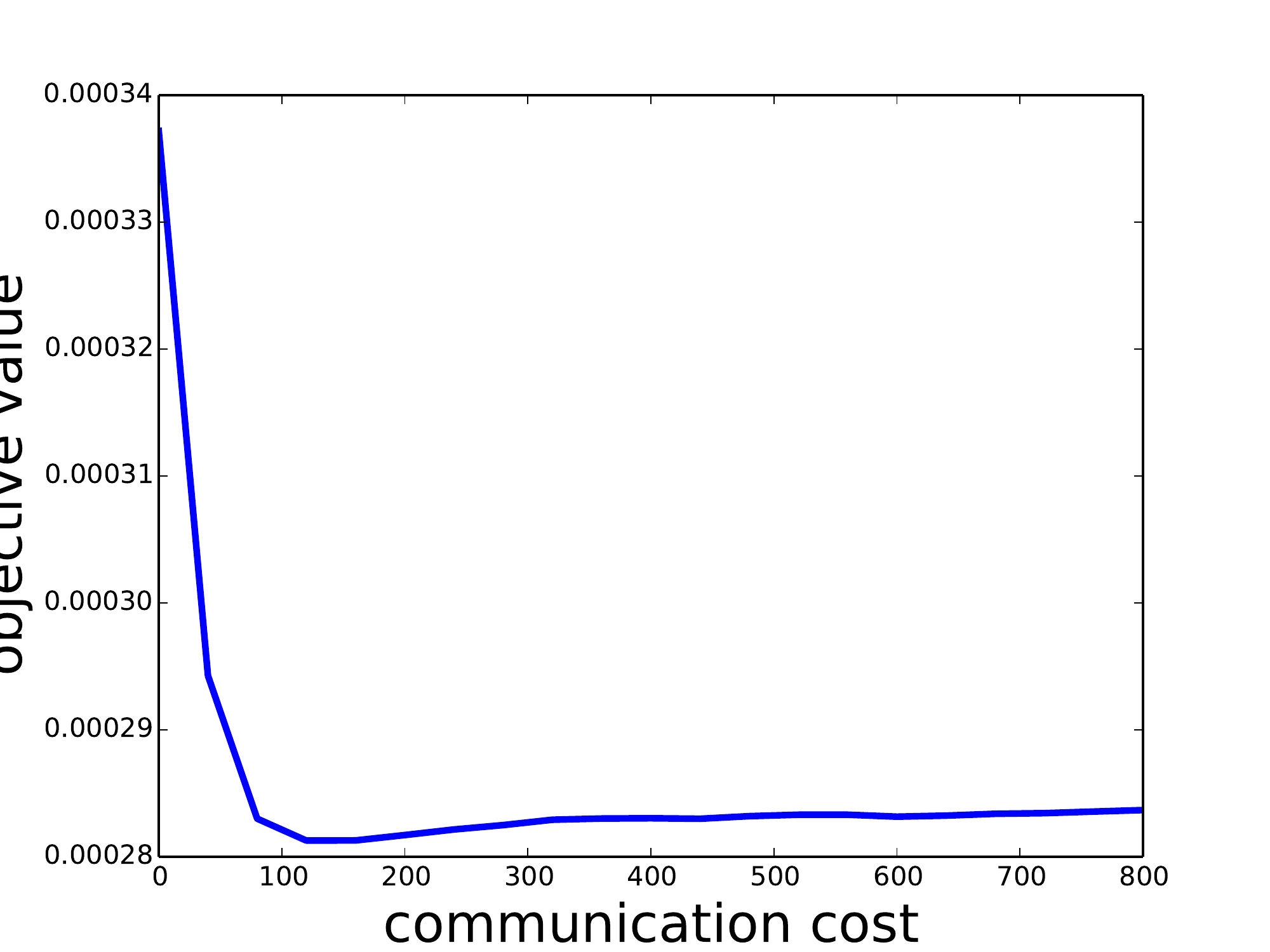}}
  \subfigure[a1a]{
    \includegraphics[width = 0.21\textwidth]{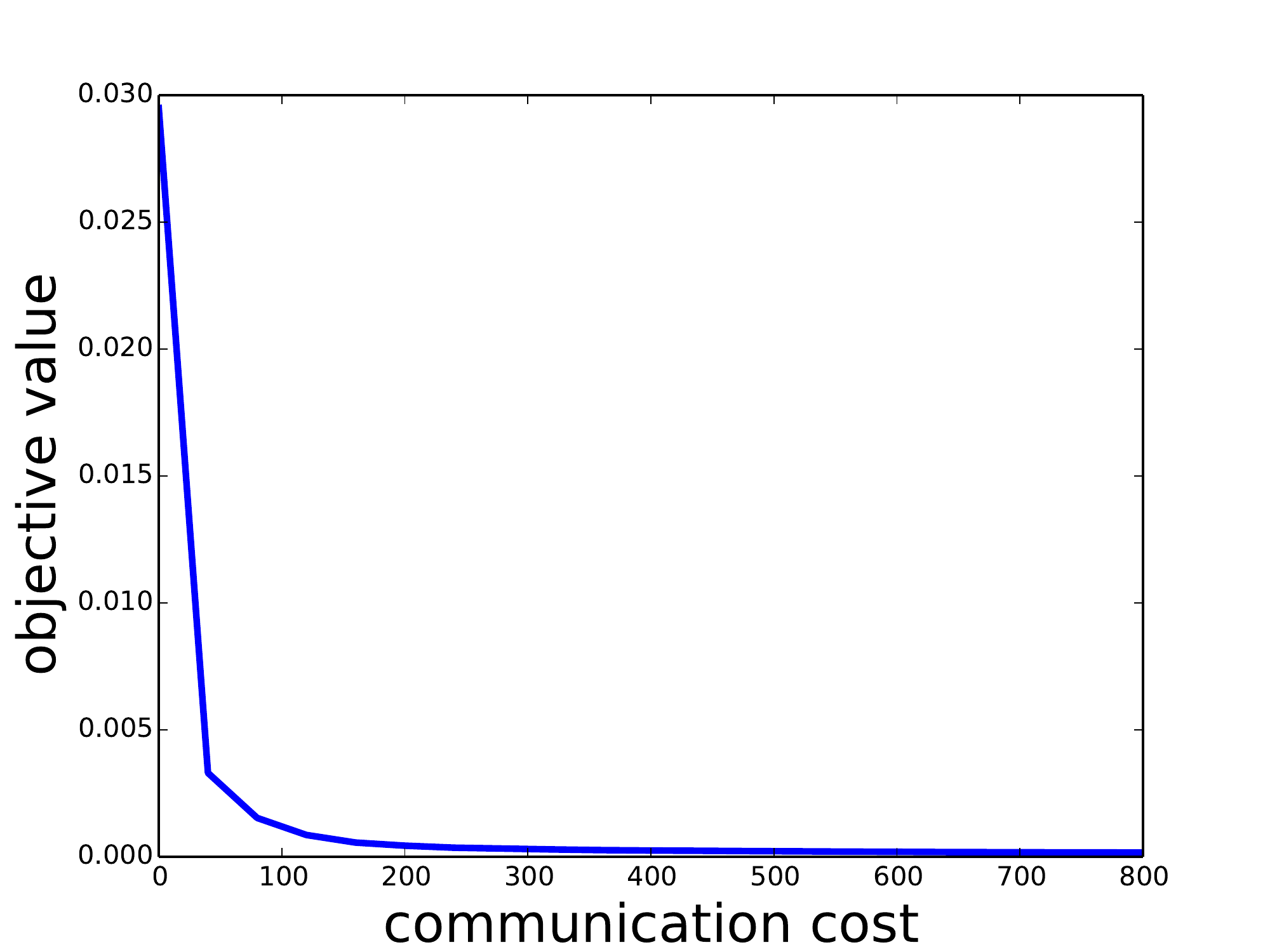}}
  \subfigure[a5a]{
    \includegraphics[width = 0.21\textwidth]{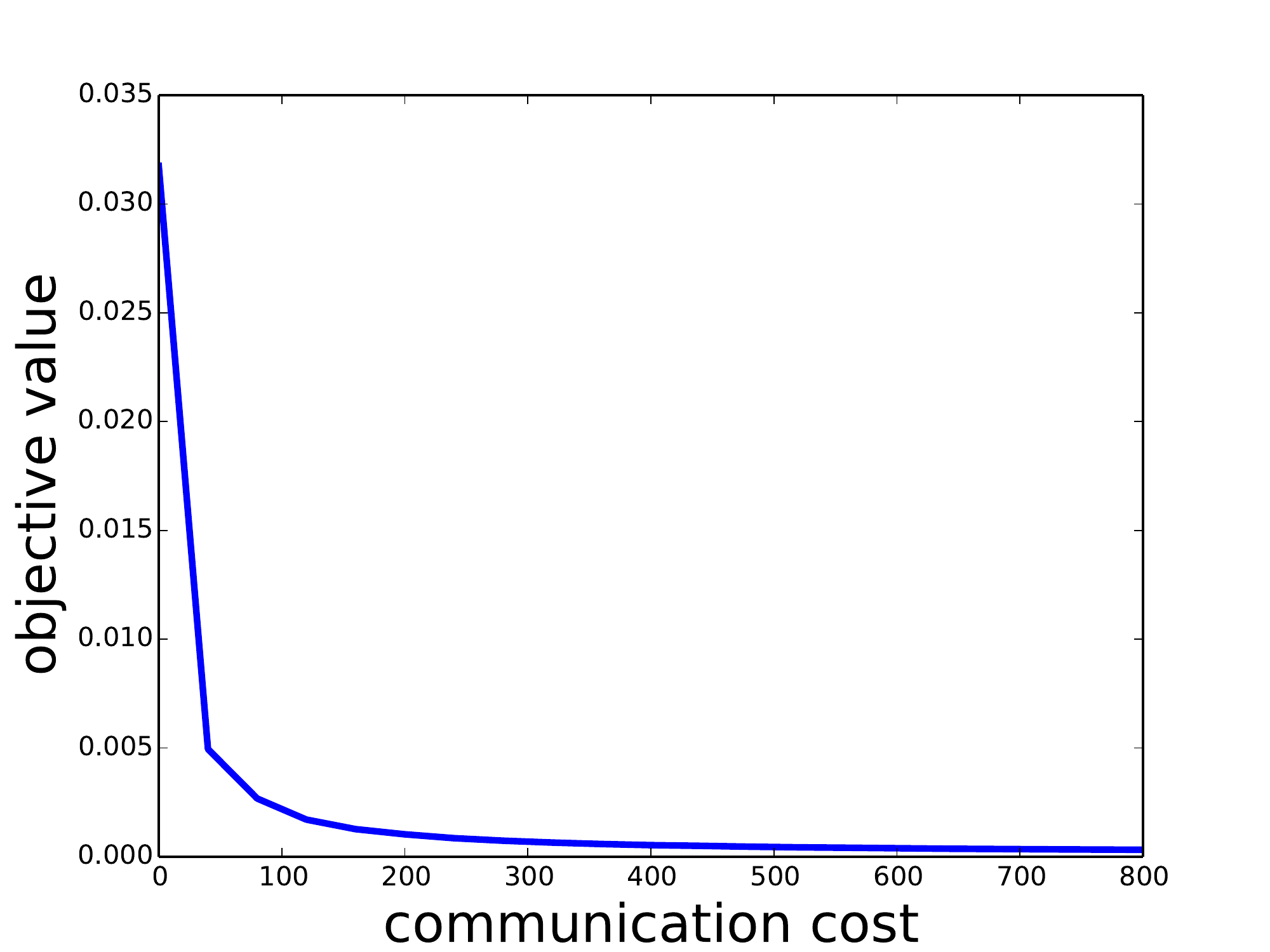}}
  \caption{The objective function  Saddle-DSVC with respective to the communication cost.  Here $k
    =20$.}
  \label{fig:moreobj}
\end{figure}
}

\begin{figure}[t]
  \centering
  \subfigure[synthetic]{
    \includegraphics[width = 0.35\textwidth]{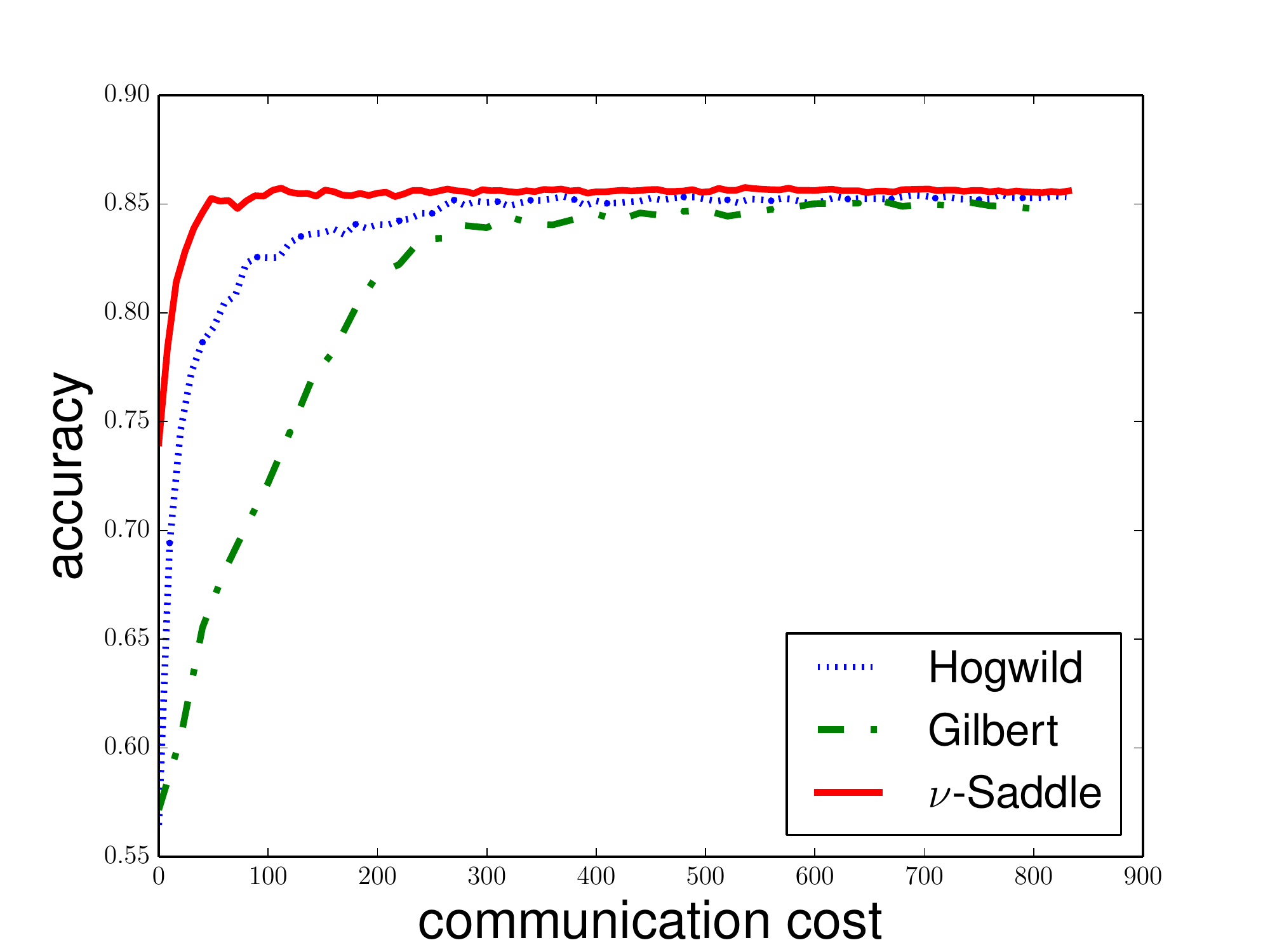}}
  \subfigure[synthetic]{
    \includegraphics[width = 0.35\textwidth]{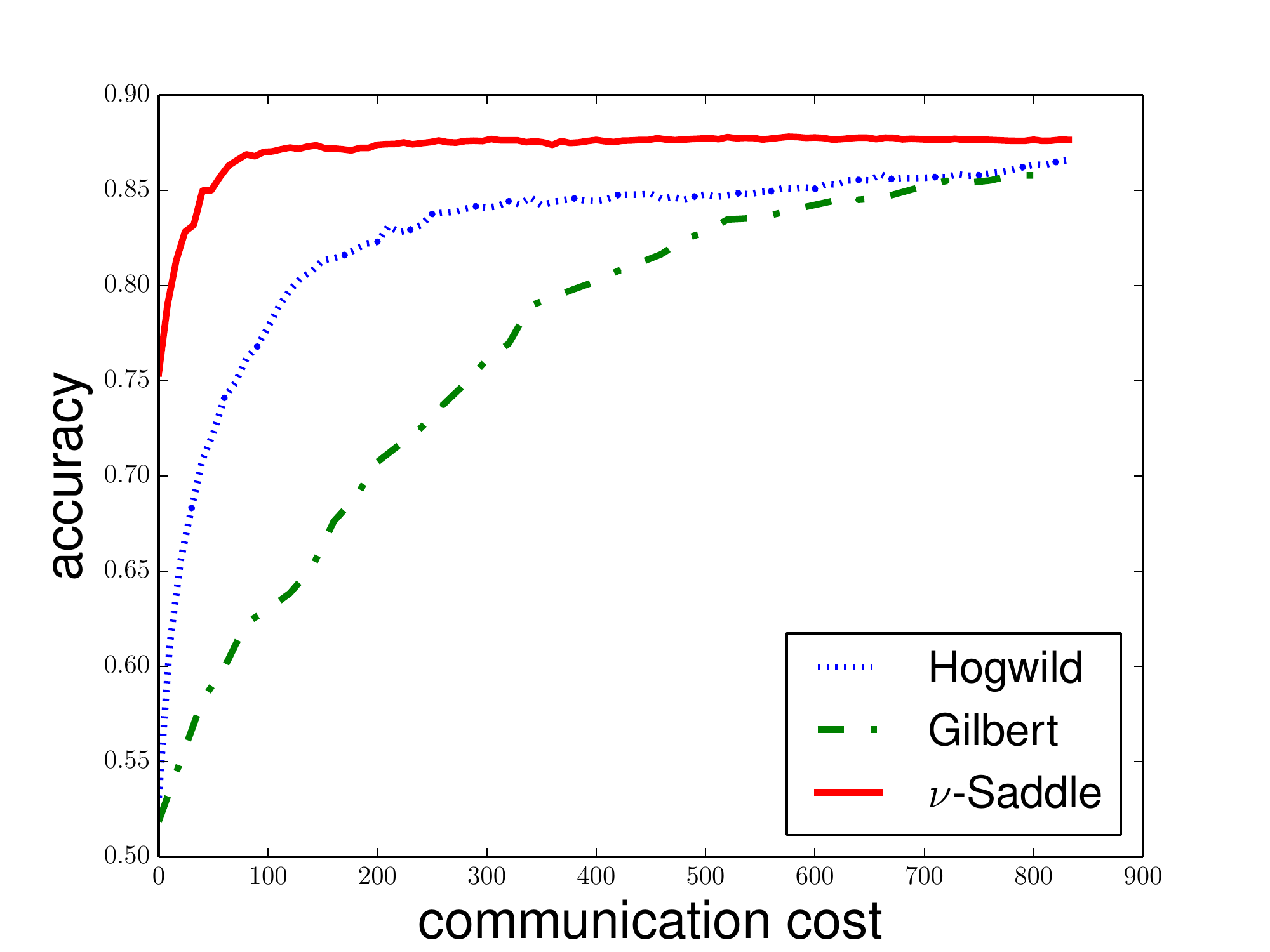}}
  \subfigure[phishing]{
    \includegraphics[width = 0.35\textwidth]{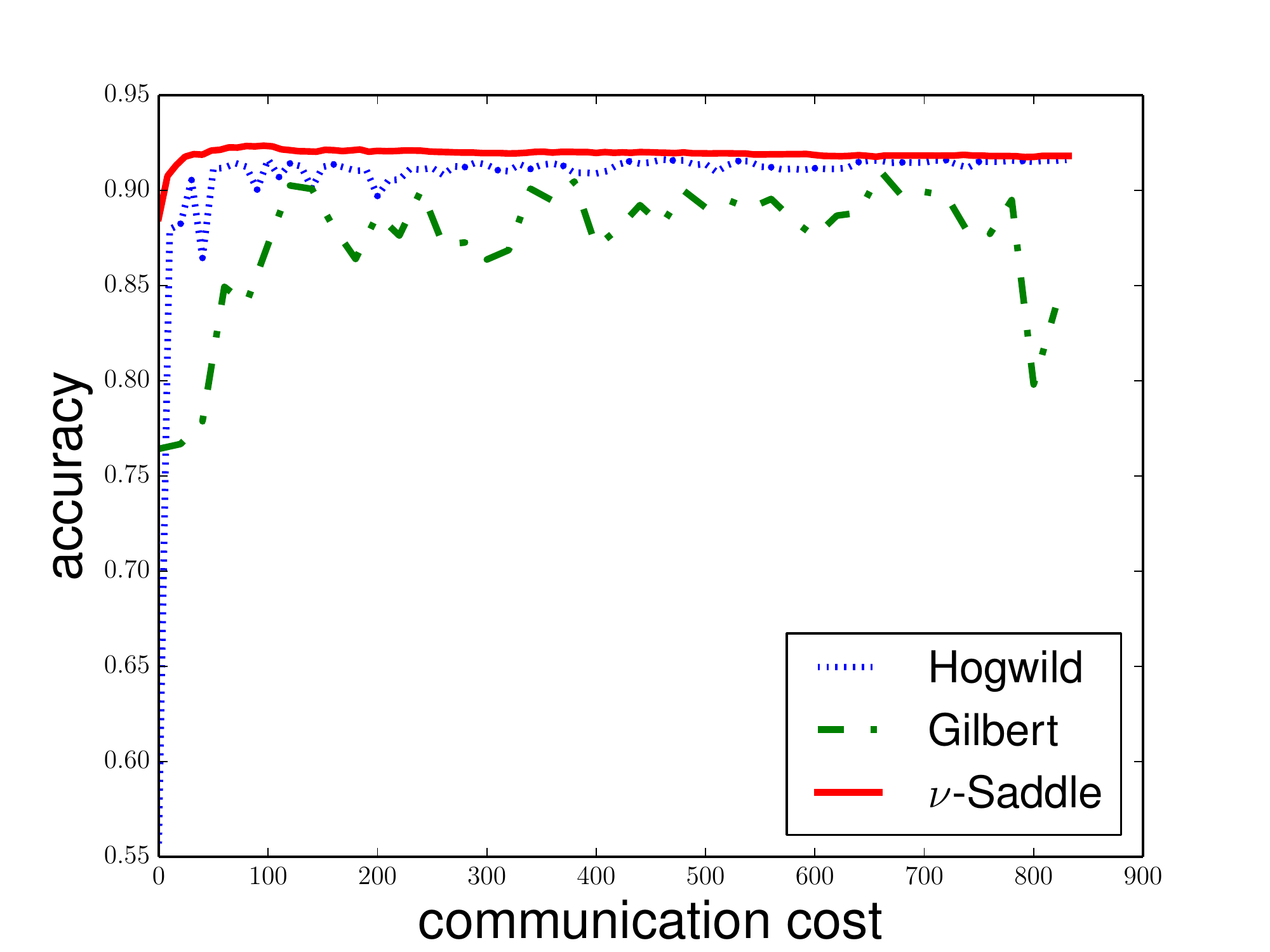}}
  \subfigure[a9a]{
    \includegraphics[width = 0.35\textwidth]{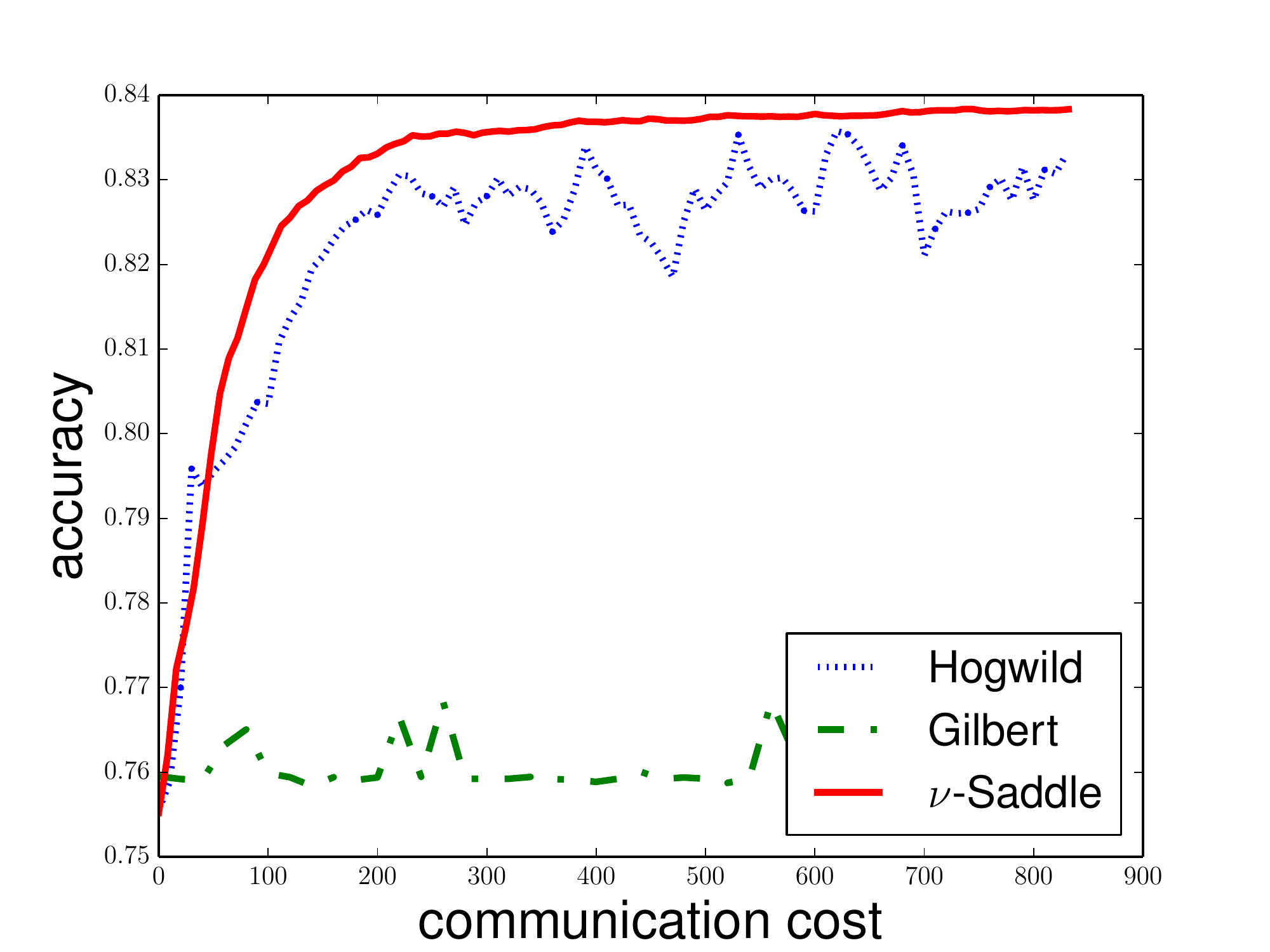}}

  \subfigure[a1a]{
    \includegraphics[width = 0.35\textwidth]{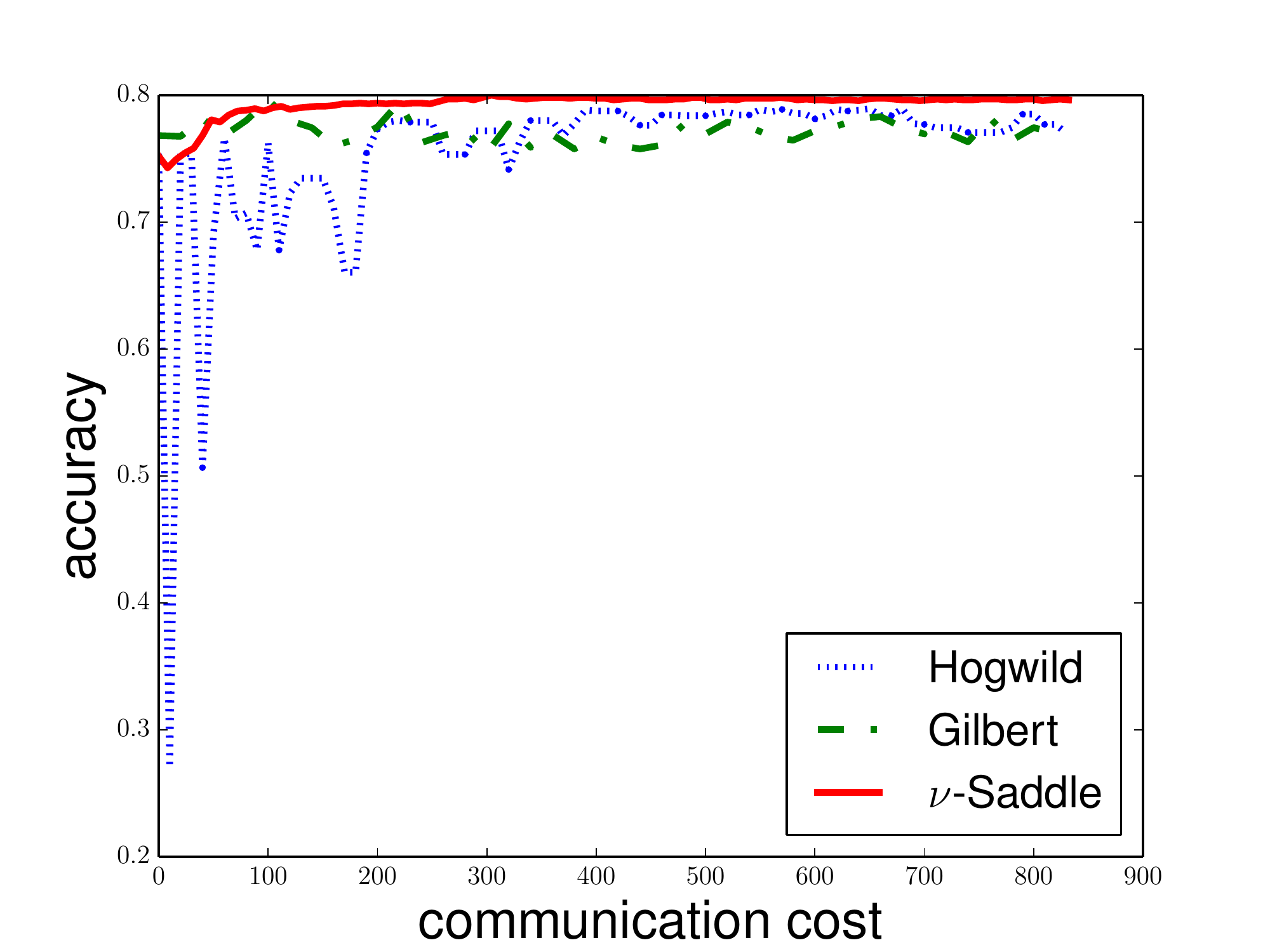}}
  \subfigure[a5a]{
    \includegraphics[width = 0.35\textwidth]{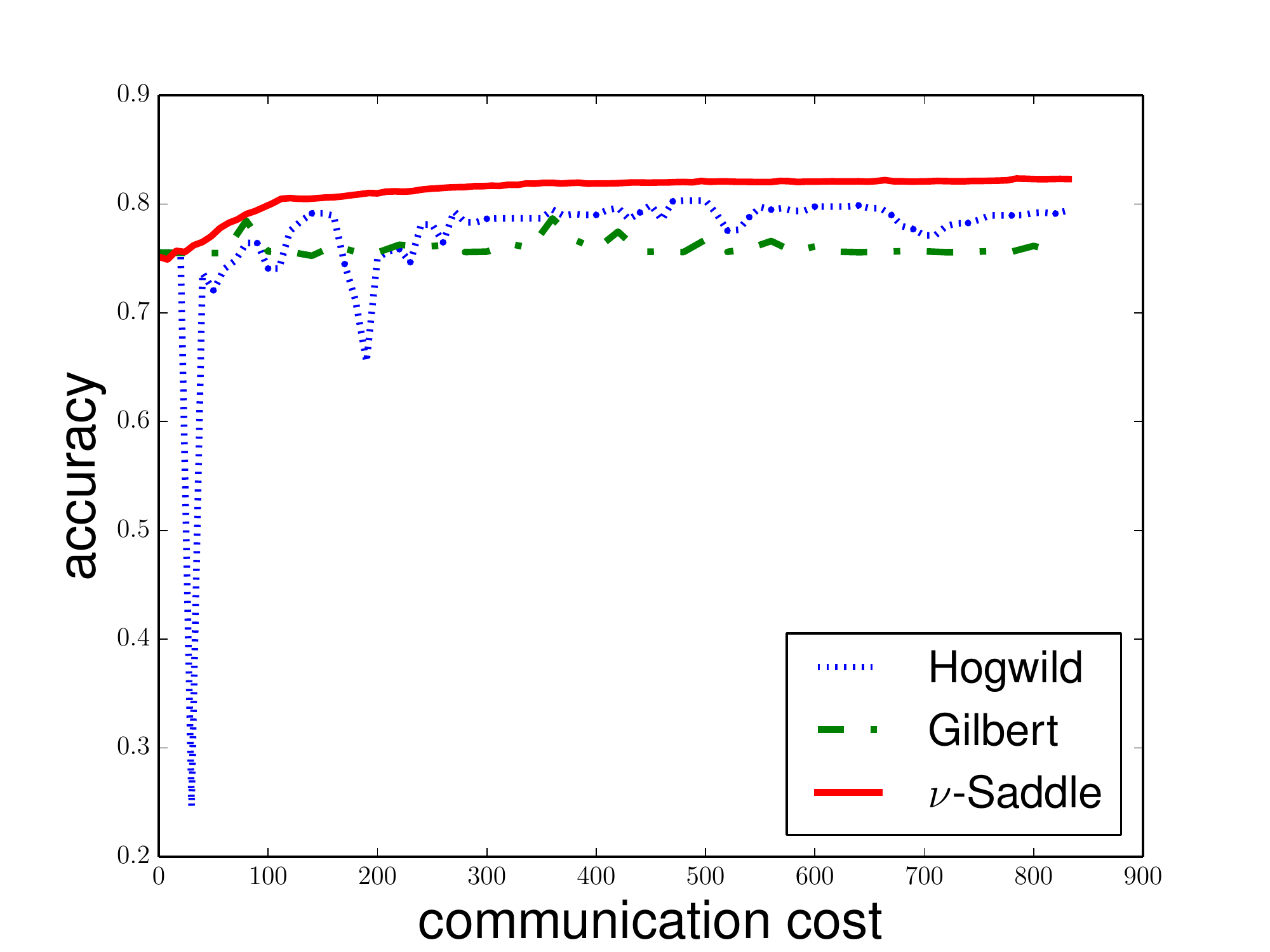}}
  \subfigure[gisette]{
    \includegraphics[width = 0.35\textwidth]{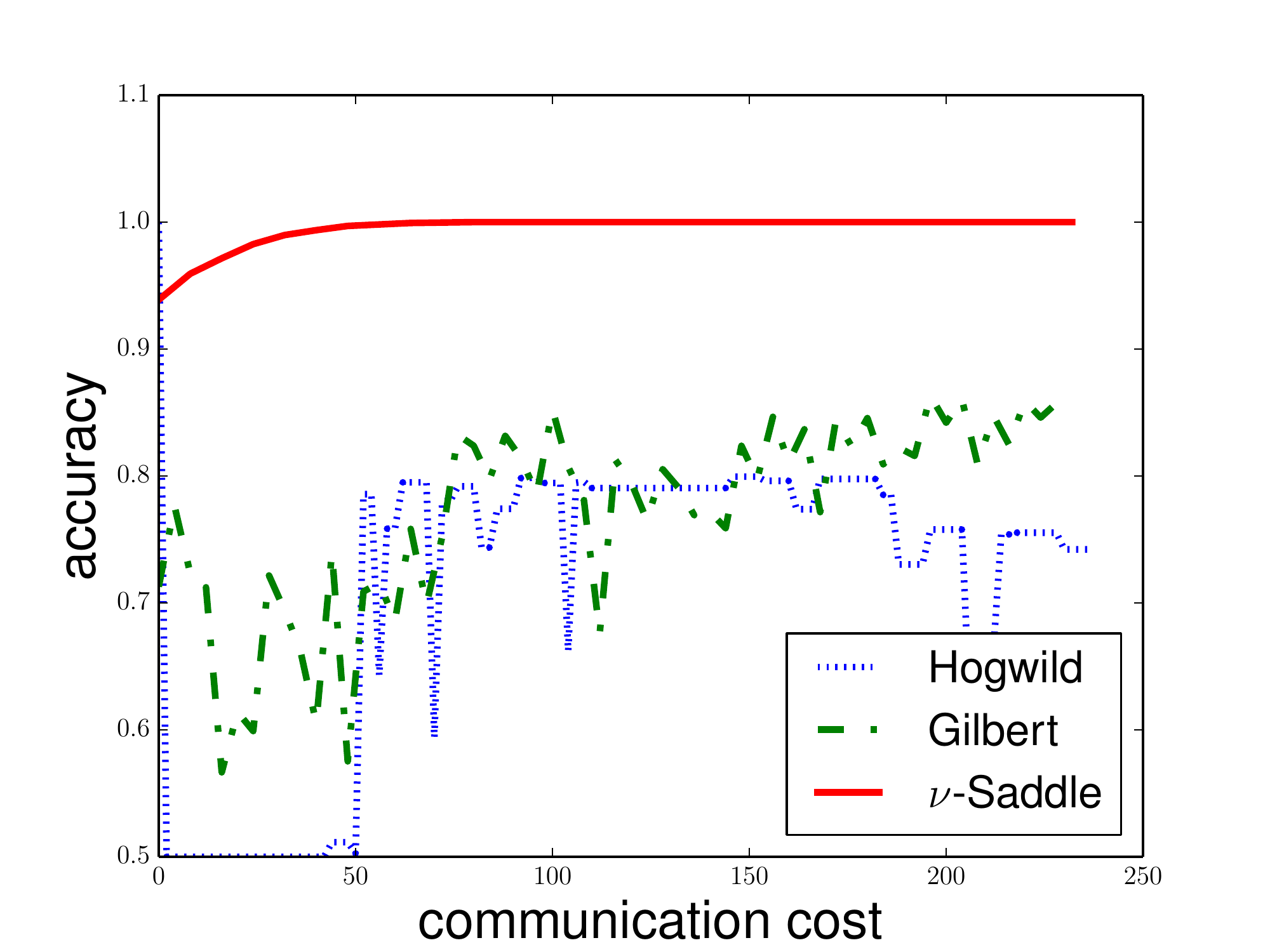}}
  \subfigure[madelon]{
    \includegraphics[width = 0.35\textwidth]{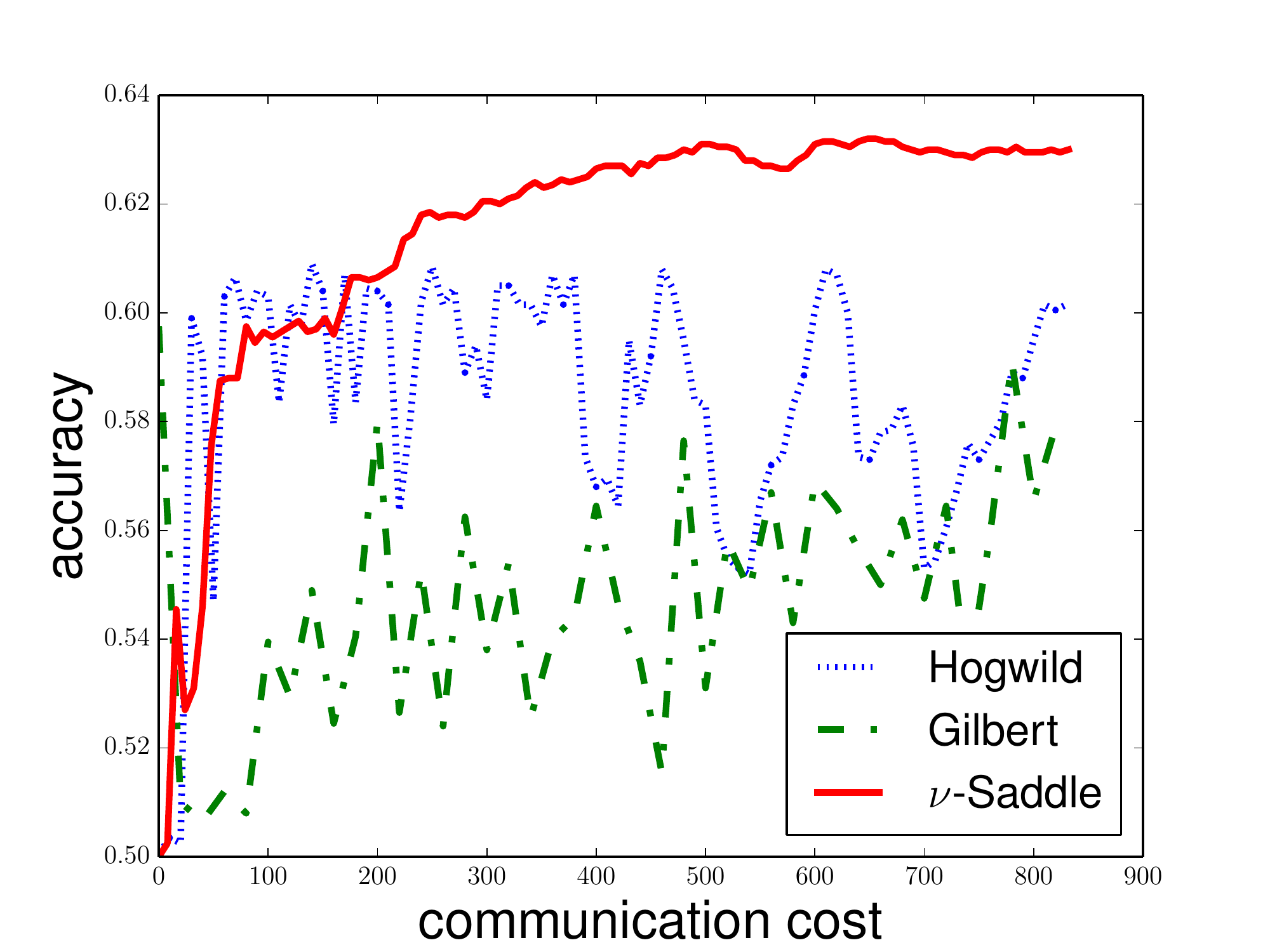}}
  \caption{The accuracy of the distributed Gilbert Algorithm, HOGWILD!, Saddle-DSVC. Here $k =20$.  Synthetic data in Figure(a): $n = 10000, d = 128$,
  Synthetic data in Figure(b): $n = 10000, d =256$.  See Table~\ref{tab:realdata} for the informations of other data sets.  Here, we choose $\alpha = 0.85$ for Saddle-DSVC  and $C = 32$ for HOGWILD!.
  }
  \label{fig:morerate}
\end{figure}


\end{document}